\documentclass[12 pt, symmetric]{article}

\pdfoutput=1
\usepackage{graphicx}
\usepackage[margin=0.75in]{geometry}
\usepackage{amsthm}
\usepackage{amsfonts}
\usepackage{multirow}
\usepackage{adjustbox}
\usepackage[all]{xy}
\usepackage{mathtools}
\usepackage[utf8]{inputenc}
\usepackage{hyperref}
\usepackage{algorithmicx}
\usepackage{algpseudocode}
\usepackage{algorithm}
\usepackage{authblk}

\newcommand{\real}{\mathbb{R}}

\theoremstyle{plain}

\newtheorem{proposition}{Proposition}

\theoremstyle{remark}
\newtheorem{remark}{Remark}
\newtheorem{example}{Example}
\newtheorem{definition}{Definition}

\begin{document}

\title{\sc Attraction-Repulsion clustering with applications to fairness}

\author[1]{Eustasio del Barrio}
\author[2]{Hristo Inouzhe}
\author[3]{Jean-Michel Loubes}

\affil[1]{\it Departamento de Estadística e Investigación Operativa and IMUVA,
Universidad de Valladolid, Spain.}
\affil[2]{\it Basque Center for Applied Mathematics (BCAM), Spain.}
\affil[3]{\it Université de Toulouse, Institut de Mathématiques de Toulouse, France.}

\date{}

\date{}
\maketitle%

\maketitle

\begin{abstract}
We consider the problem of \emph{diversity enhancing clustering}, i.e, developing clustering methods which produce clusters that favour diversity with respect to a set of protected attributes such as race, sex, age, etc. In the context of {\it fair clustering}, diversity plays a major role when fairness is understood as demographic parity. To promote diversity, we introduce perturbations to the distance in the unprotected attributes that account for protected attributes in a way that resembles attraction-repulsion of charged particles in Physics. These perturbations are defined through dissimilarities with a tractable interpretation.  Cluster analysis based on attraction-repulsion dissimilarities penalizes homogeneity of the clusters with respect to the protected attributes and leads to an improvement in diversity. An advantage of our approach, which falls into a pre-processing set-up, is its compatibility with a wide variety of clustering methods and whit non-Euclidean data. We illustrate the use of our procedures with both synthetic and real data and provide discussion about the relation between diversity, fairness, and cluster structure. Our procedures are implemented in an R package freely available at \href{https://github.com/HristoInouzhe/AttractionRepulsionClustering}{https://github.com/HristoInouzhe/AttractionRepulsionClustering}.
\end{abstract}

\section{Introduction}
Artificial Intelligence (AI) in its different branches is pervasive in today's world. It is almost impossible to ignore its effects on individuals, groups, and even entire populations. Although the AI revolution has brought a variety of major positive breakthroughs, as it is common in any revolution, time has evidenced some of the limitations and harms involved in heavily relying on algorithms and data. An account of the history of AI development and some of the causes of its limitations are given in \cite{Cristianini2019}. The author argues that the negative effects of AI that we are facing are due to an `ethical debt', a kind of externality of our technological infrastructure. This `ethical debt' is formed by, among others, the cost of fixing the lack of fairness and transparency of machine decisions, the cost of regulating AI and the cost of making services secure against surveillance and manipulation.

An entire research field, known as {\it fair learning}, is currently dedicated to `repaying' part of the `ethical debt' by studying causes and mitigating or eliminating effects of bias and unfairness in machine learning. It is widely agreed that supervised and unsupervised classification procedures are increasingly more influential in people's life since they are used in credit scoring, article recommendation, risk assessment, spam filtering or sentencing recommendations in courts of law, among others. Hence, controlling the outcome of such procedures, in particular ensuring that some variables which should not be considered due to moral or legal issues are not playing a role in the classification of the observations, has become an important field of research. The main concern is to detect whether decision rules, learnt from variables X, are biased with respect to a subcategory of the population driven by some variables called protected or sensitive. Such variables induce a bias in the observations and are correlated to other observations. Hence avoiding this effect cannot be achieved by the naive solution of ignoring such protected attributes. Indeed, if the data at hand reflects a real-world bias, machine learning algorithms can pick on this behaviour and emulate it. An extensive overview on types of biases and discrimination and on fair machine learning strategies can be found in \cite{mehrabi2019}. Additional overviews of such legal issues and mathematical solutions to address them can be found in  \cite{2016arXiv161008077L, chouldechova2017fair, 2018arXiv181001729B} or \cite{2018arXiv180204422F}.

Any fair machine learning procedure starts with a particular definition of fairness, i.e., with a statement of an abstract mathematical proxy for some fairness notion that must be considered. A non-exhaustive list of definitions can be found in \cite{mehrabi2019}. It is worth mentioning that some popular definitions of mathematical fairness are incompatible between each other, and that different situations may require different definitions of fairness. Once a definition for fairness is provided, a way of imposing or promoting it can be established, for example,  by transforming the data to avoid correlation between the set of sensitive attributes and the rest of the data \cite{cosa_3, fair_paula} or by modifying the objective functions of the algorithms in a way that favours fairness \cite{cosa_1,fair_gpc}. Notice that this scheme has sources of bias, one is the subjective choice of a definition of fairness, and another one is the choice of a mathematical proxy for that definition. Therefore, a problem specific approach may be a good way of avoiding pitfalls.


The problem we will be interested in is intimately related to the problem of fair clustering. Classical (or standard) cluster analysis or clustering is the task of dividing a set of objects in such a way that elements in the same group or cluster are more similar, according to some dissimilarity measure, than elements in different groups. Such methods have been extensively investigated in the literature. We refer to \cite{HennigMeilaMurtaghRocci} and references therein for an in-depth overview. However, classical clustering is insensitive to notions of mathematical fairness. Suppose we observe data that includes information about attributes that we know, or suspect are biased with respect to the protected class. If the biased variables are dominant enough, a standard clustering on the unprotected data will result in some biased clusters, that is, clusters with possibly large differences in distribution of the protected attributes. Therefore, if we classify new instances based on this partition, we will incur in biased decisions. In essence, this will be a violation of the notion of demographic parity (\cite{cosa_3}), a popular proxy for fairness, adapted to the clustering setting.

A natural mathematical definition for fairness in the setting of 
clustering, and the one that
we adopt, is based on the notion of demographic parity, namely, on looking 
for balanced clusters, as introduced in \cite{fair_k_means}. 
Quoting from \cite{Abbasi2021} `What does it mean for a clustering to 
be fair? One popular approach seeks to ensure that each cluster contains 
groups in (roughly) the same proportion in which they exist in the 
population. The normative principle at play is balance: any cluster 
might act as a representative of the data, and thus should reflect its 
diversity'. Hence, a fair clustering would be the situation in which, 
in the partition of the data, the proportions of the protected 
attributes are the same in each cluster (hence, the same as the 
proportions in the whole dataset). We might argue that balance 
(demographic parity) is not necessarily equivalent to fairness, and 
proper consideration must be given to the problem of when it is 
appropriate and useful to identify fairness with balance. Nonetheless, 
diversity in itself may be a desirable property. Even more, diversity 
may be easier to formalize and impose, since it is less charged with 
ethical connotations. Hence, we will focus on clustering procedures that 
favour diversity in the protected attributes, in what we call diversity 
enhancing clustering.

In our setting, diversity enhancing clustering is a relaxation of diversity preserving clustering (which is usually considered as a fair clustering method). The latter can be described as follows. Let us have a dataset $\mathcal{DS}=\{(X_1,S_1),\dots,(X_n,S_n)\}$, with $S_i\in S$, and for simplicity assume that the protected attributes $S$ are discrete. Notice that the unprotected attributes $X_i$ and the protected ones $S_i$ should be correlated for diversity preserving clustering to be relevant. A \textit{diversity preserving clustering} or a \textit{diversity preserving partition}, $\mathcal{C}_f=\{\mathcal{C}_i\}_{i=1}^k$, where $\mathcal{C}_i\subset\mathcal{DS}$, fulfils that
\begin{equation}\label{fair_constraints}
\frac{|\{(x,s)\in \mathcal{C}_i: s=j\}|}{|\mathcal{C}_i|}=\frac{|\{(x,s)\in \mathcal{DS}: s=j\}|}{|\mathcal{DS}|} \quad \text{for every} \quad j\in S \quad \text{and}\quad i=1,\dots, k. 
\end{equation}
 In this case the proportion of individuals in any cluster for any protected class is the same as the respective proportion in the total data, hence we could say that proportions are independent of the cluster. This means that any decision taken with respect to a particular cluster will affect individuals in the same proportion as if it were taken for the entire population. Therefore, disparate impact for some sub-population would be avoided. Diversity enhancing clustering seeks to improve diversity with respect to a classical clustering procedure, without imposing the diversity preserving condition (\ref{fair_constraints}). Here, improving diversity means that clusters have more diverse composition with respect to the protected attributes compared with the clusters obtained by a classical clustering procedure.

Our approach to diversity enhancing clustering falls into the category of pre-processing. This means that we want to transform the unprotected data to obtain more diverse partitions. Since we are in an unsupervised setting, which means that there is no ground truth, looking for an adequate transformation of the data becomes very challenging. To address this problem, we propose a transformation based on a heuristic that incorporates some perturbations in the original dissimilarities of the data to favour heterogeneity with respect to the protected variable. For this, we present a new methodology inspired by electromagnetism and based on {\it attraction-repulsion dissimilarities}. These new dissimilarities aim at increasing the separation between points with the same values of the protected class and/or decreasing the separation between points with different values of the protected class.  Hence they favour the formation of clusters that are more heterogeneous, leading to a possible gain in diversity. The lack of ground truth also leads us to consider a trade-off between diversity and some other quality measures of the clusterings on the data at hand, in the spirit of Empirical Risk Minimization.

The main contributions of our work are the following:
\begin{itemize}
	\item The introduction of novel dissimilarities that perturb the distance in the unprotected attributes considering the protected ones.
	\item A new methodology based on a heuristic that promotes diversity and is applicable to almost any clustering algorithm with no or minimal modifications to the original objective functions. 
	\item A call for diversity and, more broadly, fairness  in clustering to be considered jointly with other relevant aspects that make a partition sensible.   
	\item A methodology that is easily adaptable to non-Euclidean data and clusters with non-convex boundaries.
	\item An R package to perform Attraction-repulsion clustering freely available at \\
	\href{https://github.com/HristoInouzhe/AttractionRepulsionClustering}{https://github.com/HristoInouzhe/AttractionRepulsionClustering} 
\end{itemize}

The structure of this work is the following. Section \ref{section_related_work} presents a brief overview of related works. In Section~\ref{section_mds} we introduce {\it attraction-repulsion} dissimilarities.  Clustering methods are developed in Section~\ref{section_chcl} while Section~\ref{sec_kkmeans} is devoted to their extension to non-Euclidean spaces through a kernel transformation. We study the tuning of our methods in Section~\ref{section_parameters}.
Experiments for our technique are given in Section \ref{section_applications}. We provide a discussion on synthetic datasets in \ref{sub_section_synthetic}. In Section \ref{section_comparison}, we provide comparison between our methods and the ones proposed in \cite{fair_k_means}, on synthetic and real data, showing the different behaviour of both procedures. In Section \ref{CRDC} we give a full example of attraction-repulsion clustering on a real dataset. Section \ref{section_conclusions} is devoted to some final remarks and some future work proposals.

\section{Related Work}\label{section_related_work}

Related works belong to the field of fair clustering. One of the first and most popular approaches to fair clustering is to impose constraints on the objective functions of classical clustering procedures, and corresponds to what we have called diversity preserving clustering. In \cite{fair_k_means} constrained $k$-center and $k$-median clustering was introduced. The authors proposed a model where data are partitioned into two groups, codified by red and blue, where disparate impact is avoided by maintaining \textit{balance} between the proportion of points in the two categories. Their goal is to achieve a partition of the data where the respective objective function is minimized while balance in each cluster is maintained over a given threshold.  Their approach initially designed for data with two different protected classes (two-state) has led to various extensions (multi-state). Contributions to the $k$-center problem can be found in \cite{rosner18, bercea2018}, to the $k$-median problem in~\cite{backurs2019} and to the $k$-means problem in~\cite{bercea2018,bera2019, schmidt2018}. These works discuss a broader range  of clustering problems while imposing constraints related to some minimum and/or maximum value for the proportions of the protected classes in each cluster. A multi-state deep learning fair clustering procedure was developed in \cite{Wang2019}. A multi-state non-disjoint extension was presented in \cite{Huang2019}. In \cite{Mazud2019} a variational framework of fair clustering compatible with several popular clustering objective functions was proposed.

Somewhat different approaches, that do not fall into diversity preserving clustering, can be found in the following works. In \cite{Ahmadian2019} the goal is to cluster the points to minimize the $k$-center cost but with the additional constraint
that no colour is over-represented in any cluster. \cite{Chen2019} cluster $n$
points with $k$ centers, with fairness meaning that any $n/k$ points are entitled to form their own cluster if there is a center that is closer in distance for all $n/k$ points. In \cite{Abbasi2021, Ghadiri2021} fair center based clustering is understood to be fair if the centers `represent' different groups equally well and they provide algorithms for this purpose.

Fair clustering is a vibrant field with a growing number of new contributions, so we do not claim this to be an exhaustive overview of relevant works. For more information we refer to the previously mentioned works and sources therein.

\section{Charged clustering via multidimensional scaling} \label{section_mds}
Clustering relies on the choice of dissimilarities that control the part of information conveyed by the data that will be used to gather points into the same cluster, expressing how such points share some similar features. To obtain a diversity enhancing clustering we aim at obtaining clusters which are not governed by the protected variables but are rather mixed with respect to these variables. For this, we introduce interpretable dissimilarities in the space $(X,S)\in \real^{d}\times\real^p$ that separate points with the same value of the protected classes. Using an analogy with electromagnetism, the labels $S$ play the role of an electric charge and similar charges tend to have a repulsive effect while dissimilar charges tend to attract themselves. We stress that $S\in\real^p$, hence it can take non-integer values, and all our procedures are well suited for such values. However, our examples concentrate on integer values for interpretability and simplicity.

Our guidelines for choosing these dissimilarities are that we would like the dissimilarities to 
\begin{itemize}
\item[\textit{i)}] induce diversity into subsequent clustering
techniques (decrease, or, at best, eliminate, identifiability between clusters and protected attributes), 
\item[\textit{ii)}] keep the essential geometry of the data (with respect to non-protected attributes) and
\item[\textit{iii)}] be easy to use and interpret. 
\end{itemize}

To fulfil $i)$, we introduce new dissimilarities that consider the geometry of the unprotected attributes  while penalizing homogeneity with respect to the protected ones. For $ii)$, we will combine these dissimilarities with multidimensional scaling, to retain most of the original geometry. With respect to $iii)$, we refer to the discussion below. Hence, we propose the following dissimilarities.
\begin{definition}[Attraction-Repulsion Dissimilarities]
\begin{equation} \label{k_add}
\delta_1\left(\left(X_1,S_1\right),\left(X_2,S_2\right)\right) = 1'U1 + S_1'VS_2 + \|X_1-X_2\|^2
\end{equation}
with $U, V$ symmetric matrices in $\real^{p\times p}$;
\begin{equation}\label{k_mult}
\delta_2\left(\left(X_1,S_1\right),\left(X_2,S_2\right)\right) = \left(1 + ue^{-v\|S_1-S_2\|^2}\right)\|X_1-X_2\|^2
\end{equation}
with $u,v\geq 0$;
\begin{equation}\label{k_norm}
\delta_3\left(\left(X_1,S_1\right),\left(X_2,S_2\right)\right) = \|X_1-X_2\|^2 - u\|S_1-S_2\|^2
\end{equation}
with $u\geq 0$. \\
Let $0\leq u\leq 1$ and $v,w\geq 0$, 
\begin{equation}\label{k_local}
\delta_4\left(\left(X_1,S_1\right),\left(X_2,S_2\right)\right) = \left(1 + \mathrm{sign}(S_1'VS_2)u\left(1-e^{-v(S_1'VS_2)^2}\right)e^{-w\|X_1-X_2\|}\right)\|X_1-X_2\|.
\end{equation}
\end{definition}

\begin{remark}
$\delta_1((X,S),(X,S))\neq 0$ and therefore it is not strictly a dissimilarity in usual sense, see, for example, Chapter 3 in \cite{Everittetal}. Yet, for all practical purposes discussed in this work this does not affect the proposed procedures.
\end{remark}

\begin{remark}
Both $\delta_1$ and $\delta_3$ could produce negative values. For this reason, we add a step that ensures positiveness in the algorithms at the end of this and the next sections. 
\end{remark}

\begin{remark}
All our dissimilarities can handle categorical variables as protected attributes $S$, once they are appropriately codified in numerical values. Furthermore, $\delta_2$ and $\delta_3$ are applicable to more general situations where the protected class is not discrete, as, for example, age or income. This is due to the treatment of $S$ as a quantitative variable in these dissimilarities.   
\end{remark}

To the best of our knowledge this is the first time that such dissimilarities have been proposed and used in the context of clustering (in \cite{repuls_clust} repulsion was introduced modifying the objective function, considering only distances between points, to maintain centers of clusters separated). 
Dissimilarities (\ref{k_add}) to (\ref{k_local}) are natural in the context of diversity enhancing clustering because they penalize the Euclidean distance considering the protected class of the points involved. Hence, some gains in diversity could be obtained.

The dissimilarities we consider are interpretable, providing the practitioner with the ability to understand and control the degree of perturbation introduced. Dissimilarity (\ref{k_add}) is an additive perturbation of the squared Euclidean distance where the intensity of the penalization is controlled by matrices $U$ and $V$, with $V$ controlling the interactions between elements of the same and of different classes $S$. Dissimilarity (\ref{k_norm}) presents another additive perturbation but the penalization is proportional to the difference between the classes $S_1$ and $S_2$, and the intensity is controlled by the parameter $u$.

Dissimilarity (\ref{k_mult}) is a multiplicative perturbation of the squared Euclidean distance. With $u$ we control the amount of maximum perturbation achievable, while with $v$ we modulate how fast we diverge from this maximum perturbation when $S_1$ is different to $S_2$.

Dissimilarity (\ref{k_local}) is also a multiplicative perturbation of the Euclidean distance. However, it has a very different behaviour with respect to (\ref{k_add})-(\ref{k_norm}). It is local, i.e., it affects less points that are further apart. Through $w$ we control locality. With bigger $w$ the perturbation is meaningful only for points that are closer together. With matrix $V$ we control interactions between classes as in (\ref{k_add}), while with $u$ we control the amount of maximum perturbation as in (\ref{k_mult}). Again, $v$ is a parameter controlling how fast we diverge from the maximum perturbation.

Next, we present a simple example for the case of a single binary protected attribute, coded as $-1$ or $1$. This is an archetypal situation in which there is a population with an (often disadvantaged) minority, that we code as $S=-1$, and the new clustering has to be independent (or not too dependent) on $S$.
\begin{example}\label{ex_kern}
Let us take $S_1,S_2\in\{-1,1\}$. For dissimilarity (\ref{k_add}) we fix $U=V=c\geq 0$, therefore 
\[\delta_1\left(\left(X_1,S_1\right),\left(X_2,S_2\right)\right) = c(1 + S_1S_2) + \|X_1-X_2\|^2.\]
If $S_1\neq S_2$, we have the usual squared distance $\|X_1-X_2\|^2$, while when $S_1=S_2$ we have $2c+\|X_1-X_2\|^2$, effectively we have introduced a repulsion between elements with the same class. For dissimilarity (\ref{k_mult}) let us fix $u=0.1$ and $v = 100$,
\[\delta_2\left(\left(X_1,S_1\right),\left(X_2,S_2\right)\right) = \left(1 + 0.1e^{-100\|S_1-S_2\|^2}\right)\|X_1-X_2\|^2.\]
When $S_1\neq S_2$, $\delta_2\left(\left(X_1,S_1\right),\left(X_2,S_2\right)\right) \approx\|X_1-X_2\|^2$, while when $S_1=S_2$, $\delta_2\left(\left(X_1,S_1\right),\left(X_2,S_2\right)\right) \approx 1.1\|X_1-X_2\|^2$, again introducing a repulsion between elements of the same class. For dissimilarity (\ref{k_norm}), when $S_1 = S_2$, $\delta_3\left(\left(X_1,S_1\right),\left(X_2,S_2\right)\right) =\|X_1-X_2\|^2$ and when $S_1\neq S_2$ we get $\delta_3\left(\left(X_1,S_1\right),\left(X_2,S_2\right)\right) =\|X_1-X_2\|^2-2u$, therefore we have introduced an attraction between different members of the sensitive class. When using dissimilarity (\ref{k_local}), fixing  $V = c > 0$, $u = 0.1$, $v = 100$, $w = 1$, we get
\[\delta_4\left(\left(X_1,S_1\right),\left(X_2,S_2\right)\right) = \left(1 + 0.1\mathrm{sign}(cS_1'S_2)\left(1-e^{-100(cS_1'S_2)^2}\right)e^{-\|X_1-X_2\|}\right)\|X_1-X_2\|.\]
If $S_1=S_2$, $\delta_4\left(\left(X_1,S_1\right),\left(X_2,S_2\right)\right)\approx\left(1+0.1e^{-\|X_1-X_2\|}\right)\|X_1-X_2\|$, therefore we have a repulsion. If $S_1\neq S_2$, $\delta_4\left(\left(X_1,S_1\right),\left(X_2,S_2\right)\right)\approx\left(1-0.1e^{-\|X_1-X_2\|}\right)\|X_1-X_2\|$, which can be seen as an attraction.\hfill$\Box$
\end{example}

Our proposals are flexible thanks to the freedom in choosing the class labels. If we encode $S$ as $\{-1,1\}$, as in the previous example, we can only produce attraction between different classes and repulsion between the same classes (or exactly the opposite if $V<0$) in (\ref{k_add}) and (\ref{k_local}). On the other hand, if we encode $S$ as $\{(1,0),(0,1)\}$, we have a wider range of possible interactions induced by $V$. For example, taking $V = ((1,-1)'|(-1,0)')$ we produce attraction between different classes, no interaction between elements labelled as $(0,1)$ and repulsion between elements labelled as $(1,0)$. If we had three classes we could use $\{(1,0,0),(0,1,0),(0,0,1)\}$ as labels and induce a tailor-made interaction between the different elements via a $3\times 3$ matrix $V$. For example, $V=((0,-1,-1)'|(-1,0,-1)'|(-1,-1,0)')$ provides attraction between different classes and no interaction between elements of the same class. Extensions to more than three classes are straightforward. More details on parameter and class codification selection will be given in Sections \ref{section_parameters} and \ref{section_applications}.

These dissimilarities can be used directly in some agglomerative hierarchical clustering methods, as described in Section~\ref{section_chcl}. However, there is a way to extend our methodology to a broader set of clustering methods. This is done using our dissimilarities to produce some embedding of the data into a suitable Euclidean space which allows the use of optimization clustering techniques (in the sense described in Chapter 5 in~\cite{Everittetal}) on the embedded data. For example, the dissimilarities $\delta_l$ can be combined with a common optimization clustering technique, such as $k$-means,  via some embedding of the data. We note that our dissimilarities aim at increasing the separation of points with equal values in the protected attributes while respecting otherwise the geometry of the data. Using multidimensional scaling (MDS) we can embed the original points in the space $\real^{d'}$ with $d'\leq d$ and use any clustering technique on the embedded data. The approach of using MDS to embed dissimilarities in an Euclidean space and then perform clustering has been pursued in \cite{Hausdorf2003} and \cite{Oh2007}. Quoting \cite{multid_scaling}, multidimensional scaling `is the search for a low dimensional space, usually Euclidean, in which points in the space represent the objects, one point representing one object, and such that the distances between the points in the space, match, as well as possible, the original dissimilarities'. Thus, applied to dissimilarities $\delta_l$, MDS will lead to a representation of the original data that approximates the original geometry of the data in the unprotected attributes and, at the same time, favours clusters with diverse values in the protected attributes. We stress that we are using MDS as a tool for embedding our dissimilarity relations in a Euclidean space, not necessarily as a dimension reduction technique. However, it is worth mentioning that MDS embeddings produce some loss of the original dissimilarity structure. Therefore, we advise to use embeddings in a space with similar dimension as the original data to minimize this adverse effect. In fact, in our experiments, we use the same dimension as that of the original data's unprotected attributes.

Here is an outline of how to use the dissimilarities $\delta_l$ coupled with MDS for a sample $(X_1,S_1),\dots,\allowbreak(X_n,S_n)$.
\begin{description}
\item[Attraction-Repulsion MDS] For any $l \in \{1,2,3,4\}$
\vskip .1in
\begin{itemize}
\item \textit{Compute the dissimilarity matrix $[\Delta_{i,j}]=[\delta_l((X_i,S_i),(X_j,S_j))]$ with a particular choice of the free parameters.}
\item \textit{If $\mathrm{min} \Delta_{i,j} \leq 0$, transform the original dissimilarity to have positive entries: \\$\Delta_{i,j} = \Delta_{i,j} + |\mathrm{min}\Delta| + \epsilon$, where $\epsilon$ is small.}
\item \textit{For $\delta_1,\delta_2,\delta_3$: $\Delta_{i,j} = \sqrt{\Delta_{i,j}}$.}
\item \textit{Use MDS to transform $(X_1,S_1),\dots,(X_n,S_n)$ into $X'_1, \dots, X'_n\in \real^{d'}$, where $D_{i,j} = \|X'_i-X'_j\|$ is similar to $\Delta_{i,j}$.}
\item \textit{Apply a clustering procedure on the transformed data $X'_1, \dots, X'_n$.}
\end{itemize}
\end{description}
This procedure will be studied in Section \ref{section_applications} for some synthetic and real datasets.
\section{Charged hierarchical clustering}\label{section_chcl}
Agglomerative hierarchical clustering methods (bottom-top clustering) encompass many of the most widely used methods in unsupervised learning. 
Rather than a fixed number of clusters, these methods produce a hierarchy of clusterings starting from the bottom level, at which each sample point constitutes a group, 
to the top of the hierarchy, where all the sample points are grouped into a single unit. We refer to \cite{hier_clust} for an overview. The main idea is simple. At each level, the two groups with the lowest dissimilarity are merged to form a single group. The starting point is typically a matrix of dissimilarities between pairs of data points. Hence, the core of a particular agglomerative hierarchical clustering lies at the way in which dissimilarities between groups are measured.
Classical choices include single linkage, complete linkage, average linkage or McQuitt's method. Additionally, some other very popular methods are readily available for using dissimilarities, as, for example, PAM (Partitioning Around Medoids, also known as $k$-medoids) introduced in \cite{Kaufman1987} and DBSCAN (Density-Based Spatial Clustering of Applications with Noise) introduced in \cite{dbscan}.

When a full data matrix (rather than a dissimilarity matrix) is available it is possible to use a kind of agglomerative hierarchical clustering
in which every cluster has an associated prototype (a center or centroid) and dissimilarity between clusters is measured through dissimilarity between the prototypes.
A popular choice (see \cite{Everittetal}) is \textit{Ward's minimum variance clustering}: dissimilarities between clusters are measured through a weighted squared Euclidean distance between mean vectors within each cluster. More precisely, if clusters $i$ and $j$ have $n_i$ and $n_j$ elements and mean vectors $g_i$ and $g_j$ then 
Ward's dissimilarity between clusters $i$ and $j$ is
$$\delta_{W}(i,j)={\textstyle \frac{n_i n_j}{n_i+n_j}}\| g_i -g_j\|^2,$$
where $\| \cdot\|$ denotes the usual Euclidean norm. Other methods based on prototypes are the centroid method or Gower's median method (see \cite{hier_clust}). However, these last two methods may present some undesirable features (the related dendrograms may present \textit{reversals} that make the interpretation harder, see, e.g., \cite{Everittetal}) and Ward's method is the most frequently used within this prototype-based class of agglomerative hierarchical clustering methods.

Hence, in our approach to diversity enhancing clustering we will focus on Ward's method. Given two clusters $i$ and $j$ 
consisting of points $\{(X_{i'},S_{i'})\}_{i'=1}^{n_i}$ and $\{(Y_{j'},T_{j'})\}_{j'=1}^{n_j}$, respectively, we define the charged dissimilarity between them as

\begin{equation}\label{dis_ward}
\delta_{W,l}(i,j) = \frac {n_i n_j}{n_i+n_j}\delta_l((\bar{X}_i,\bar{S}_i),  (\bar{Y}_j,\bar{T}_j))
\end{equation}
where $\delta_l$, $l=1,\ldots,4$ is any of the point dissimilarities defined by (\ref{k_add}) to (\ref{k_local}) and the bar notation refers to the standard sample mean over the points in the respective cluster.

The practical implementation of agglomerative hierarchical methods depends on the availability of efficient methods for the computation of dissimilarities between merged clusters. This is the case of the family of Lance-Williams methods (see \cite{lance-wil}, \cite{hier_clust} or \cite{Everittetal}) for which a recursive formula allows to update the dissimilarities when clusters $i$ and $j$ are merged into cluster $i\cup j$ in terms of the dissimilarities of the initial clusters. This allows to implement the related methods using computer time of order $O(n^2\log n)$. We show next that a recursive formula like the Lance-Williams class holds for the dissimilarities $\delta_{W,l}$ and, consequently, the related agglomerative
hierarchical method can be efficiently implemented. The fact that we are dealing differently with genuine and protected attributes results in the need for some additional notation (and storage). Given clusters $i$ and $j$ 
consisting of points $\{(X_{i'},S_{i'})\}_{i'=1}^{n_i}$ and $\{(Y_{j'},T_{j'})\}_{j'=1}^{n_j}$, respectively, we denote

\begin{equation}\label{dis_x}
d_x(i,j) = \big\|\bar{X}_i-\bar{Y}_j\|.
\end{equation}

Note that $d_x(i,j)$ is simply the Euclidean distance between the means of the $X$-attributes in clusters $i$ and $j$. Similarly, we set

\begin{equation}\label{dis_s}
d_s(i,j) = \big\|\bar{S}_i-\bar{T}_j\|.
\end{equation}

\begin{proposition}\label{lma_recurs_form}
For $\delta_{W,l}$ as in (\ref{dis_ward}), $d_x(i,j)$ as in (\ref{dis_x}) and $d_s(i,j)$ as in (\ref{dis_s}) and assuming that clusters $i, j$ and $k$ have sizes $n_i,n_j$ and $n_k$, respectively, we have the following recursive formulas:
\begin{enumerate}
\item[i)] $\delta_{W,1}(i\cup j,k) = \frac{n_i + n_k}{n_i+n_j + n_k}\delta_{W,1}(i,k) + \frac{n_j + n_k}{n_i+n_j+n_k}\delta_{W,1}(j,k)-\frac{n_k}{n_i+n_j+n_k}d^2_{W,x}(i,j);$
\item[ii)] \begin{eqnarray*}
\lefteqn{\delta_{W,2}(i\cup j,k) = \Big(1+ue^{-v (\frac{n_i}{n_i+n_j}d^2_s(i,k)+\frac{n_j}{n_i+n_j}d^2_s(j,k)-\frac{n_in_j}{(n_i+n_j)^2}d^2_s(i,j))}\Big)}\hspace*{3cm}\\
&\times &\Big({\textstyle \frac{n_i + n_k}{n_i+n_j + n_k}}d^2_{W,x}(i,k)
+{\textstyle\frac{n_j + n_k}{n_i+n_j+n_k}}d^2_{W,x}(j,k)-{\textstyle\frac{n_k}{n_i+n_j+n_k}}d^2_{W,x}(i,j)\Big);\\
\end{eqnarray*}
\item[iii)] $\delta_{W,3}(i\cup j,k) = \frac{n_i+n_k}{n_i+n_j+n_k}\delta_{W,3}(i,k) + \frac{n_j+n_k}{n_i+n_j+n_k}\delta_{W,3}(j,k) - \frac{n_k}{n_i+n_j + n_k}\delta_{W,3}(i,j)$,
\end{enumerate}
where $d^2_{W,x}(i,j)=\frac{n_i n_j}{n_i+n_j}d_x^2(i,j)$.
\end{proposition}
\begin{proof}
For \textit{i)} we just denote by $R_s$, $S_t$ and $T_r$ the protected attributes in clusters $i,j$ and $k$, respectively and  note that
\begin{align*}
\delta_{W,1}(i\cup j,k) &= {\textstyle \frac{(n_i+n_j)n_k}{n_i+n_j+n_k}}\Big(1'U1 + {\textstyle \frac{1}{n_i+n_j}\big(\sum_{s=1}^{n_i}R_s+\sum_{t=1}^{n_j}S_t\big)'V \bar{T}_k}+d^2_{x}(i\cup j, k)\Big)\\
&= {\textstyle \frac{(n_i+n_j)n_k}{n_i+n_j+n_k}\frac{n_i}{n_i+n_j}\Big(1'U1+\bar{R}_i'V\bar{T}_k\Big)}\\
&+ {\textstyle  \frac{(n_i+n_j)n_k}{n_i+n_j+n_k}\frac{n_j}{n_i+n_j}\Big(1'U1+\bar{S}_j'V\bar{T}_k\Big)+d^2_{W,x}(i\cup j, k)}\\
&={\textstyle \frac{n_i+n_k}{n_i+n_j+n_k}\frac{n_in_k}{n_i+n_k}\Big(1'U1+\bar{R}_i'V\bar{T}_k + d_{x}^2(i,k)\Big)}\\
&+ {\textstyle \frac{n_j+n_k}{n_i+n_j+n_k}\frac{n_jn_k}{n_j+n_k}\Big(1'U1+\bar{S}_j'V\bar{T}_k + d_x^2(j,k)\Big)}\\
&-{\textstyle \frac{n_k}{n_i+n_j+n_k}d^2_{W,x}(i,j)}\\
&={\textstyle \frac{n_i+n_k}{n_i+n_j+n_k}\delta_{W,1}(i,k)+\frac{n_j+n_k}{n_i+n_j+n_k}\delta_{W,1}(j,k)}-{\textstyle \frac{n_k}{n_i+n_j+n_k}d^2_{W,x}(i,j)}.
\end{align*}
Observe that we have used the well-known recursion for Ward's dissimilarities, namely,
\begin{equation}\label{Ward_recursion}
d_{W,x}^2(i\cup j, k)={\textstyle \frac{n_i+n_k}{n_i+n_j+n_k}d_{W,x}^2(i,k)+\frac{n_j+n_k}{n_i+n_j+n_k}d_{W,x}^2(j,k)-{\textstyle \frac{n_k}{n_i+n_j+n_k}d^2_{W,x}(i,j)}} 
\end{equation}
(see, e.g., \cite{Everittetal}). The update formulas \textit{ii)} and \textit{iii)} are obtained similarly. We omit details.

\end{proof}

From Proposition \ref{lma_recurs_form} we see that a practical implementation of agglomerative hierarchical clustering based on $\delta_{W,l}$, $l=1,2$ would require the computation of $d_{W,x}^2(i,j)$, which can be done using the Lance-Williams formula (\ref{Ward_recursion}).
In the case of $\delta_{W,2}$ we also need $d_s^2(i,j)$, which again can be obtained through a Lance-Williams recursion. This implies that 
agglomerative hierarchical clustering based on $\delta_{W,l}$, $l=1,2$ or $3$ can be implemented using computer time of order $O(n^2\log n)$ (at most twice the required time
for the implementation of a `standard' Lance-Williams method). 

We end this section with an outline of the implementation details for our proposal for diversity enhancing agglomerative hierarchical clustering based on dissimilarities $\delta_{W,l}$.
\begin{description}
\item[Iterative Attraction-Repulsion Clustering]
For  $l \in \{1,2,3\}$

\begin{itemize}
\item \textit{Compute the dissimilarity matrix $[\Delta_{i,j}]=[\delta_l((X_i,S_i),(X_j,S_j))]$ with a particular choice of the free parameters.}
\item \textit{If $\mathrm{min}\, \Delta_{i,j}\leq 0$, transform the original dissimilarity to have positive entries: $\Delta_{i,j} = \Delta_{i,j} + |\mathrm{min}\Delta| + \epsilon$, where $\epsilon$ is arbitrarily small.}
\item \textit{Use the Lance-Williams type recursion to determine the clusters $i$ and $j$ to be merged; iterate until there is a single cluster}
\end{itemize}
\end{description}
\section{Diversity enhancing clustering with kernels}\label{sec_kkmeans}
Clustering techniques based on the minimization of a criterion function typically result in clusters with a particular geometrical shape. For instance,
given a collection of points $x_1,\dots x_n \in\real^d$, the classical $k$-means algorithm looks for a grouping of the data into $K\leq n$ clusters $C = \{C_1,\dots, C_K\}$ with corresponding means $\{\mu_1,\dots, \mu_K\}$ such that the objective function 
\[\sum_{k=1}^{K}\sum_{x\in C_i}\|x-\mu_i\|^2\]
is minimized. The clusters are then defined by assigning each point to the closest center (one of the minimizing $C_i$'s). This results in convex clusters with linear boundaries. It is often the case that this kind of shape constraint does not adapt well to the geometry of the data. A non-linear transformation of the data could map some
clustered structure to make it more adapted to convex linear boundaries (or some other pattern). In some cases, this transformation can be implicitly handled via kernel methods. This approach is commonly called the `kernel trick'. The use of kernels in statistical learning is well documented, see, for instance, \cite{Cristianini2004} and Chapter 16 in \cite{Shalev2014}. For completeness, let us state that in this work a kernel is a symmetric and non-negative function  $\kappa: \real^{d}\times\real^{d}\to \real$. Additionally, a Mercer (positive semidefinite) kernel, is a kernel for which there is a transformation $\phi:\mathbb{R}^d \to \Omega$, with $\Omega$ a Hilbert space, such that $\kappa(x,y)=<\phi(x),\phi(y)>_{\Omega}$. In this section we explore how the attraction-repulsion dissimilarities that we have introduced can be adapted to the kernel clustering setup, focusing on the particular choice
of kernel $k$-means.

Kernel $k$-means is a non-linear extension of $k$-means that allows to find arbitrary shaped clusters introducing a suitable kernel similarity function, where the role of the squared Euclidean distance between two points $x,y$ in the classical $k$-means is taken by 
\begin{equation}
\label{d_kappa}
d_\kappa^2(x,y) = \kappa(x,x)+\kappa(y,y)-2\kappa(x,y).
\end{equation}
Details of this algorithm can be found in \cite{kkmeans}. 

In a first approach, we could try to introduce a kernel function for vectors $(X_1,S_1), (X_2,S_2)\in \real^{d+p}$ such that $d_\kappa^2$ takes into account the squared Euclidean distance between $X_1$ and $X_2$ but also tries to separate points of the same class and/or tries to bring closer points of different classes, i.e., makes use of $S_1,S_2$. Some simple calculations show that this is not an easy task, if possible at all in general. If we try, for instance, a joint kernel of type $\kappa((X_1,S_1),(X_2,S_2)) = \tau(S_1,S_2) + k(X_1, X_2)$, $S_1,S_2\in\{-1,1\}$ with $\tau, k$ Mercer (positive semi-definite) kernels (this covers the case $k(X_1,X_2)=X_1\cdot X_2$, the usual scalar product in $\mathbb{R}^d$), our goal can be written as 
\begin{equation}\label{desig_squared_distance}
d_\kappa^2((X_1,S_1),(X_2,S_1)) > d_\kappa^2((X_1,S_1),(X_2,S_2)),
\end{equation}
for any $X_1,X_2$, with $S_1\neq S_2$. However, (\ref{desig_squared_distance}) implies that
\[2\tau(S_1,S_2)> \tau(S_1,S_1)+\tau(S_2,S_2)\]
which contradicts that $\tau$ is positive-semi-definite. Therefore, there is no kernel on the sensitive variables that we can add to the usual scalar product. Another possibility is to consider a multiplicative kernel, $\kappa((X_1,S_1),(X_2,S_2)) = \tau(S_1,S_2)k(X_1, X_2)$, $S_1,S_2\in\{-1,1\}$ with $\tau, k$ Mercer kernels. From (\ref{desig_squared_distance}) we get
\[2\left(\tau(S_1,S_1)-\tau(S_1,S_2)\right)k(X_1, X_2)<\left(\tau(S_1,S_1)-\tau(S_2,S_2)\right)k(X_2,X_2)\]
which depends on $k(X_1,X_2)$ and makes it challenging to eliminate the dependence of the combinations $X_1,X_2$.

The previous observations show that it is difficult to think of a simple and interpretable kernel $\kappa$ that can be a simple combination of a kernel in the space of unprotected attributes and a kernel in the space of sensitive attributes. This seems to be caused by our desire to separate vectors that are similar in the sensitive space, which goes against our aim to use norms induced by scalar products. In other words, a naive extension of the kernel trick to our approach to diversity enhancing clustering seems to be inappropriate.

Nonetheless, the difficulty comes from a naive desire to carry out the (implicit) transformation of the attributes and the penalization of homogeneity in the protected
attributes in the clusters in a single step. We still may obtain gains in diversity, while improving the separation of the clusters in the unprotected attributes if we embed the $X$ data into a more suitable space by virtue of some sensible kernel $\kappa$ and consider the corresponding kernel version of $\delta_l$, with $\delta_l$ as in (\ref{k_add}) to (\ref{k_local}). Instead of using the Euclidean norm $\|X_1-X_2\|$ we should use $d_\kappa(X_1,X_2)$. In the case of $\delta_1$, for instance,
this would amount to consider the dissimilarity
\begin{equation} \label{k_add_kappa}
\delta_{\kappa,1}\left(\left(X_1,S_1\right),\left(X_2,S_2\right)\right) = 1'U1 + S_1'VS_2 + d_\kappa^2(X_1,X_2),
\end{equation}
with similar changes for the other dissimilarities. Then we can use an embedding (MDS the simplest choice) as in Section \ref{section_mds} and apply a clustering
procedure to the embedded data. This would keep the improvement in cluster separation induced (hopefully) by the kernel trick and apply, at the same time, a
diversity correction. An example of this adaptation of the kernel trick to our setting is given in Section \ref{k_trick_example}. We recall that our procedure inherits the possible benefits of the kernel trick, but also the difficulties related to it, for example the appropriate selection of kernels. 

\section{Parameter selection and tunning}\label{section_parameters}
There are two main steps required for tuning our methods. First, the user must select some criteria for diversity and some criteria for the quality of a clustering which are suitable to the problem at hand. Second, the practitioner should choose values in an informed way or just define a set of possible values. Then, an optimal selection of parameters and clustering methods can be done. A complete example of how to select the best parameters, the best dissimilarity, and the best clustering method (among some reasonable selection of methods) is provided in Section \ref{CRDC}. We do not advise trying to select the best method among all available clustering methods, since this seems rather infeasible.

\subsection{Diversity and quality of a clustering}
There are different measures for diversity, which, in a context where diversity is a good proxy for fairness, can also measure fairness. We will be mainly interested in two of them. As introduced in \cite{fair_k_means}, the balance of a set of points $X$ with protected attributes $S=\{red,black\}$ is defined as 
$$\mathrm{balance}(X)=\min\left(\frac{\# Black}{\# Red},\frac{\# Red}{\# Black}\right)$$ 
and the balance of $\mathcal{C}$, a clustering of the data in $X$, is given by 
\begin{equation}\label{balance}
\mathrm{balance}(\mathcal{C})=\min_{C\in\mathcal{C}}\mathrm{balance}(C).
\end{equation}
Let $\mathcal{C}$ be a clustering of a dataset into $K$ clusters. For each cluster $k$, $1\leq k\leq K$, there is an associated proportions vector $p_k$, where $p_k$ is formed by the proportion of each value of the protected attributes in the cluster $k$. A simple measure for the diversity of a partition is
\begin{equation}\label{unfiarness}
\mathrm{unbalance}(\mathcal{C}) = \frac{1}{K}\sum_{k=1}^K\|p_k-p_t\|,
\end{equation}
where $\|\cdot\|$ denotes the usual Euclidean norm and $p_t$ the vector of proportions of each value of the protected attributes in the whole dataset. Notice that $\mathrm{unbalance}(\mathcal{C}) = 0$ means that the partition  fulfils the diversity preserving criteria (\ref{fair_constraints}).

Due to the lack of ground truth in unsupervised learning and due to the characteristics of clustering, many notions of what a good clustering is and how to compare different partitions exist (see Section VI in \cite{HennigMeilaMurtaghRocci}). One popular approach is the Silhouette Index which measures how similar objects are in the same cluster with respect to objects in different clusters. Hence, a good partition has clusters that are cohesive and well separated. We recall that the silhouette index of an observation $X_i$ is given by 
\begin{equation}\label{silhouette}
s(i) = \frac{b(i)-a(i)}{\max(a(i),b(i))}
\end{equation}
where $a(i)$ is the average distance to $X_i$ of the observations in the same group as $X_i$, and $b(i)$ is the average distance to $X_i$ of the observations in the closest group different from the one of $X_i$ (see \cite{silhouette}). The average silhouette index of a group is the average of the silhouette indexes of the members of the group and the average silhouette index is the average of the silhouette indexes of all points.

On the other hand, in diversity enhancing clustering, a good benchmark for comparisons could be the partition obtained by a `standard' clustering procedure. A well-established measure for comparing partitions is the Adjusted Rand Index (ARI) (see \cite{Hubert1985}). Hence, there is a simple quantity that measures the difference between a diversity enhancing clustering and a standard one that we can interpret as the effect that a diversity correction has on the clustering structure.

\subsection{Parameter proposal}

The attraction-repulsion dissimilarities (\ref{k_add}) -
(\ref{k_local}) introduced in Section \ref{section_mds} depend on two main sets of parameters. The first set consist of free parameters used to balance the influence of the variables $X$ and the protected variable $S$. The second set is formed by the labels for the different classes, where different encoding of the labels allows for different interactions between the groups. Below, we propose some guidelines on the choice of these parameters focusing on how to understand them and what their effect on diversity is.  

The  dissimilarities we consider can be divided into two groups: (\ref{k_mult}) and (\ref{k_norm}) do not depend on the
codification of the class variable, while (\ref{k_add}) and
(\ref{k_local}) do depend on such a choice. In our method, the level of perturbation, which influences the level of diversity, is imposed through the choice of
the parameters in the dissimilarities. Choosing the parameters
enables us to balance diversity and  the original structure of the data which
may convey information that should not be erased by diversity constraints.

Consider first  dissimilarities (\ref{k_mult}) and (\ref{k_norm}). They rely on two parameters $u$ and $v$. In the multiplicative dissimilarity (\ref{k_mult}), $v$ is a parameter that measures how sudden the change in the distance is when switching  from elements with a different protected class to elements with the same protected class. For $v$ large enough,  $e^{-v\|S_1-S_2\|^2}$ is small when $S_1\neq S_2$, which implies that the diversity dissimilarity only modifies the distance between points inside the same protected class, increasing heterogeneity of the clusters.

Once $v$ has been fixed, the main parameter $u$ controls the intensity of the perturbation in a similar way for both dissimilarities (\ref{k_mult}) and (\ref{k_norm}). To illustrate the effect of this parameter, we focus on  (\ref{k_mult}) and perform a diversity enhancing
clustering, with MDS or hierarchically, for different values of the
intensity parameter $u$ and measure the diversity of the clusters
obtained. Such an example is depicted in the left middle row of Figure
\ref{fig_transf}. We can see that, as expected, increasing the values for $u$ puts more weight on the part of the dissimilarity that enforces heterogeneity of the clusters. $u=0$ leads to the usual clustering. Indeed, varying $u$ from $0$ to $4.5$ in steps of
$0.5$ increases the diversity achieved for both clusters, with a saturation effect  from
$4.5$ to $5$ where we do not appreciate an improvement in diversity. Hence, maximum diversity is achieved for $u = 4.5$
 and gives  the lowest perturbation
that achieves the highest level of diversity. If one aims at preserving more of the structure of the original
information at the expense of a lower level of diversity, some smaller value of $u$ can be selected. For example, in the right middle row of Figure \ref{fig_transf} we provide
the result of choosing $u=1$. Hence $u$ balances both effects of closeness to the usual dissimilarities and the amount of heterogeneity reached in the clustering.

Next,  dissimilarities (\ref{k_add}) and (\ref{k_local}), as described in
Section \ref{section_mds}, depend on the values chosen for the protected variable $S$, and a matrix $V$, which  plays the role of the matrix of interactions for different classes. When dealing with a two-class discrimination problem where the protected class has only two values, labelling the classes as
$\{-1,1\}$ or $\{(1,0),(0,1)\}$ can lead to the same results for
appropriate choices of $V$. However, for more than two protected classes
we will use only the following vectorial codification:  for $q$ different
values of the protected class, we will codify the values as the $q$
unitary vectors $\{a_1,\dots,a_q\}$ where $a_{i,j} =1$ if $i = j$ and
$a_{i,j} =0$ if  $i \neq j$ for $1\leq i,j\leq q$.

To build the interaction matrix we proceed as follows. First, consider a matrix $\tilde{V}_{i,j}$ with $1\leq i,j\leq q$. We fix
$\tilde{V}_{i,j} = 0$ if we want no interaction between classes $i$ and
$j$, in particular, if $i=j$ this means that there is no interaction
between elements with the same class. We take $\tilde{V}_{i,j} = 1$ if
we want repulsion (relative increase in the distance) between classes
$i$ and $j$. We fix $\tilde{V}_{i,j}=-1$ if we want attraction (relative
decrease in distance) between classes $i$ and $j$.  Hence, if the practitioner wants to promote diversity for a class represented by
$a_{i^*}$, it is recommendable to set values of $\tilde{V}_{i^*,j}=-1$
for $j\neq i^*$.  As an example, in Section \ref{k_trick_example}, we have chosen the interaction matrix
$V = ((1,-1)'|(-1,0)')$, to model repulsion between elements of the same class $(1,0)$, attraction between
elements of the classes $(1,0)$ and $(0,1)$, and no interaction between
the elements of the same class $(0,1)$. Then  intensity of the interaction is modelled using  a constant $v_0 > 0$, and we set  $V=v_0\tilde{V}$. In the previous example $v_0=1$.
The parameter $v$ for dissimilarity
(\ref{k_local}) has the same meaning as the corresponding parameter for
(\ref{k_mult}) and can be selected in the same way.

Finally, matrix $U$ for
dissimilarity (\ref{k_add}) represents an extra additive shift.  In many
cases it can be set to $U = 0$ (the zero matrix).

We provide an example to explain how  to select the intensity $v_0$ for dissimilarity (\ref{k_add}) in  the top left image of Figure \ref{fig_transf}. Notice that using
$V>0$ and $S\in\{-1,1\}$ is the same as using $V=v_0\tilde{V}$ with
$\tilde{V} = ((1,-1)'|(-1,1)')$ and $S\in\{(1,0),(0,1)\}$. We plot the variation of diversity in each cluster when we vary the
intensity of the interaction between  $0$ and $4.4$ with steps of size  $0.44$.
There is a steady improvement in diversity in both clusters until the
intensity reaches  $v_0 = 3.52$, but from this level, as previously, there is no more
improvement in diversity. Therefore, if a practitioner wants to achieve
the highest level of diversity, $v_0 = 3.52$ should be the proper
intensity, since it corresponds to the smallest perturbation to the geometry that
achieves the best observed diversity. However, a smaller correction in
diversity may be of interest, we have a representation of that top right
in Figure \ref{fig_transf} for $v_0 = 1.32$.

For dissimilarity (\ref{k_local}), after choosing the interaction matrix
$\tilde{V}$, we can try to find a maximum in diversity, fixing a grid
formed by different combinations for the vector of parameters
$(v_0,u,w)$. In the second and third column of Table \ref{table_k_trick}
we see the diversity of the respective clusters when we look at the grid
$(1,u,0.05)$ with $u \in \{0,0.098,\dots, 0.98\}$. What we notice is an
improvement in diversity for all values of $u$, therefore a practitioner
would be advised to select $u = 0.98$ where we obtain the best diversity
values.

\section{Experiments} \label{section_applications}
In this section we provide examples of attraction-repulsion clustering and some insights into the complex relation between diversity and other properties of clustering.

In the first two subsections we deal with simulated examples since we mainly want to describe how attraction-repulsion clustering works, how it compares to some diversity preserving clustering procedures and how diversity relates to cluster structure. To the best of our knowledge, this last point seems to be overlooked in the fair clustering literature, whenever diversity is the proxy used for fairness, in favour of obtaining high fairness with respect to some fairness measure.

The last subsection is a full example on a non-trivial real data set where full tuning of our methods is provided and where preservation of the geometrical structure is of crucial importance. We use a range of different clustering procedures to illustrate that our methods are, as claimed, easily adaptable to a wide range of clustering algorithms. We do not advocate for any particular clustering method, but we do try to use some of the extensions we have developed in Section \ref{section_chcl}.

\subsection{Synthetic data} \label{sub_section_synthetic}
\subsubsection{Diversity and cluster structure} \label{general_example}
Our first example shows how our procedures behave and gives an intuition of how diversity can affect cluster structure. We generate 50 points from four distributions,
$$\mu_1 \sim N((-1,0.5),\mathrm{diag}(0.25,0.25)),\mu_2 \sim N((-1,-0.5),\mathrm{diag}(0.25,0.25));$$
$$\mu_3 \sim N((1,0.5),\mathrm{diag}(0.25,0.25)),\mu_4 \sim N((1,-0.5),\mathrm{diag}(0.25,0.25)),$$ and label the samples from $\mu_1$ and $\mu_2$ as $S = 1$ (squares) and the samples from $\mu_3$ and $\mu_4$ as $S = -1$ (circles). A representation of the data in the original space is given in the third column of Figure \ref{fig_transf}. We can think of the $x$ direction of the data as a source of a lack of diversity, therefore any sensible clustering procedure is going to have clusters that are highly homogeneous in the class $S$ when the original coordinates are used. This is exemplified in Table \ref{tabla_bias}, as we look for different number of clusters: with $k$-means we are detecting almost pure groups (1st row); the same happens with a complete linkage hierarchical clustering with the Euclidean distance (5th row) and with Ward's method with the Euclidean distance (9th row).

Therefore, it may be useful to apply our procedures to the data to gain diversity in $S$. In the first column of Figure \ref{fig_transf} we study the gain in diversity from the increase in intensity of the corrections we apply, the relation between the gains in diversity and the disruption of the cluster structure of the original classes after MDS, and the relation between stronger correction and the change with respect to the original $k$-means partition measured by ARI. In the first row we use dissimilarity (\ref{k_add}), where we fix $U = 0$, and we vary $V= 0,0.44,0.88,\dots,4.4$. In the second row we work with dissimilarity (\ref{k_mult}), where we fix  $v = 20$ and set $u = 0,0.5,1,\dots,5$. In the last row we work with dissimilarity (\ref{k_local}) fixing $V = 1,\,v = 20,\, w = 1$ and we vary $u = 0, 0.099,0.198,\dots,0.99$. We do not show results for dissimilarity (\ref{k_norm}), since in this example it gives results very similar to dissimilarity (\ref{k_add}). With some abuse of notation, throughout Section \ref{section_applications} we will use $S$ as the name of the protected variable. Solid and dashed black lines represent, respectively, the proportion of class $S = 1$ in the two clusters found by $k$-means after the MDS transformation. Solid and dashed red lines represent, respectively, the average silhouette index of class $S=1$ and class $S=-1$. The green line is the Adjusted Rand Index between the $k$-means partition obtained in the original data and the different corrected partitions.

What we see top-left and middle-left in Figure \ref{fig_transf} is that higher intensity (higher $V$ an $u$, respectively) relates to greater diversity and to significantly different partitions compared to the original $k$-means clustering, but also relates to lower silhouette index. This can be interpreted as the fact that greater intensity in dissimilarities (\ref{k_add}) and (\ref{k_mult}) has a greater impact in the geometry of the original problem. In essence, the greater the intensity, the more indistinguishable $S = 1$ and $S= -1$ become after MDS, therefore, any partition with $k$-means will result in very diverse clusters in $S$. This is equivalent to saying that diversity is essentially destroying the original cluster structure of the data, since it is mainly due to the variable $x$. Hence, here a diversity enhancing partition and a good partition are competing against each other. By construction this is not what happens with dissimilarity (\ref{k_local}). The strong locality penalty ($w = 1$) allows to conserve the geometry, shown by the little reduction in silhouette index (row 3 column 1), but results in smaller corrections in the proportions and in small differences with the original $k$-means clustering.

\begin{figure}
\caption{Top row: dissimilarity (\ref{k_add}). Middle row: dissimilarity (\ref{k_mult}). Bottom row: dissimilarity (\ref{k_local}). Left column: proportions of  $S = 1$ in each cluster (black lines), average silhouette indexes for $S = 1$ and $S = -1$ in the transformed space (red lines) and ARI with respect to the $k$-means in the original space (green line), for input parameters that vary. The $x$ label represents $i$ where: $\{V_i\}_{i=1}^{11}=\{ 0,0.44,0.88,\dots,4.4\}$ for dissimilarity (\ref{k_add}); $\{u_i\}_{i=1}^{11} = \{0,0.5,1,\dots,5\}$ for dissimilarity (\ref{k_mult}); $\{u_i\}_{i=1}^{11} = \{0, 0.099,0.198,\dots,0.99\}$ for dissimilarity (\ref{k_local}). Middle column: two clusters in the transformed space for a particular choice of parameters. Right column: same two clusters in the original space.}
\label{fig_transf}
\begin{center}
\includegraphics[scale=0.21]{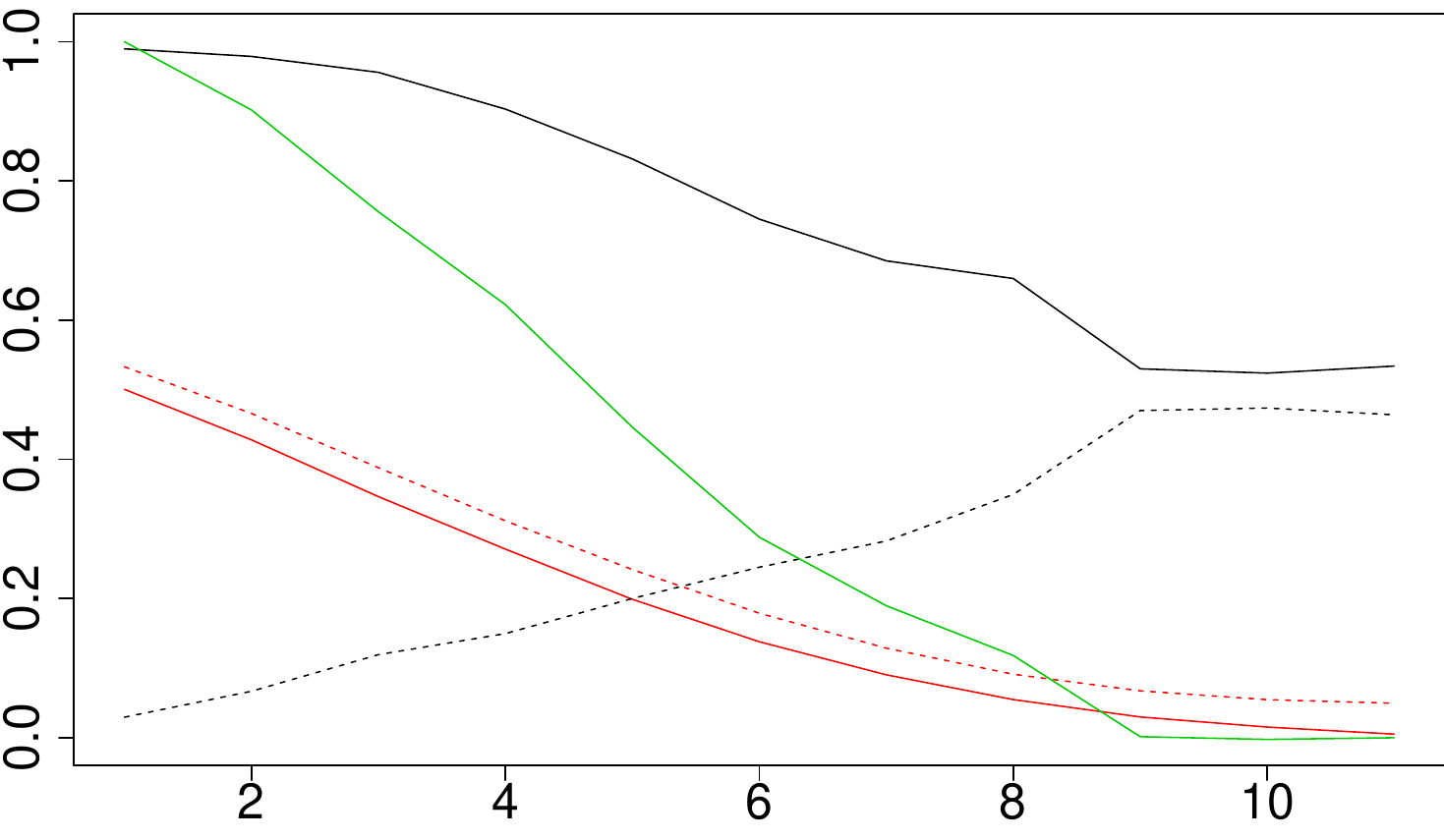}\includegraphics[scale=0.21]{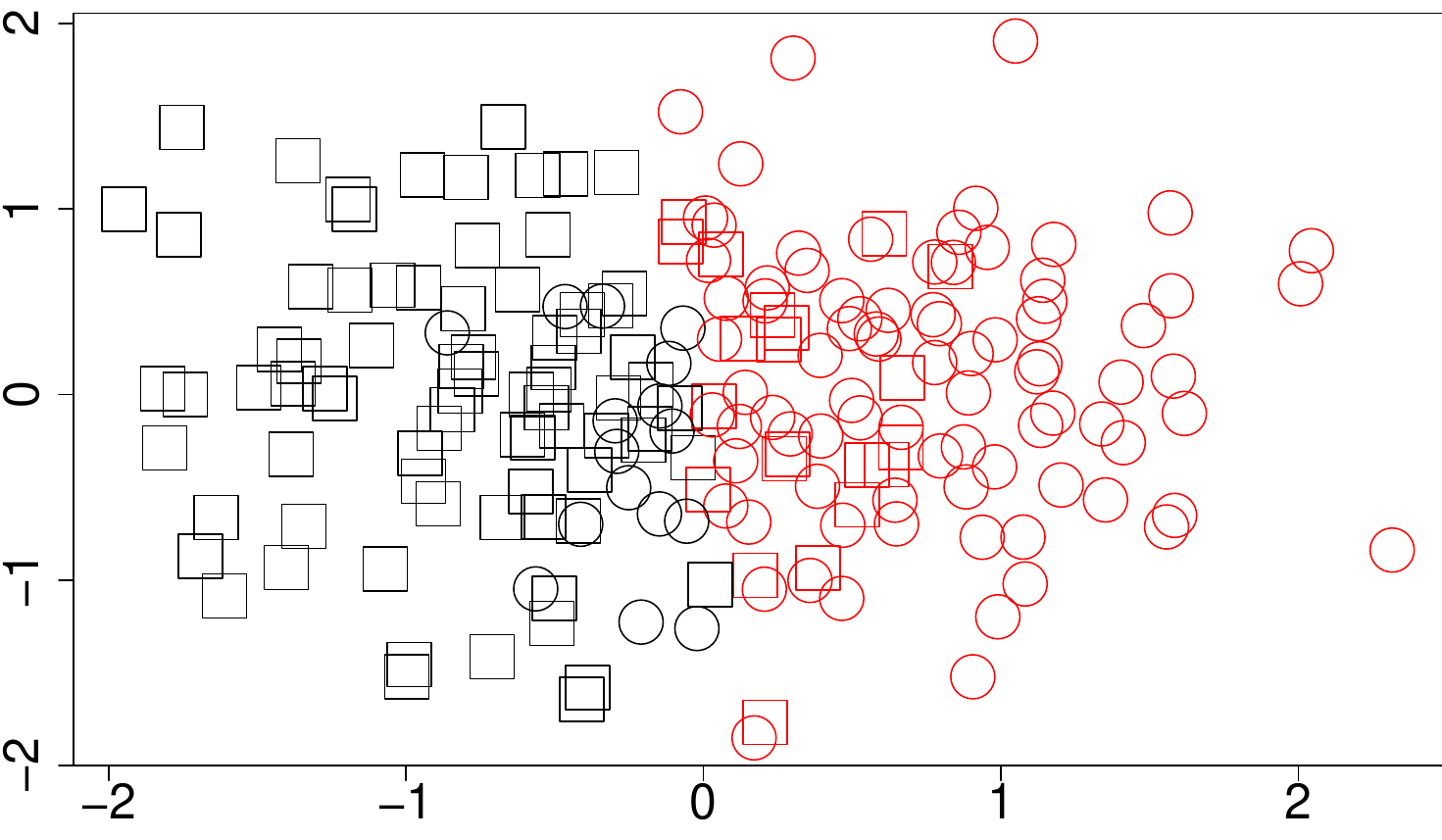}\includegraphics[scale=0.21]{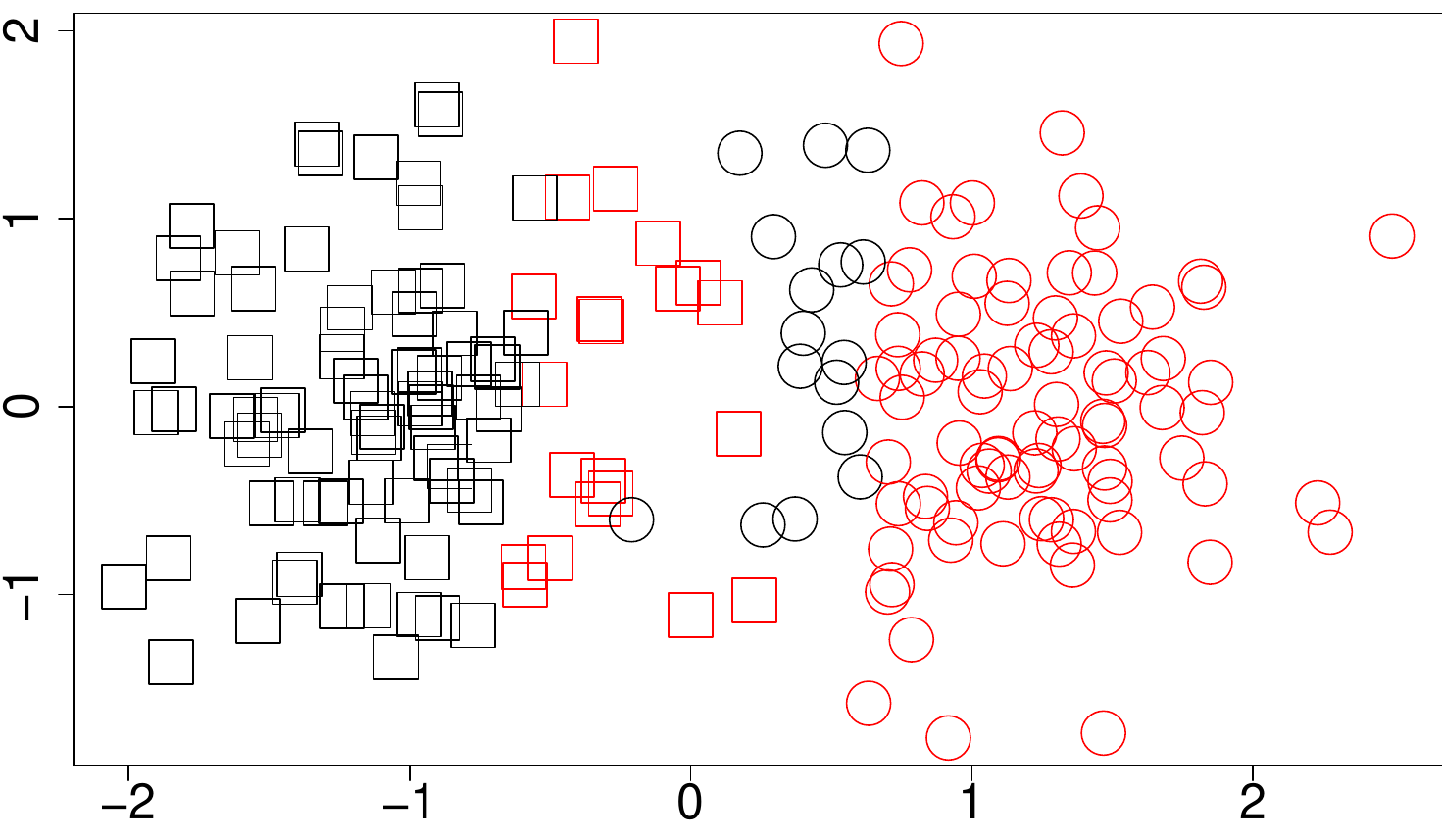}
\includegraphics[scale=0.21]{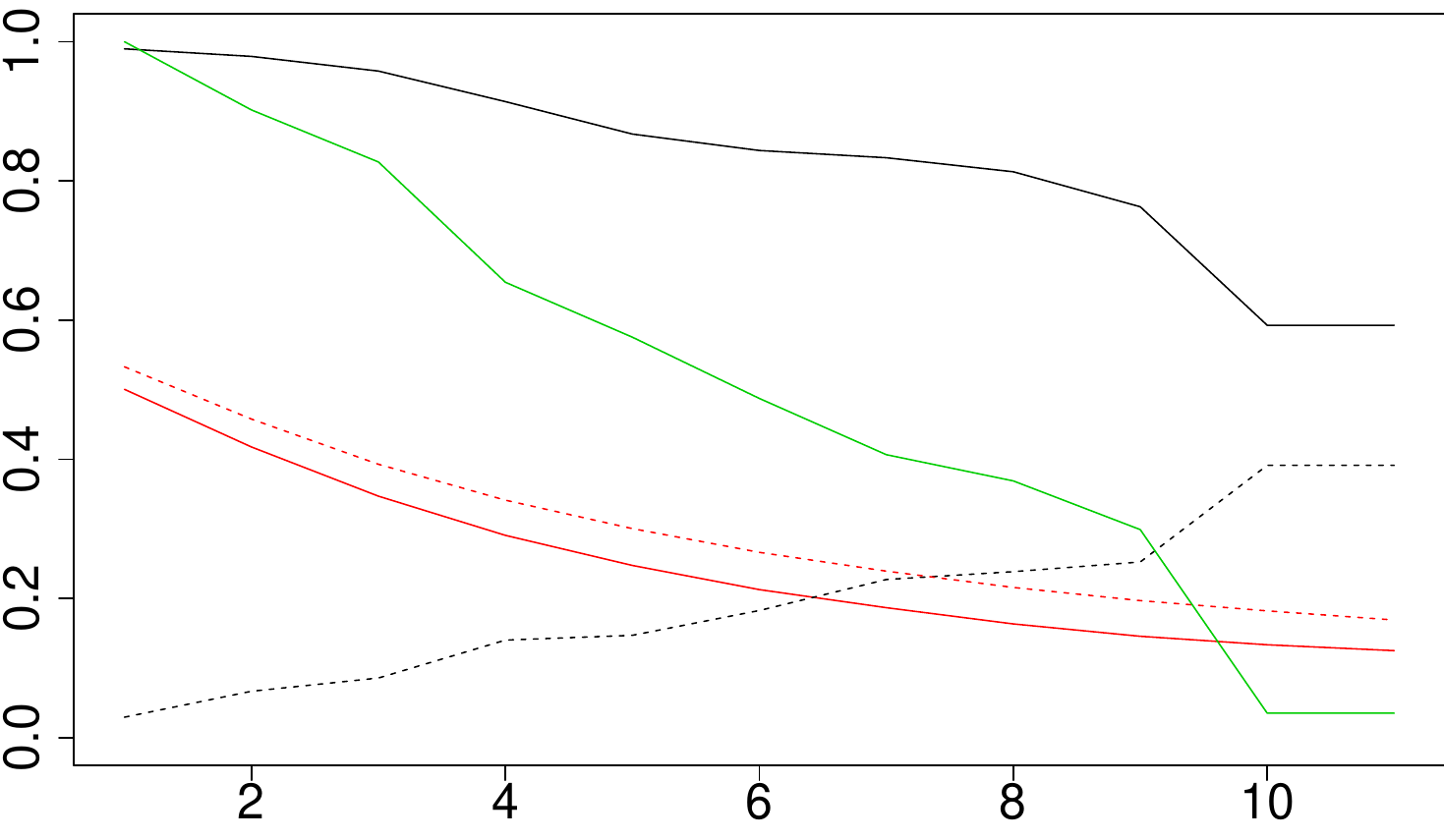}\includegraphics[scale=0.21]{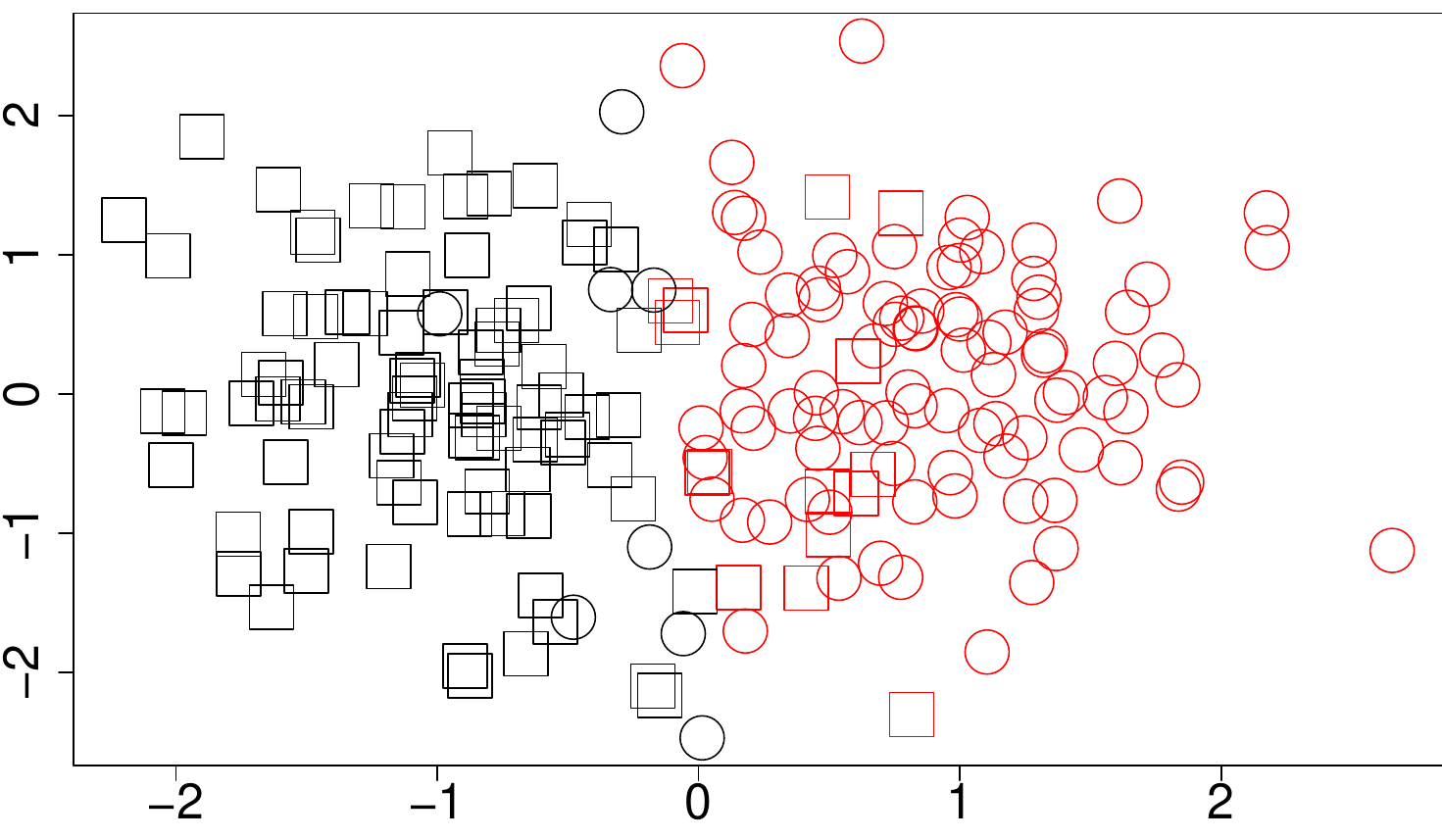}\includegraphics[scale=0.21]{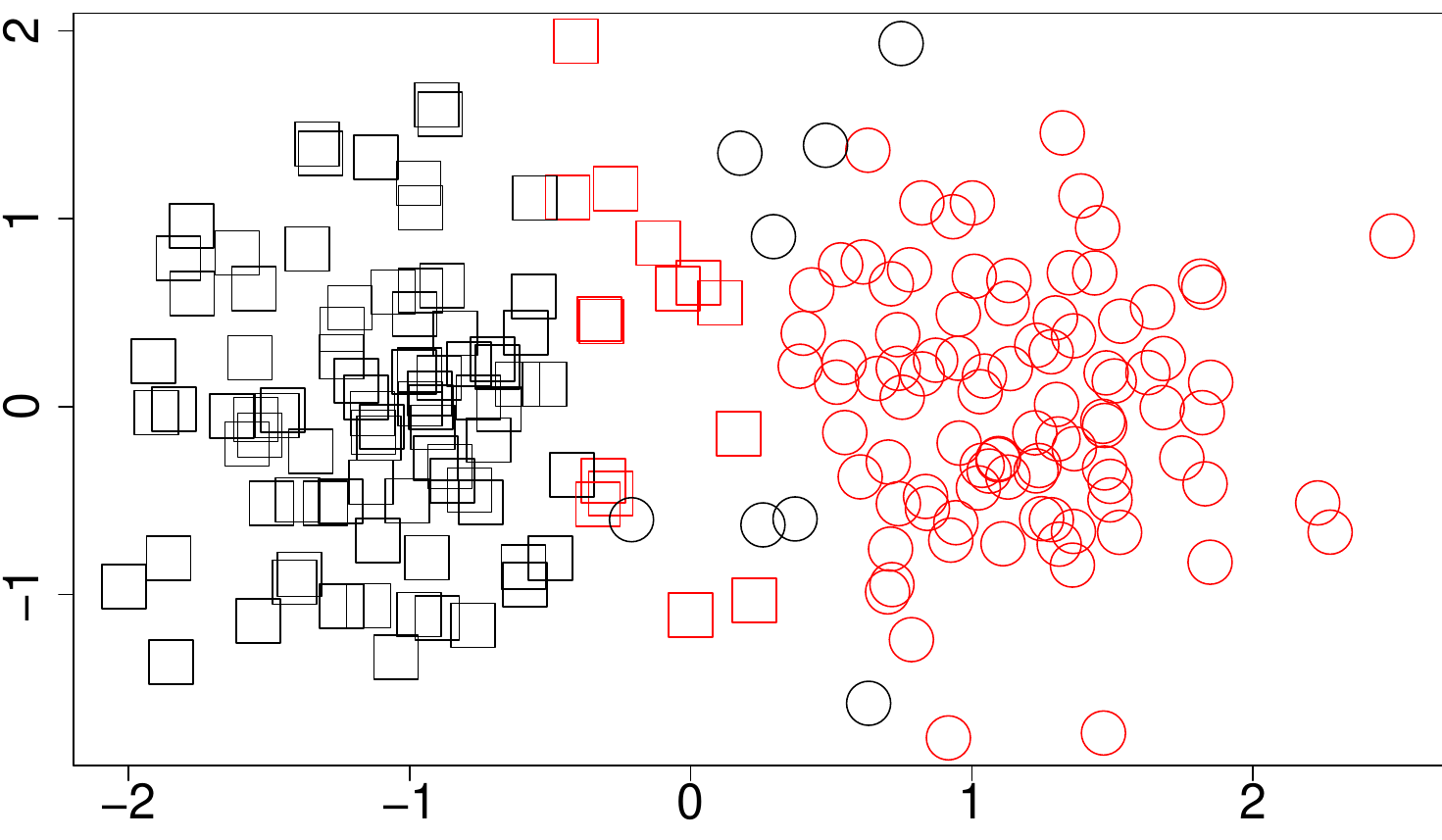}
\includegraphics[scale=0.21]{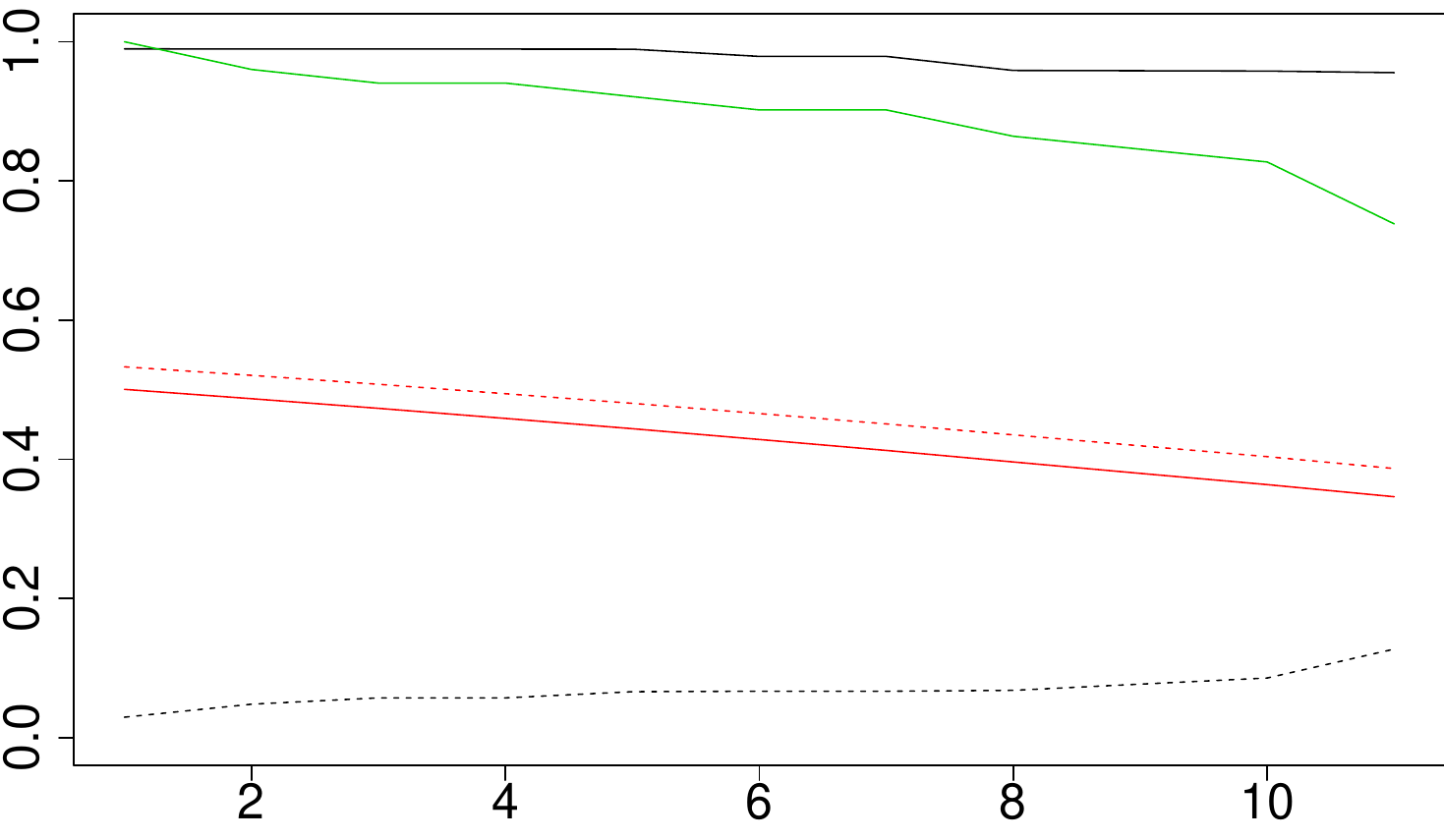}\includegraphics[scale=0.21]{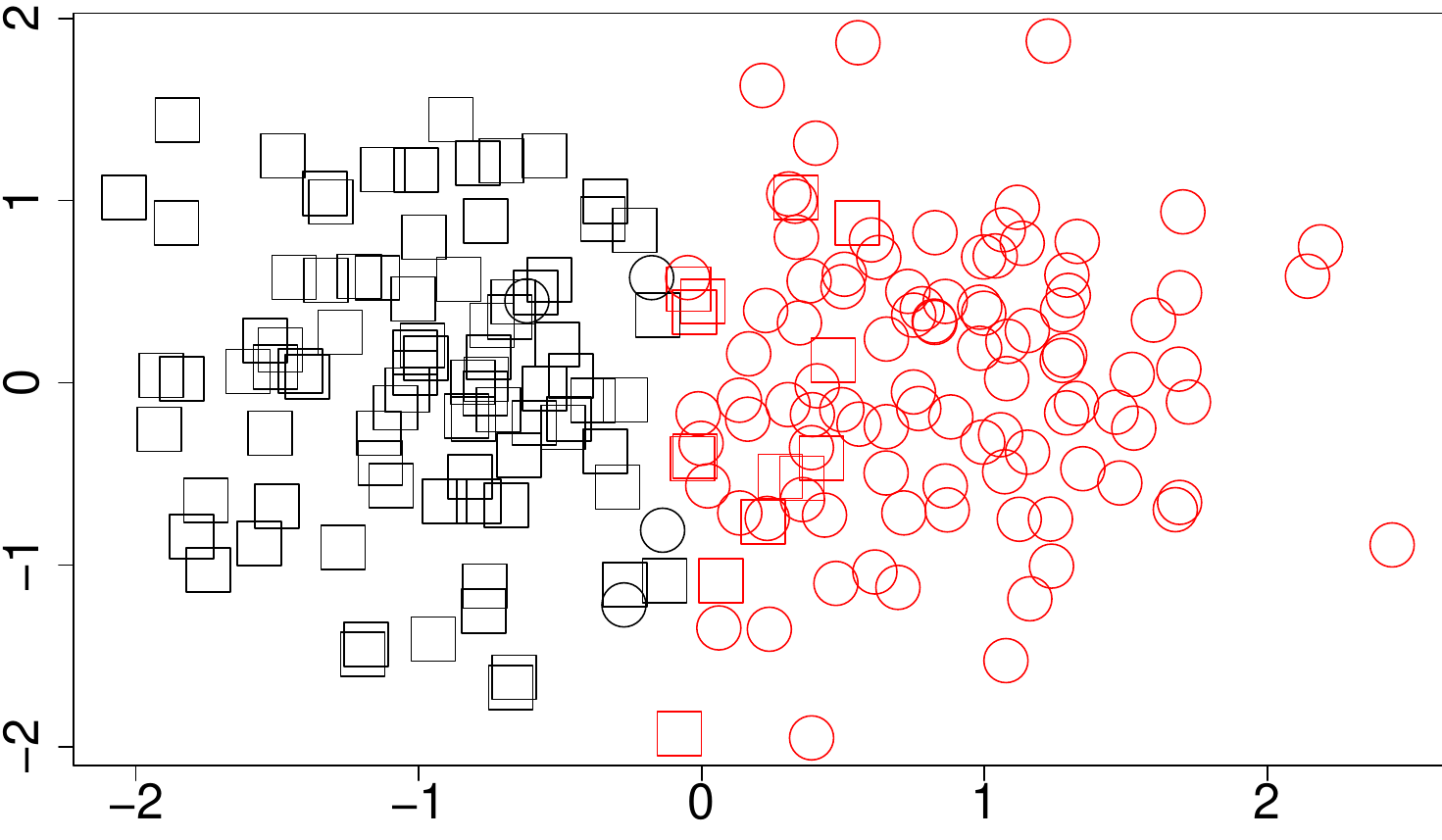}\includegraphics[scale=0.21]{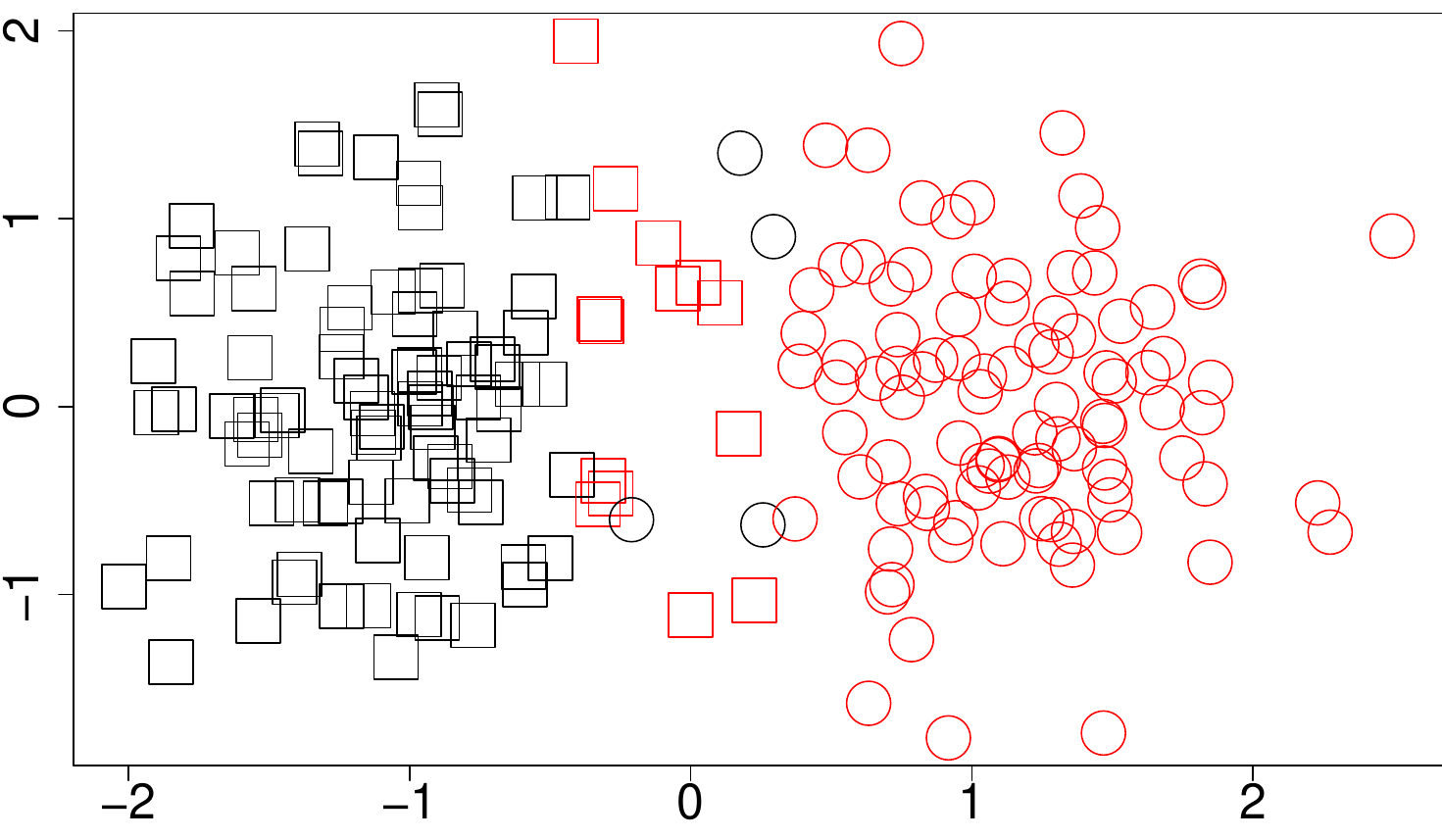}
\end{center}
\end{figure}

\begin{figure}
\caption{Top row: dissimilarity (\ref{k_add}). Middle row: dissimilarity (\ref{k_mult}). Bottom row: dissimilarity (\ref{k_local}). Left column: proportions of  $S = 1$ in the clusters. Middle column: average silhouette indexes for $S = 1$ in the transformed space. Right column: ARI index with respect to the unperturbed $k$-means partition. Curves represent averages over 20 samples. Colours correspond to an increase in perturbation on the label space, ranging from red (lowest amount) to magenta (highest amount). The $x$ label represents $i$ where: $\{V_i\}_{i=1}^{11}=\{ 0,0.44,0.88,\dots,4.4\}$ for dissimilarity (\ref{k_add}); $\{u_i\}_{i=1}^{11} = \{0,0.5,1,\dots,5\}$ for dissimilarity (\ref{k_mult}); $\{u_i\}_{i=1}^{11} = \{0, 0.099,0.198,\dots,0.99\}$ for dissimilarity (\ref{k_local}).}
\label{fig_transf_pert}
\begin{center}
\includegraphics[scale=0.21]{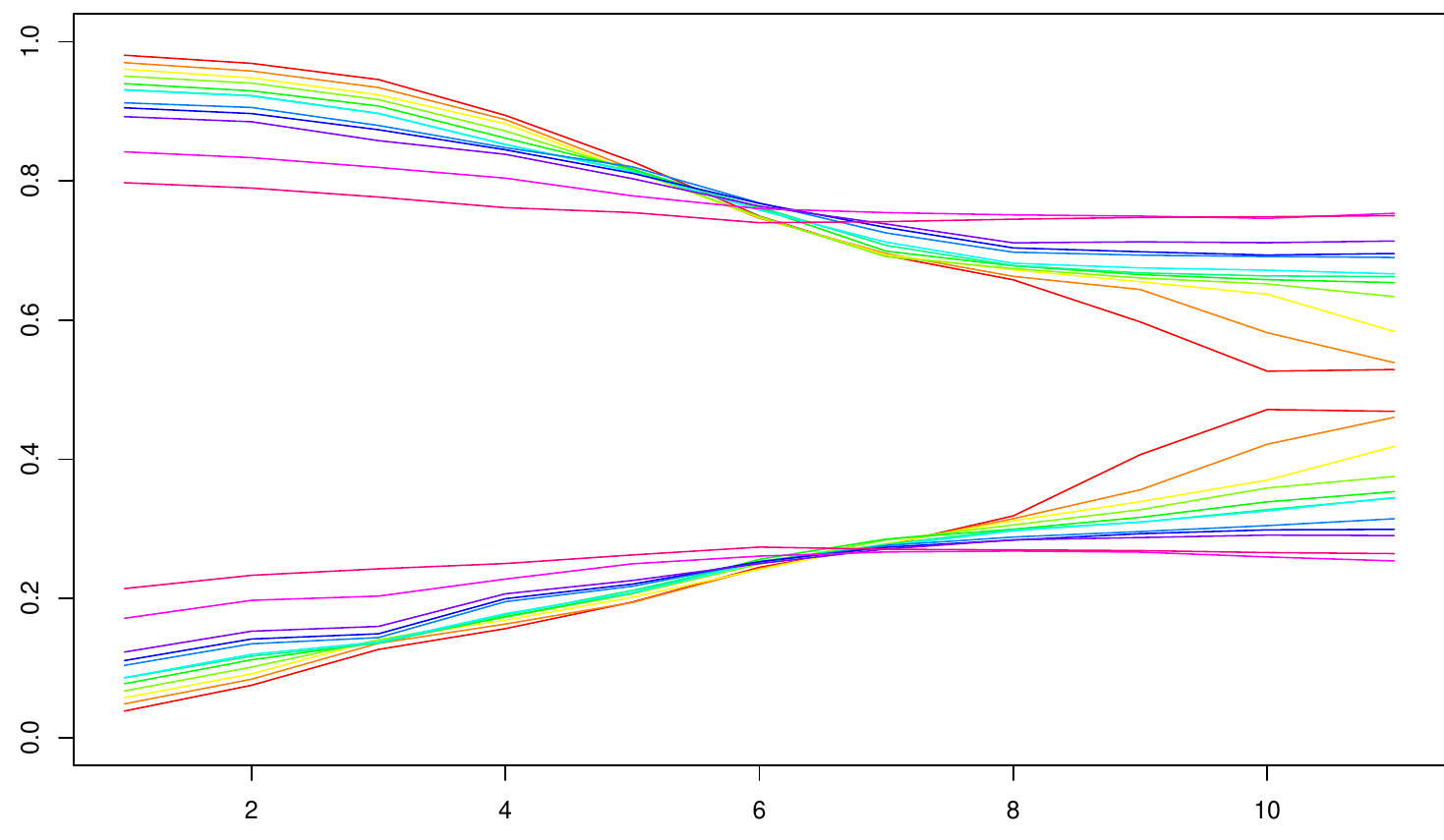}\includegraphics[scale=0.21]{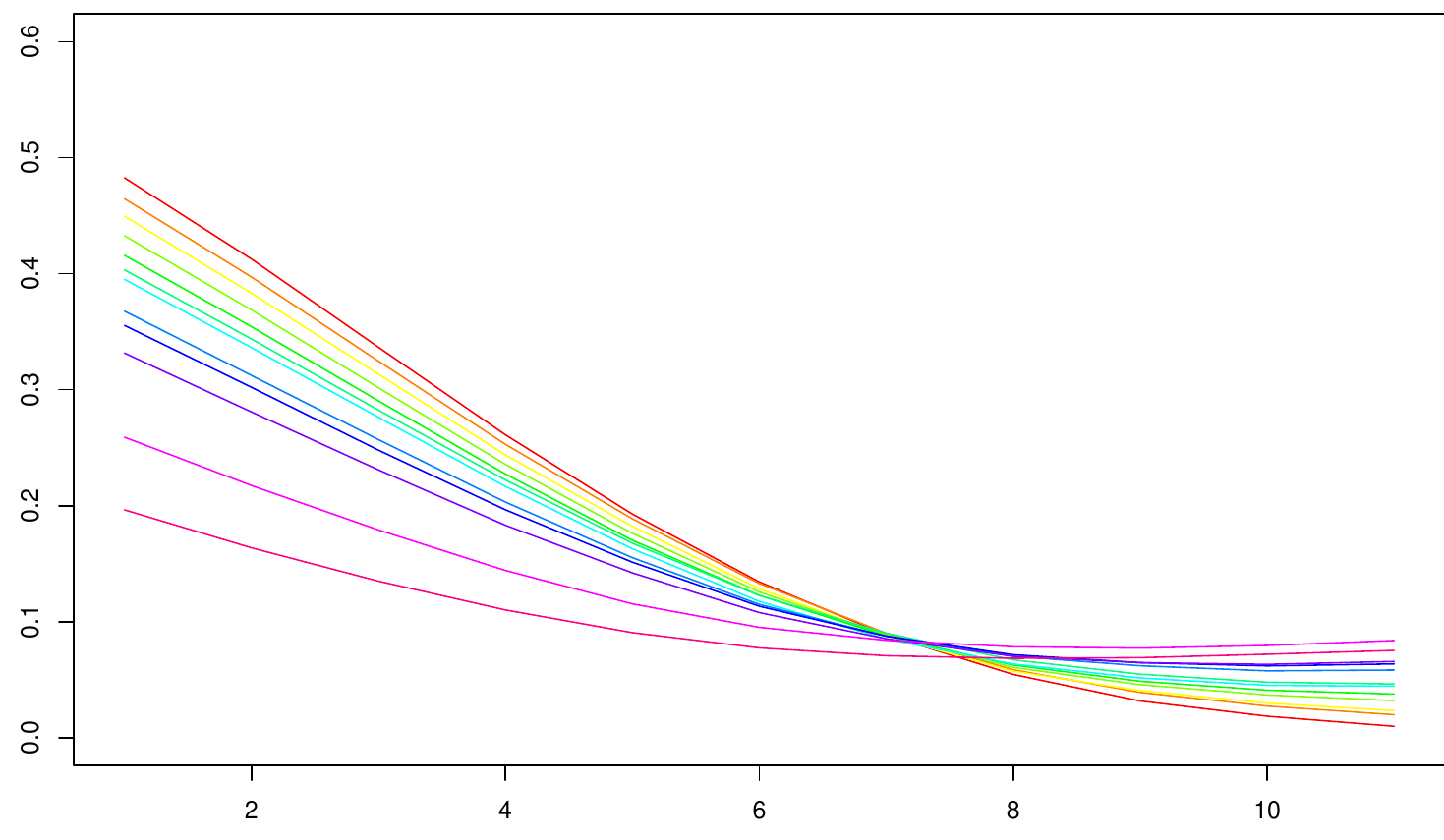}\includegraphics[scale=0.21]{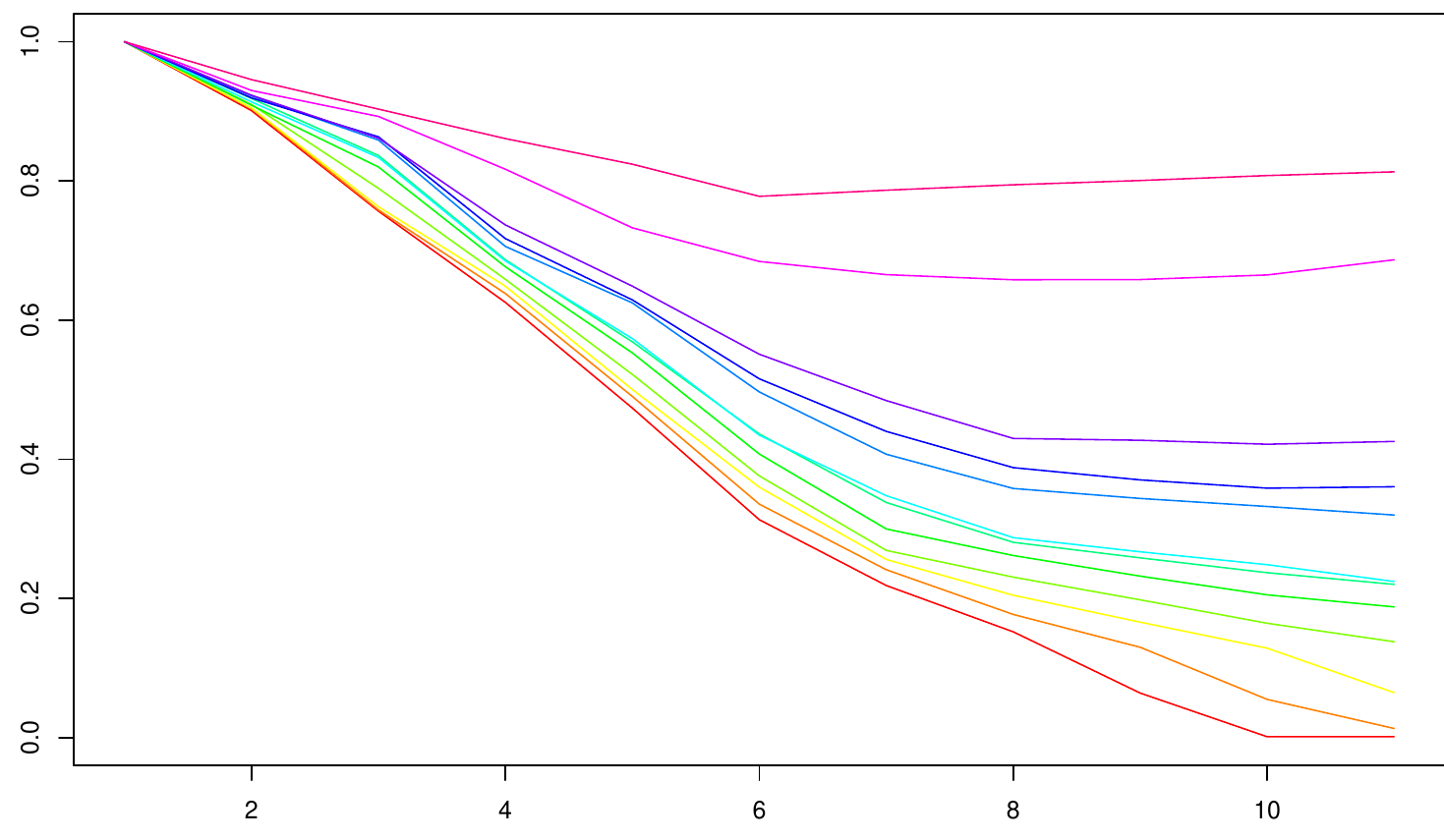}
\includegraphics[scale=0.21]{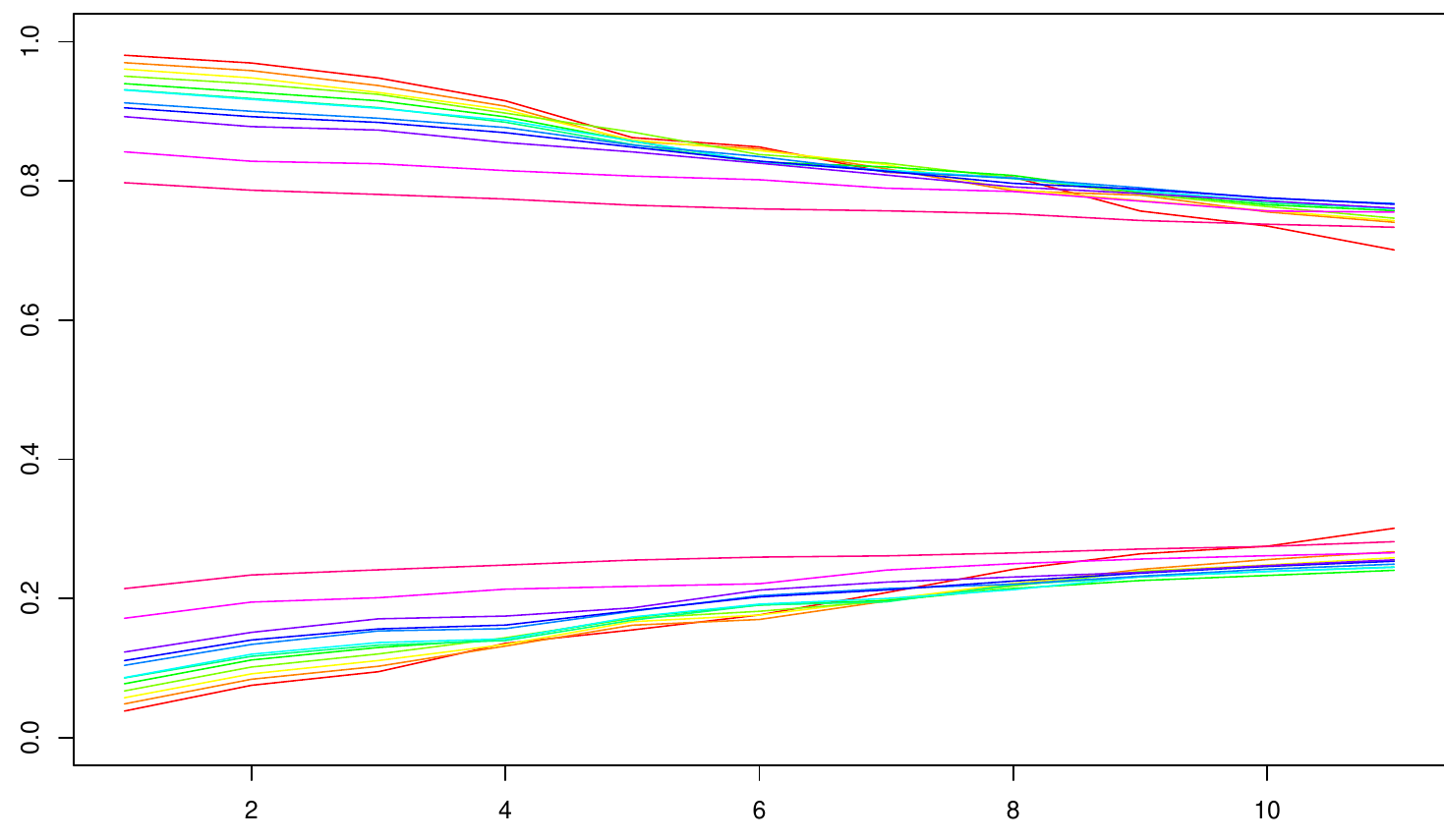}\includegraphics[scale=0.21]{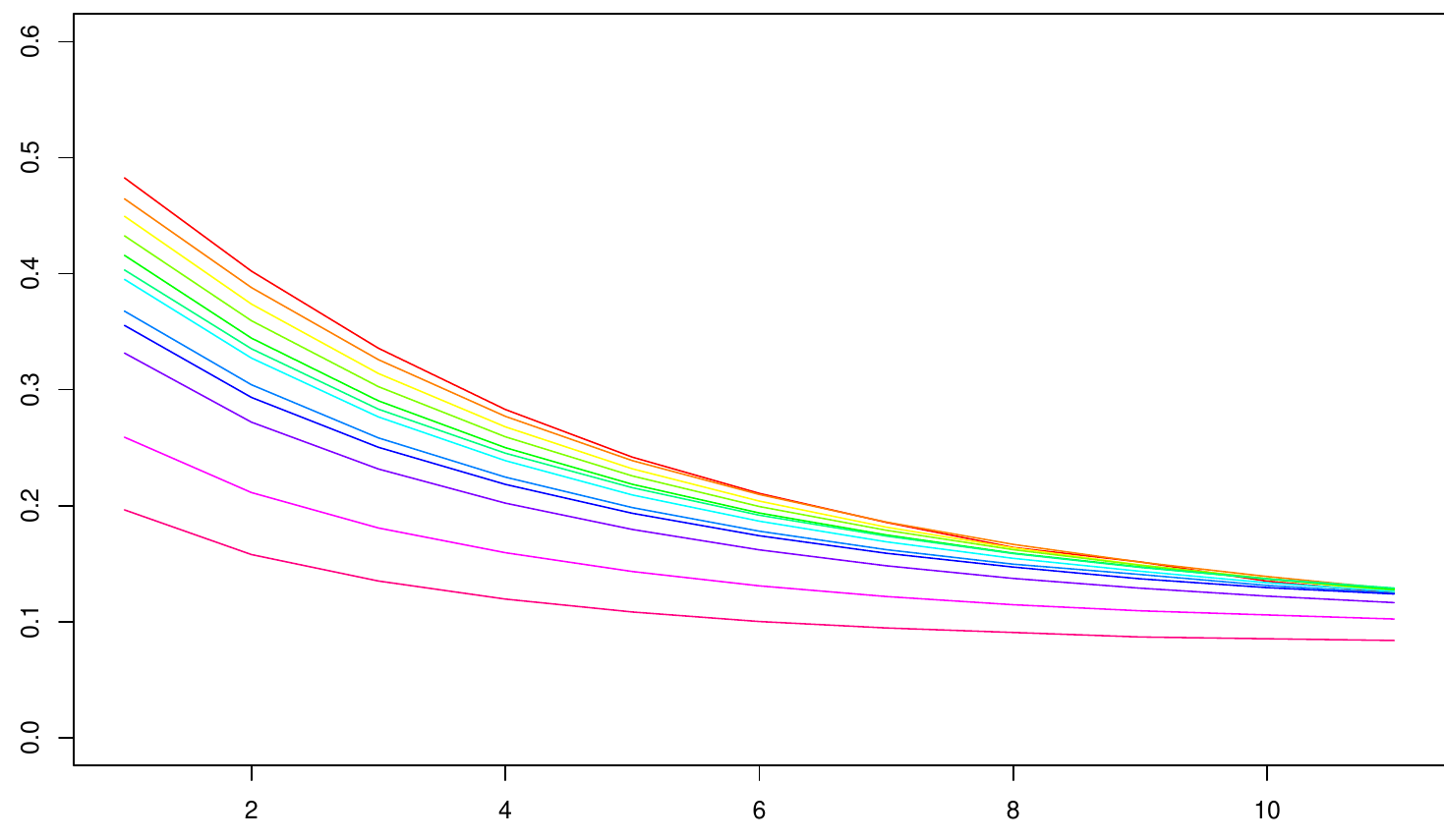}\includegraphics[scale=0.21]{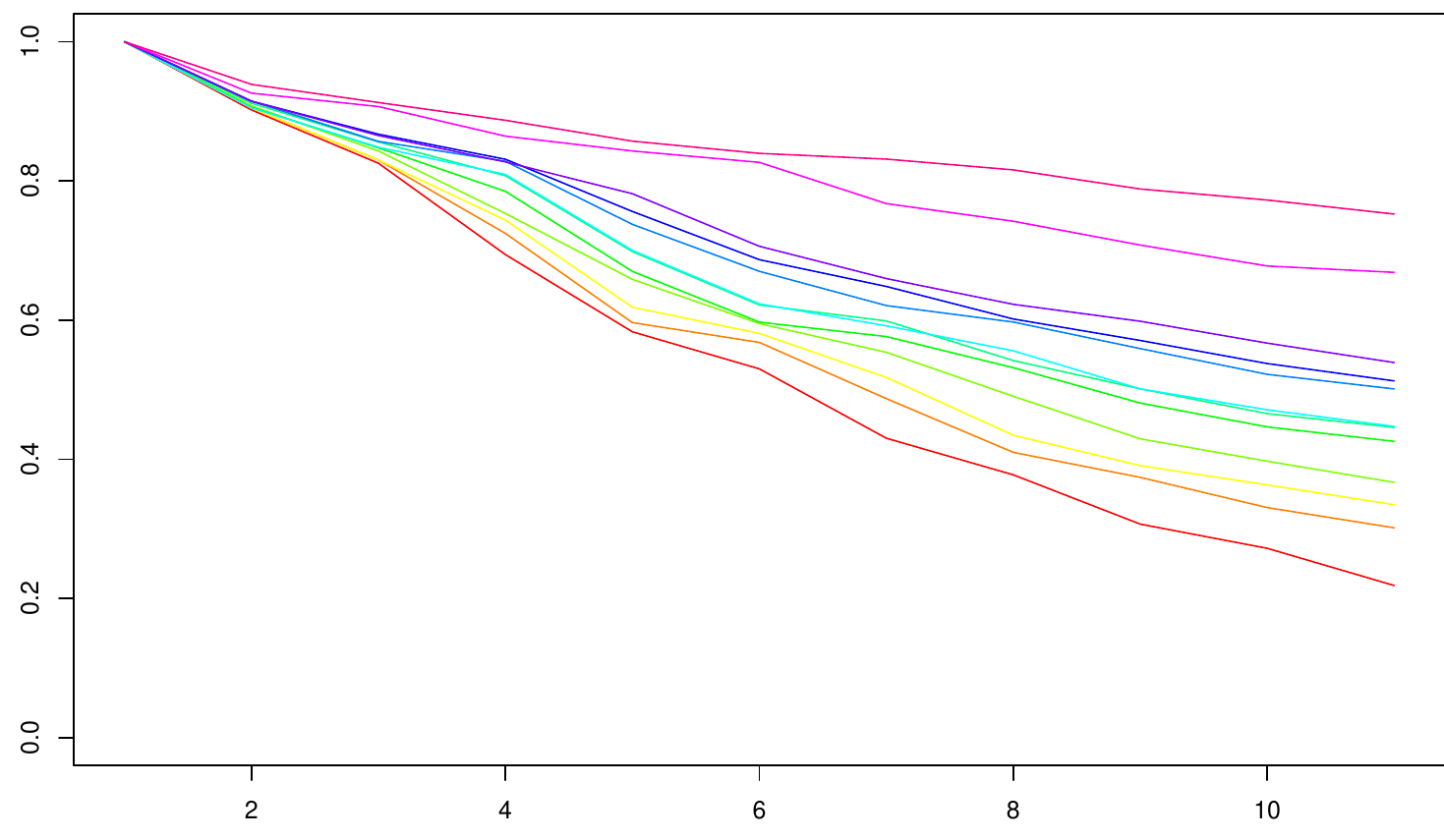}
\includegraphics[scale=0.21]{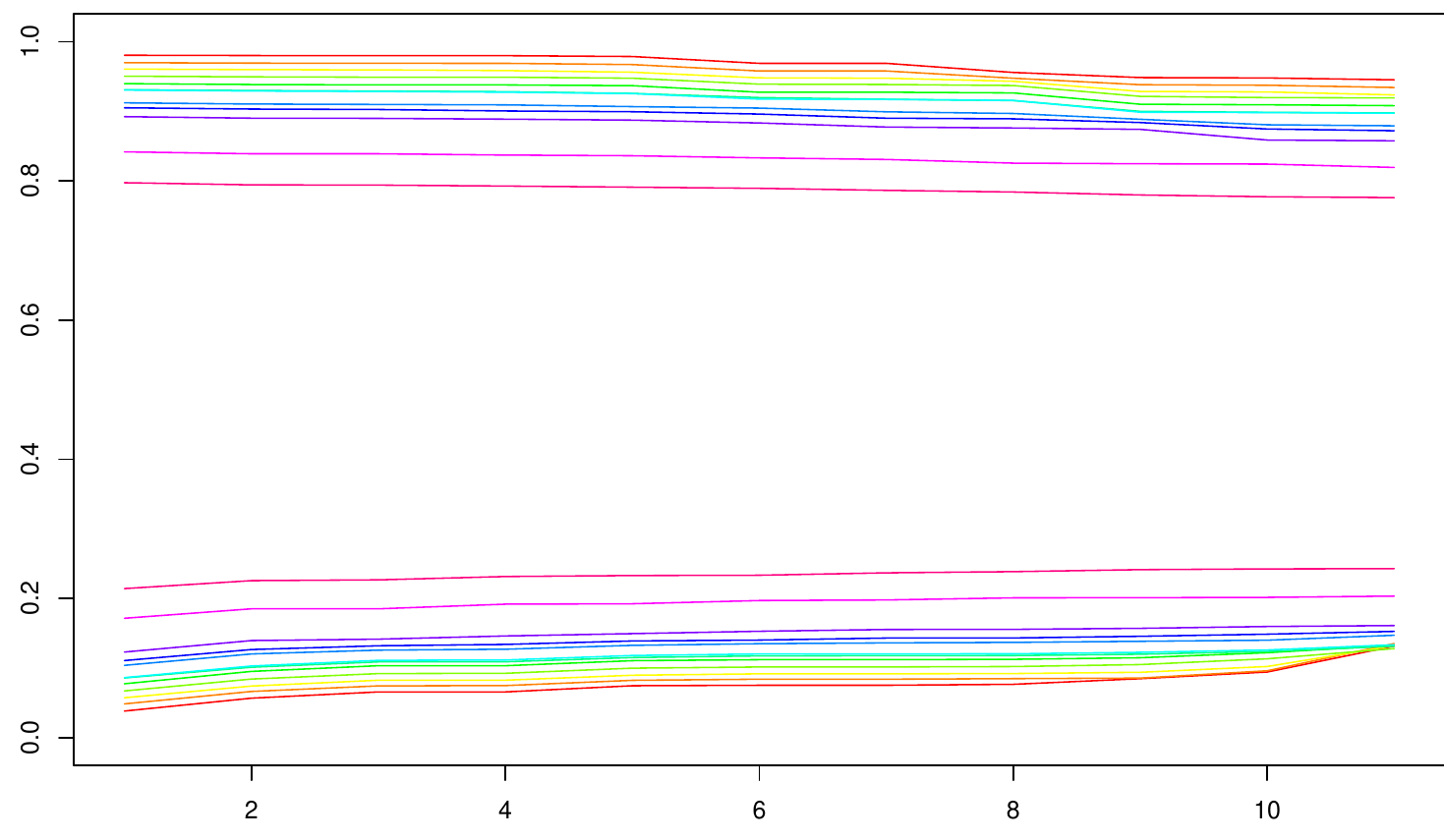}\includegraphics[scale=0.21]{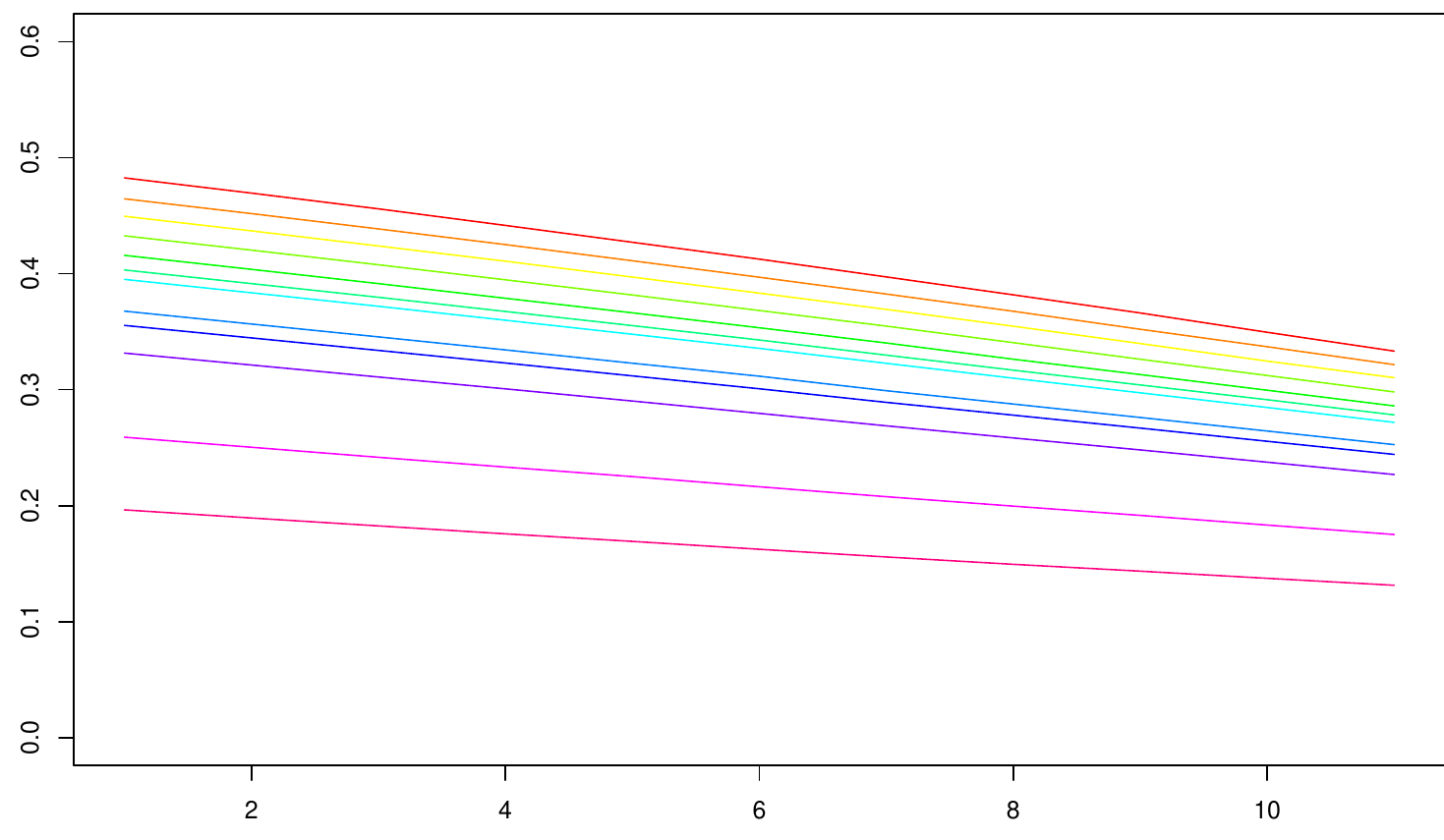}\includegraphics[scale=0.21]{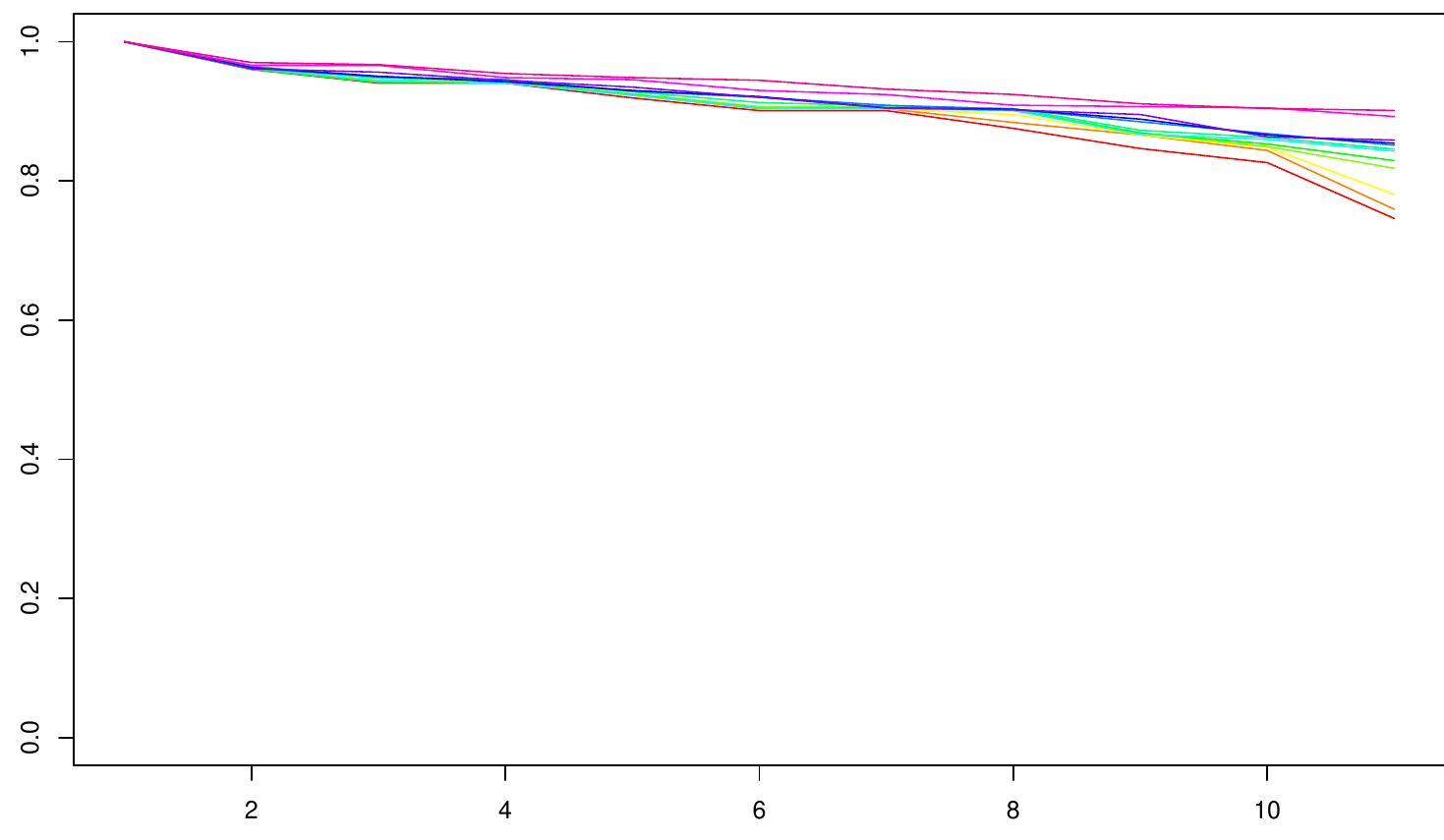}
\end{center}
\end{figure}

In the rest of Figure \ref{fig_transf}, we show the actual clusters in the MDS embedding obtained with $k$-means (column 2) and the same clusters in the original space (column 3), for some moderate intensities. For dissimilarity (\ref{k_add}) we take $V = 1.32$, for (\ref{k_mult}) $u = 1$ and for (\ref{k_local}) we use $u = 0.99$. A short remark is that a rotation of a MDS is a MDS, and that is the cause of the rotations that we see in column 2. Indeed, after MDS the geometry of the groups is not heavily modified, but at the same time some corrections to the proportions are achieved when clustering. These corrections appear very natural once we return to the original space.
\begin{table}[]
\caption{Proportion of class S = 1 in every group in different clustering procedures.} 
\label{tabla_bias}
\centering
\begin{tabular}{ccc|c|c|c|c|c|c|c|c|c|}
\cline{4-12}
\multicolumn{1}{l}{}                                    & \multicolumn{1}{l}{}                      & \multicolumn{1}{l|}{} & \multicolumn{9}{c|}{Proportion of squares in the group}                              \\ \cline{4-12} 
                                                        &                                           &                       & \multicolumn{2}{c|}{K = 2} & \multicolumn{3}{c|}{K = 3} & \multicolumn{4}{c|}{K = 4} \\ \hline
\multicolumn{1}{|c|}{\multirow{4}{*}{\footnotesize{$k$-means}}}          & \multicolumn{2}{c|}{\footnotesize{Unperturbed}}                             & 0.03         & 0.99        & 0.02    & 0.88    & 0.98   & 0.98  & 1.00 & 0.06 & 0.02 \\ \cline{2-12} 
\multicolumn{1}{|c|}{}                                  & \multicolumn{1}{c|}{\multirow{3}{*}{MDS}} & $\delta_1$            & 0.15         & 0.90        & 0.11    & 0.49    & 0.95   & 0.75  & 0.94 & 0.16 & 0.14 \\ \cline{3-12} 
\multicolumn{1}{|c|}{}                                  & \multicolumn{1}{c|}{}                     & $\delta_2$            & 0.09         & 0.96        & 0.43    & 0.04    & 0.97   & 0.08  & 0.97 & 0.85 & 0.06 \\ \cline{3-12} 
\multicolumn{1}{|c|}{}                                  & \multicolumn{1}{c|}{}                     & $\delta_4$            & 0.13         & 0.96        & 0.96    & 0.05    & 0.42   & 0.96  & 0.11 & 0.13 & 0.95 \\ \hline
\multicolumn{1}{|c|}{\multirow{4}{*}{\footnotesize{Complete Linkage}}} & \multicolumn{2}{c|}{\footnotesize{Unperturbed}}                            & 0.99         & 0.07        & 0.99    & 0.08    & 0.05   & 0.99  & 1.00 & 0.08 & 0.05 \\ \cline{2-12} 
\multicolumn{1}{|c|}{}                                  & \multicolumn{2}{c|}{$\delta_1$}                                   & 1.00         & 0.32        & 1.00    & 0.40    & 0.25   & 1.00  & 0.56 & 0.25 & 0.00 \\ \cline{2-12} 
\multicolumn{1}{|c|}{}                                  & \multicolumn{2}{c|}{$\delta_2$}                                   & 0.72         & 0.22        & 0.72    & 0.30    & 0.00   & 1.00  & 0.30 & 0.24 & 0.00 \\ \cline{2-12} 
\multicolumn{1}{|c|}{}                                  & \multicolumn{2}{c|}{$\delta_4$}                                   & 0.78         & 0.11        & 1.00    & 0.43    & 0.11   & 1.00  & 0.43 & 0.47 & 0.00 \\ \hline
\multicolumn{1}{|c|}{\multirow{3}{*}{\footnotesize{Ward's Method}}}     & \multicolumn{2}{c|}{\footnotesize{Unperturbed}}                            & 0.99         & 0.07        & 0.08    & 0.05    & 0.99   & 0.97  & 0.08 & 0.05 & 1.00 \\ \cline{2-12} 
\multicolumn{1}{|c|}{}                                  & \multicolumn{2}{c|}{$\delta_1$}                                   & 0.11         & 0.78        & 0.54    & 0.98    & 0.11   & 0.54  & 0.18 & 0.00 & 0.98 \\ \cline{2-12} 
\multicolumn{1}{|c|}{}                                  & \multicolumn{2}{c|}{$\delta_2$}                                   & 0.99         & 0.18        & 0.37    & 0.99    & 0.03   & 0.07  & 0.00 & 0.37 & 0.99 \\ \hline
\end{tabular}
\end{table}

For the same values as the previous paragraph, we present Table \ref{tabla_bias}, where we look for 2, 3 and 4 clusters with MDS and $k$-means, but also using the approximation-free complete linkage hierarchical clustering and our Ward's-like method. Since we are applying a small perturbation, we see some, but not a drastic, improvement in the heterogeneity of the groups. We also see that the more clusters we want the smaller the improvement. We stress that we can produce some improvements in diversity while modifying slightly the geometry of the data. This is a desirable situation when a lot of relevant information is codified in the geometry of the data. We also notice that we produce a partition that is almost diversity preserving with a stronger perturbation as shown in the last row of Table \ref{Table_Fig_1_comparison}.

In Figure \ref{fig_transf_pert} we study the behaviour of the same set-up as in Figure \ref{fig_transf} but with perturbation on the sensitive attribute labels. We attempt to empirically estimate a kind of break point behaviour with respect to the sensitive class, i.e., the maximal amount of perturbation in the original labels that our methods can sustain before returning no gains in diversity. The results can be interpreted as a measure of robustness against perturbations in the sensitive class, but also as a further study of the behaviour of attraction-repulsion clustering. Specifically, we select uniformly without replacement the same number of points belonging originally to class $S=1$ and $S=-1$ and we switch their labels. We take 20 samples of this kind, perform our methods, and average the resulting proportions, silhouette indexes and ARIs. The amount of perturbation goes from changing labels of $1,2,\dots,10,15$ to $20$ points in each sensitive class. Red represents the smallest perturbation in the labels and the perturbation increases as we get closer to magenta, following the rainbow colour convention.  We observe that diversity gains diminish as we increase the number of swapped labels, becoming almost non-existent when we change 20 percent of the labels in each sensitive class. At that level, essentially all dissimilarities return a clustering very similar to the $k$-means partition in the original space represented in Figure \ref{fig_k_means_pert}. This clustering does not fulfil the diversity preserving condition (\ref{fair_constraints}), however it can be easily argued that clusters capture spatial information well and are diverse.

These observations about diversity are similar to the familiar notion that fairness definitions can differ sharply, together with the fact that adequate cluster structure can depend on the particular goals of the analysis.

\begin{figure}
\caption{A partition obtained with $k$-means when randomly swapping the sensitive labels of 20 points in each original class.}
\label{fig_k_means_pert}
\begin{center}
\includegraphics[scale=0.33]{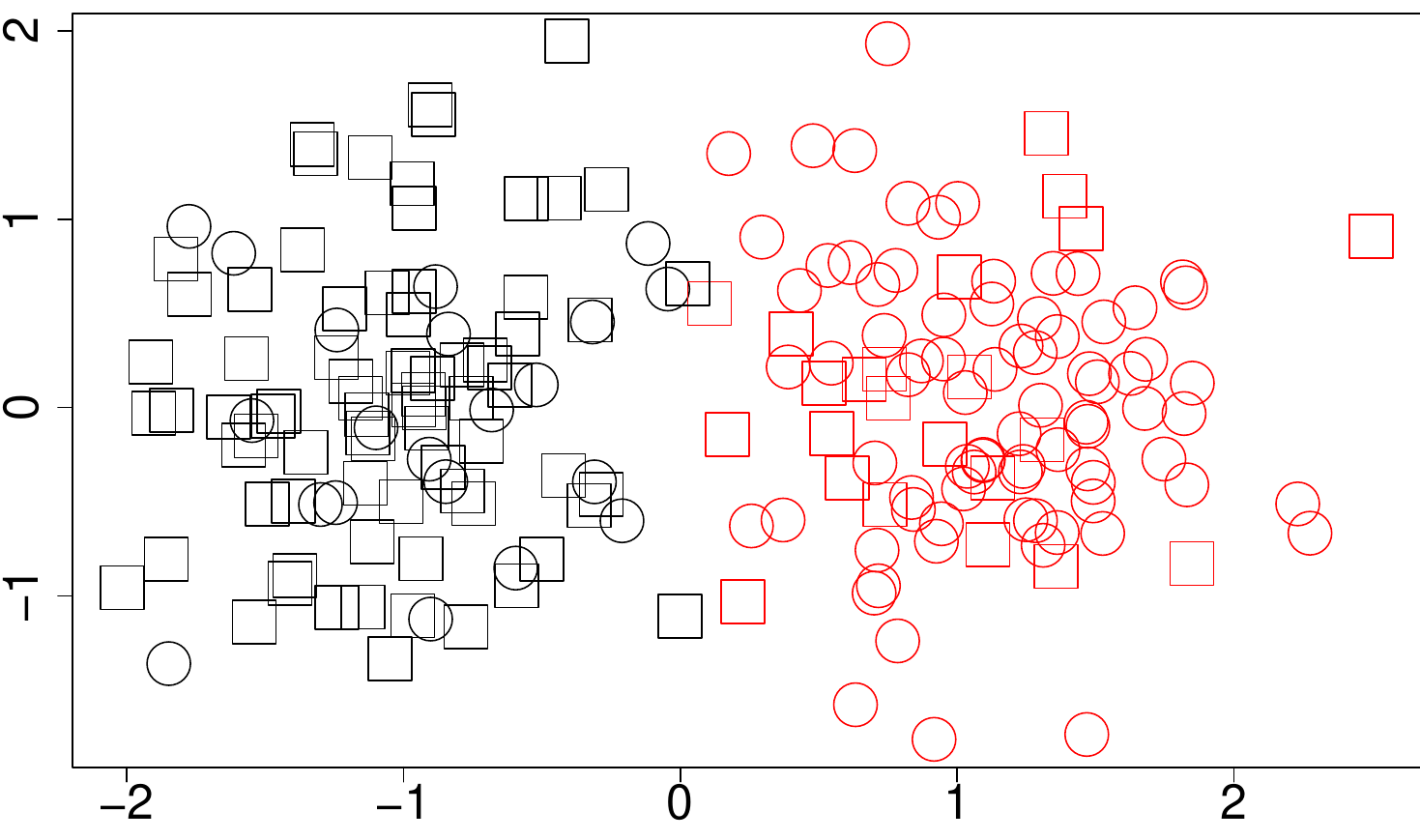}
\end{center}
\end{figure}

\subsubsection{Diversity in a non-linear setting}\label{k_trick_example}
We explore the methods introduced in Section \ref{sec_kkmeans}. With this example we want to stress that our methodology is also well-suited for non-linear clustering structure. Additionally, we want to emphasize that good partition properties and diversity may not be enough to capture some other relevant information in the data, which in this example is represented by the geometric structure.

We consider the data in the top-left image of Figure \ref{fig_k_trick}. These data have a particular geometrical shape and are split into two groups. There is an inside ring of squares, a middle ring of circles, and then an outer ring of squares. There are 981 observations, and the proportions of the classes are approximately 3 to 1 (circles are 0.246 of the total data).

It is natural to apply to the original data some clustering procedure as $k$-means or a robust extension as tclust (deals with groups with different proportions and shapes and with outliers \cite{tclust}). Looking for two clusters, we are far from capturing the geometry of the groups, but the clusters have proportions of the classes that are like the total proportion, hence, the diversity preserving condition (\ref{fair_constraints}) is satisfied, and there is a nice cluster structure as measured by average silhouette index. Indeed, this is what we see in Figure \ref{fig_k_trick} middle-left when we apply $k$-means to the original data.

On the other hand, the kernel trick is convenient in this situation.  In this toy example it is easy to select an appropriate kernel function, for instance, $\kappa (x,y) = x_1^2y_1^2 + x_2^2y_2^2$, which corresponds to a transformation $\phi((x_1,x_2)) = (x_1^2,x_2^2)$. Indeed, this kernel produces linear separation between the groups. The data in the transformed space is depicted in the top-right of Figure \ref{fig_k_trick}. 
Our adaptation to the kernel trick uses $d_\kappa$ as defined in (\ref{d_kappa}) and dissimilarity (\ref{k_local}) in the form
\begin{equation}
\label{k_loc_kt}
\delta_{\kappa,4}((X_1,S_1),(X_2,S_2)) = \big(1 + \mathrm{sign}(S_1'VS_2)u(1-e^{-v(S_1'VS_2)^2})e^{-wd_\kappa(X_1,X_2)}\big)d_\kappa(X_1,X_2),
\end{equation}
for $X_1,X_2$ in the original two-dimensional space, as described in Section \ref{sec_kkmeans}.


Considering the discussion at the end of Section \ref{section_mds} and Section \ref{section_parameters} we use dissimilarity (\ref{k_loc_kt}) with $S_1,S_2\in \{(1,0),(0,1)\}$. In our setting circles are labelled as $(1,0)$ and squares as $(0,1)$. Now if we fix $u = 0$, use (\ref{k_loc_kt}) to calculate the dissimilarity matrix $\Delta$ and use MDS, essentially, we will be in the space depicted top-right on Figure \ref{fig_k_trick}. Looking for two clusters with tclust, allowing groups with different sizes, we get the result depicted middle-right in Figure \ref{fig_k_trick}. We have captured the geometry of the clusters but the proportions of the class $S$ are not the best, as seen in row 1 columns 2 and 3 of Table  \ref{table_k_trick} (ideally they should be close to 0.754).
\begin{figure}
\caption{Top row: data in the original space (left) and after transformation $\phi$ (right). Middle row: $k$-means in the original space (left) and tclust applied in the transformed space and plotted in the original one (right). Bottom row: tclust after diversity enhancing corrections, corresponding to the case shown in the last row of Table \ref{table_k_trick}, applied in the transformed space (left) and represented in the original space (right).}
\label{fig_k_trick}
\begin{center}
\includegraphics[scale=0.32]{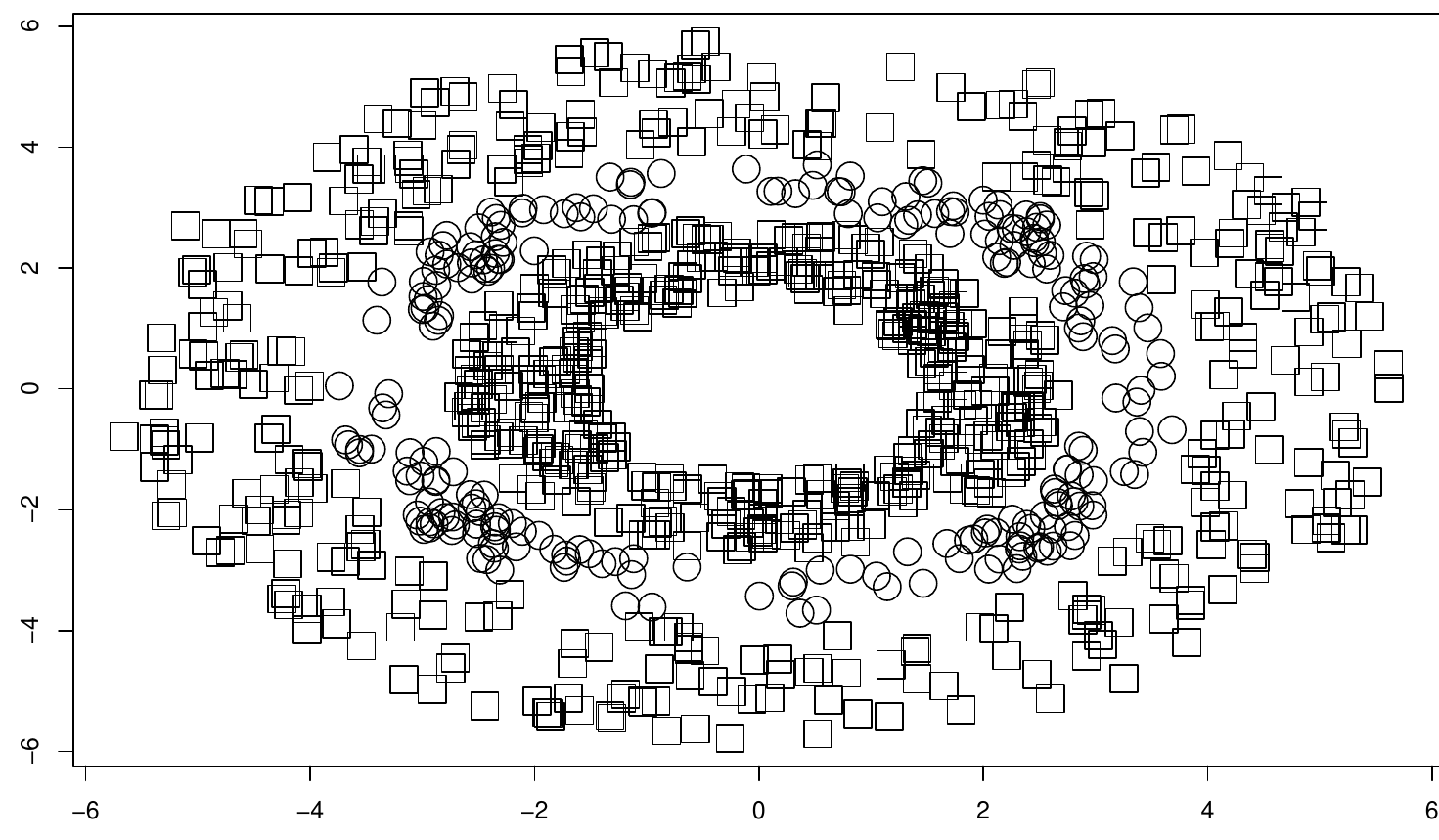}\includegraphics[scale=0.32]{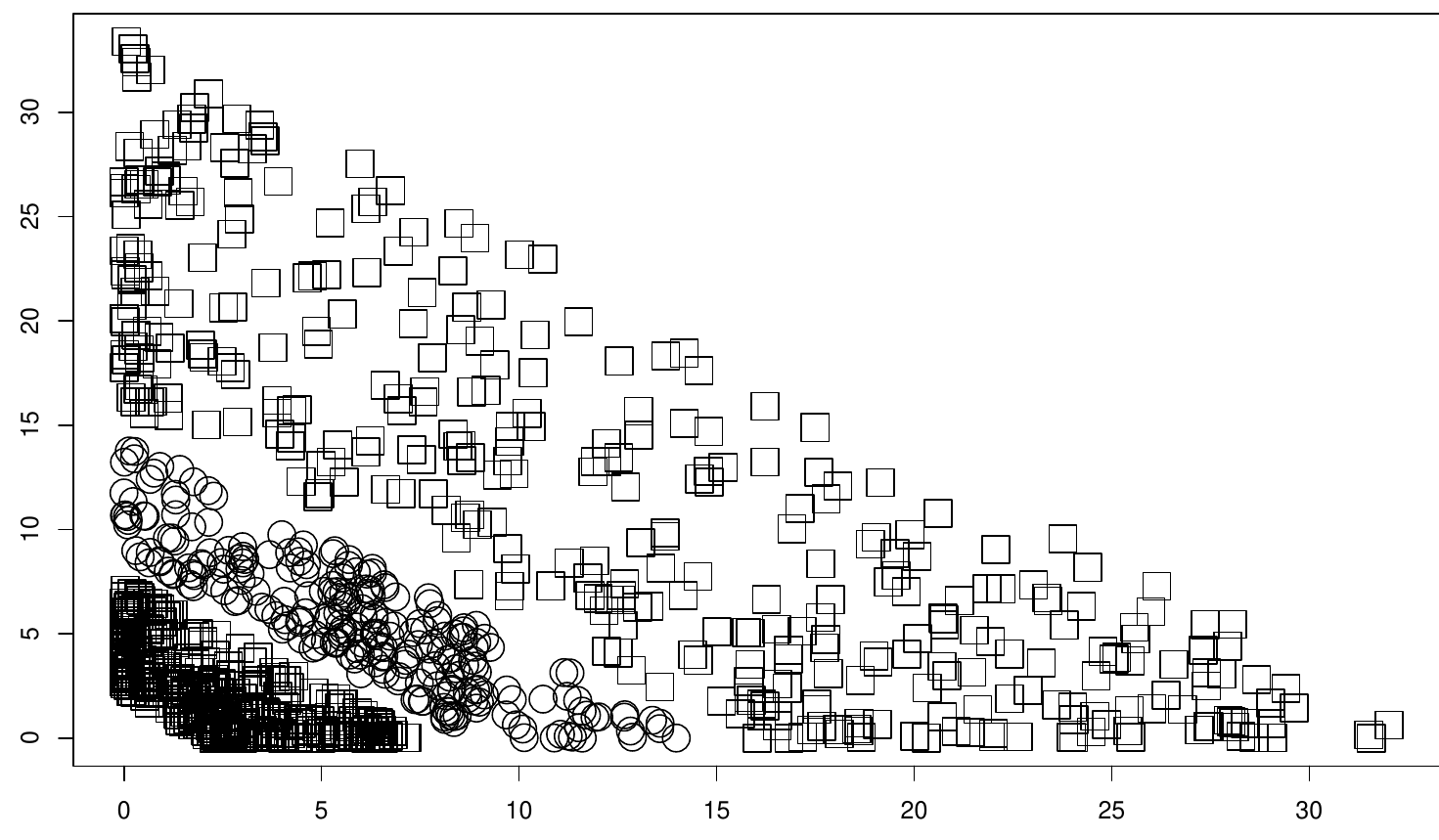}
\includegraphics[scale=0.32]{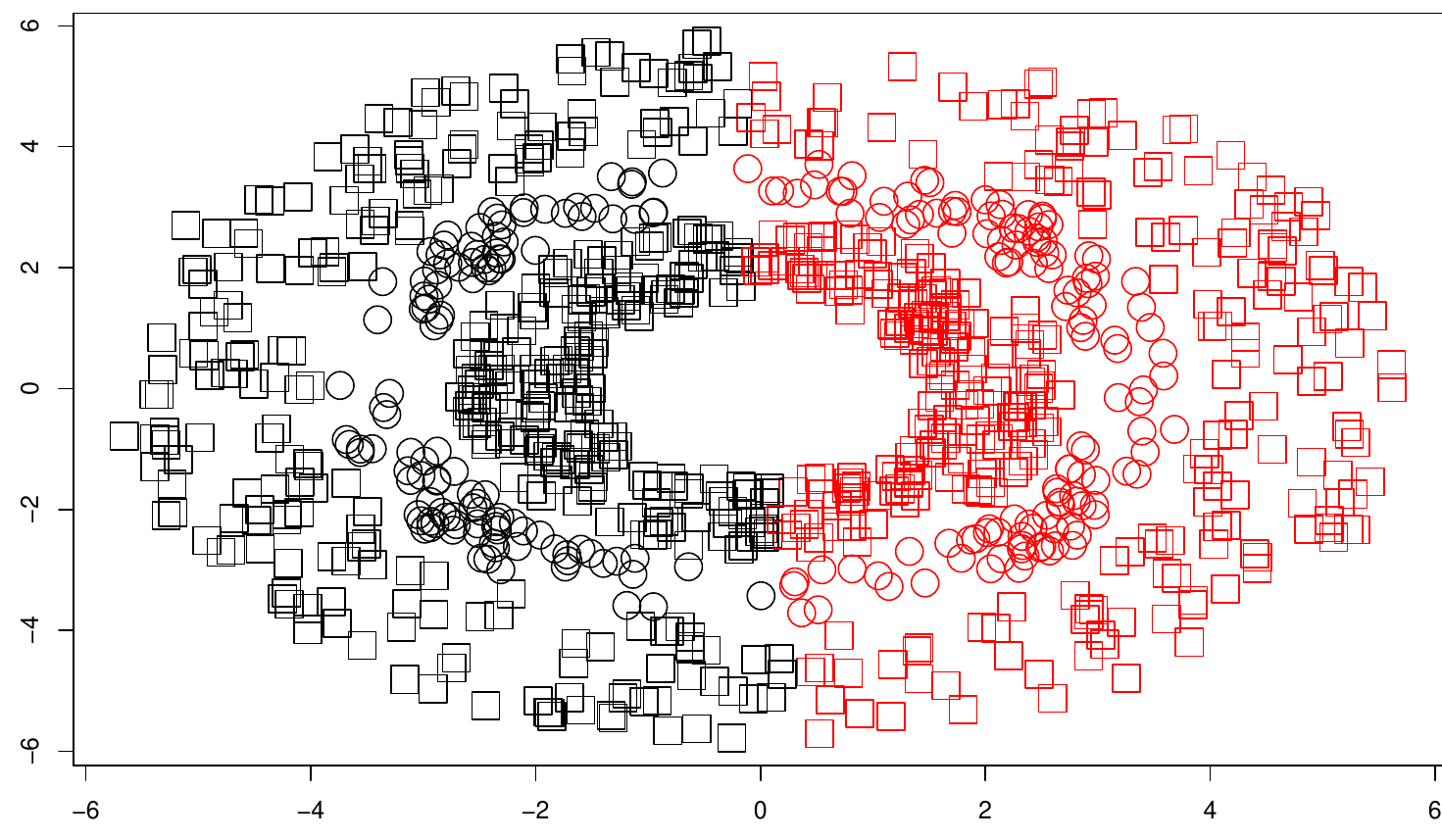}\includegraphics[scale=0.32]{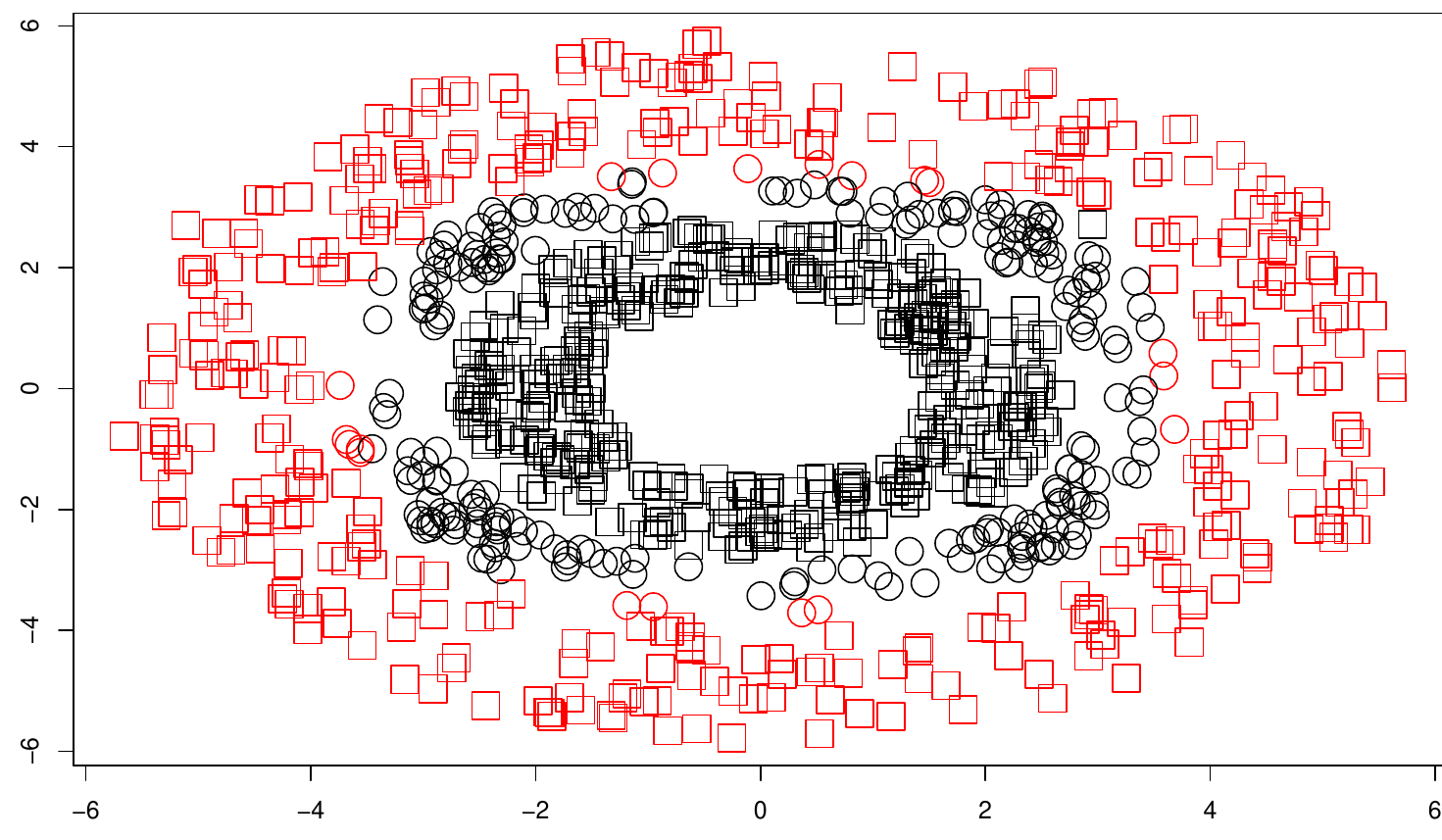}
\includegraphics[scale=0.32]{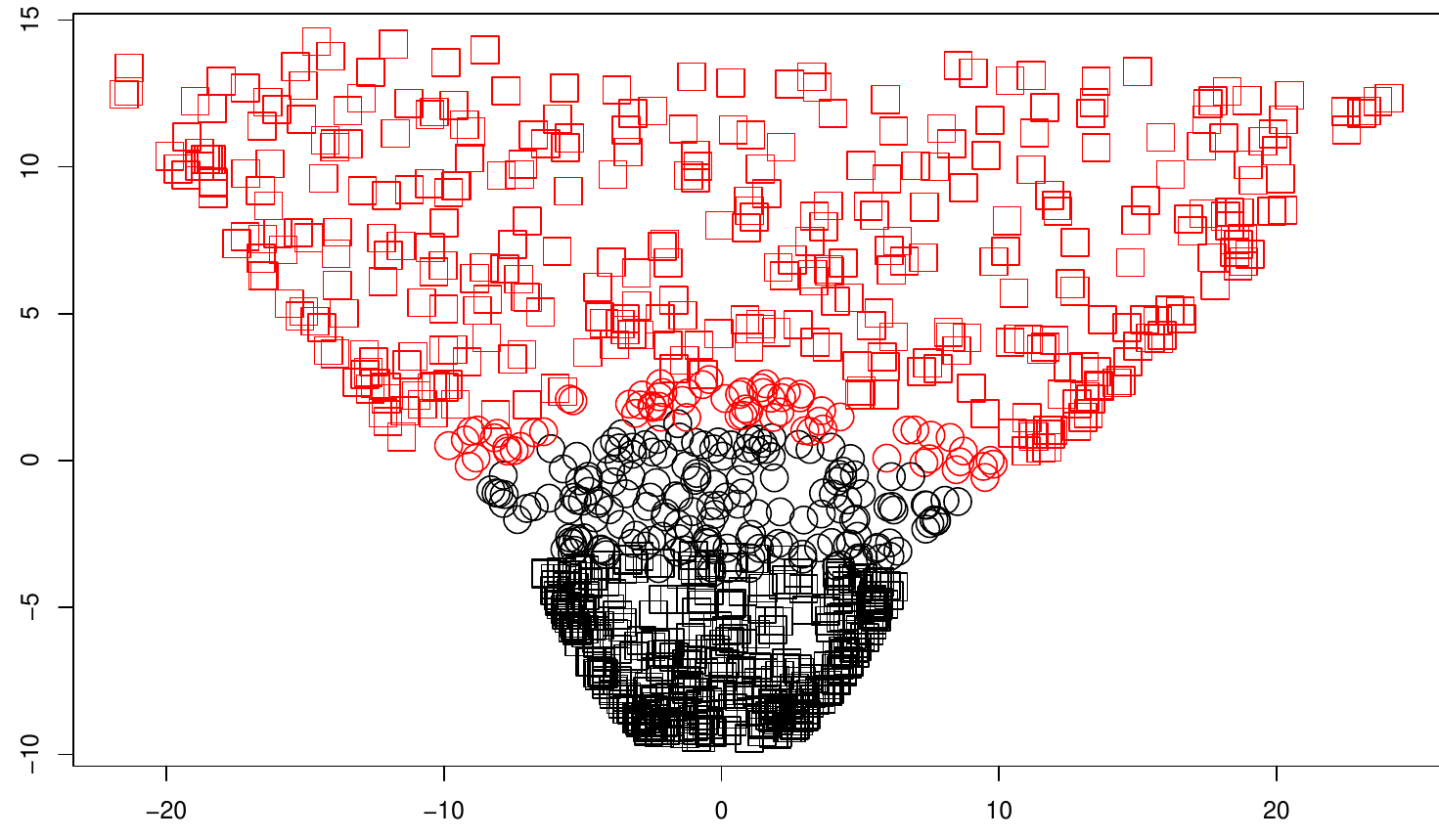}\includegraphics[scale=0.32]{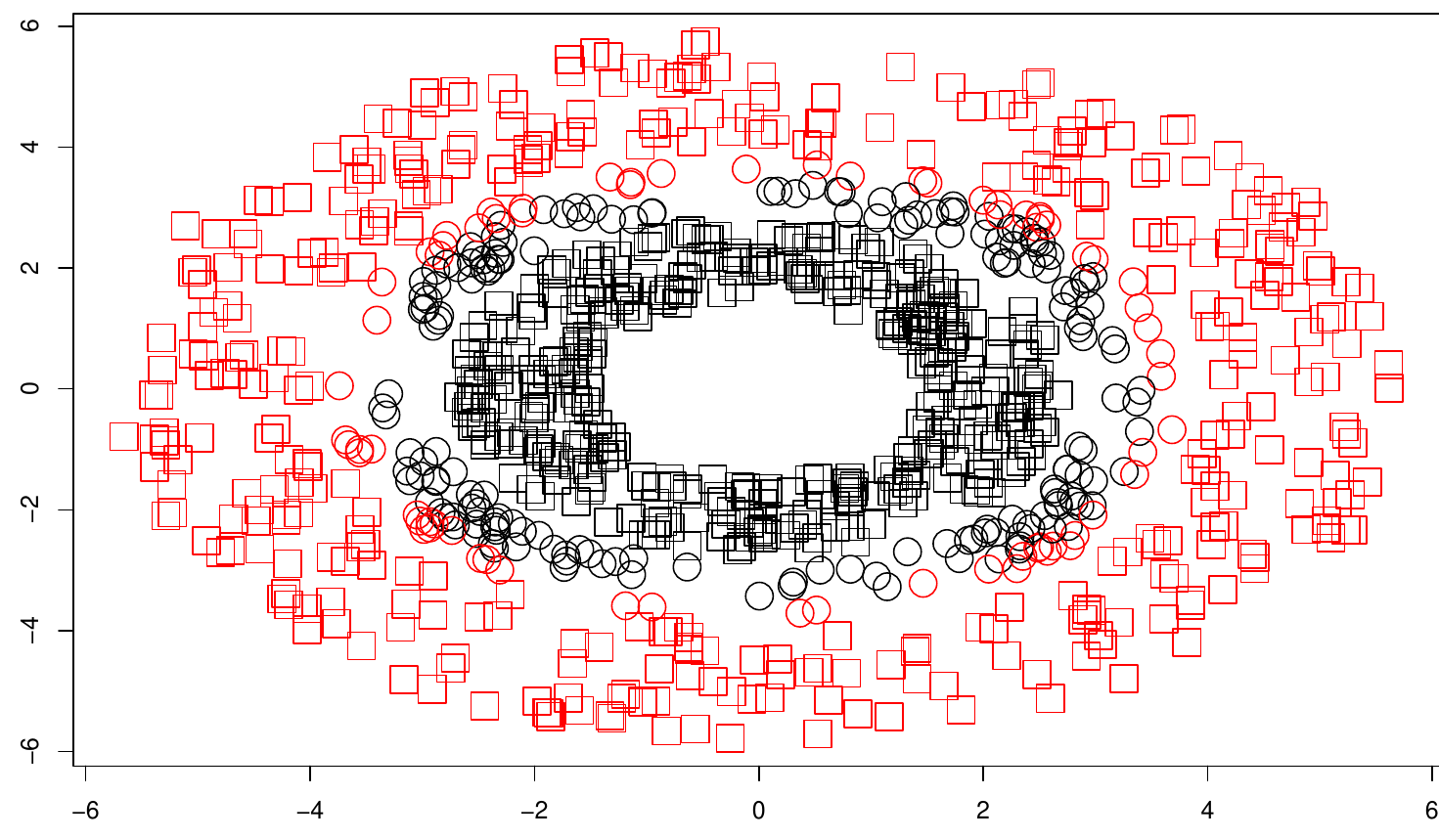}
\end{center}
\end{figure}
To gain diversity in what is referred to as cluster 2 in Table \ref{table_k_trick} (in red in the plots), we vary the intensity $u$ of our local dissimilarity, with the other parameters set as indicated in Table \ref{table_k_trick}. We see that as we increase the intensity of the interactions we gain heterogeneity in cluster 2, and both proportions come closer to the total proportion 0.754 (columns 2-3). Again, this is achieved without destroying the geometry of the original classes after the MDS, as seen in the small variation of the average silhouette index in columns 4-5.

We plot the best performance, given by $u = 0.98$, after MDS in bottom-left and in the original space in bottom-right of Figure \ref{fig_k_trick}. Clearly, we have been able to capture the geometry of the groups and to produce relatively divers clusters.

\begin{table}[]
\caption{Effect of varying the intensity $u$ of the local dissimilarity (\ref{k_loc_kt}), for fixed $V = ((1,-1)'|(-1,0)')$, $v = 20$ and $w = 0.05$. First two columns contain the proportion of points with $S = (0,1)$ in the clusters found with tclust in the transformed space. Last two columns show the silhouette of the original classes in the MDS.}
\label{table_k_trick}
\centering
\begin{tabular}{|c|c|c|c|c|}
\hline
u     & Prop. in cluster 1 & Prop. in cluster 2 & Silhouette for (0, 1) & Silhouette for (1,0) \\ \hline
0.000 & 0.629              & 0.950              & -0.247               & 0.502               \\ \hline
0.098 & 0.629              & 0.950              & -0.245               & 0.502               \\ \hline
0.196 & 0.629              & 0.950              & -0.243               & 0.499               \\ \hline
0.294 & 0.630              & 0.948              & -0.241               & 0.495               \\ \hline
0.392 & 0.631              & 0.945              & -0.239               & 0.491               \\ \hline
0.490 & 0.631              & 0.943              & -0.237               & 0.486               \\ \hline
0.588 & 0.631              & 0.943              & -0.235               & 0.481               \\ \hline
0.686 & 0.631              & 0.943              & -0.234               & 0.476               \\ \hline
0.784 & 0.630              & 0.946              & -0.232               & 0.471               \\ \hline
0.882 & 0.672              & 0.863              & -0.231               & 0.467               \\ \hline
0.980 & 0.681              & 0.849              & -0.229               & 0.465               \\ \hline
\end{tabular}
\end{table}

\subsection{Comparison with fair clustering through fairlets}\label{section_comparison}
In this section we present a comparison of the results of our diversity enhancing clustering methods with results obtained by implementing, in Python and R, the fair clustering procedure introduced in \cite{fair_k_means} based on fairlets decomposition. We recall that this clustering method is an example of what we call diversity preserving clustering. Since our examples are concerned with two values for the protected class it is justified to use \cite{fair_k_means} for comparison since it is well suited for this situation. Our implementation of the case when the size of both protected classes is the same, which reduces to an assignment problem, is implemented using the function \textit{max\_bipartite\_matching} of the package \emph{igraph} in R. In the case of different sizes, we must solve a min cost flow problem as stated in \cite{fair_k_means}, which can be done in Python with the function \textit{min\_cost\_flow} of the package \textit{networkx} (also it can be solved with a \textit{min\_cost\_flow} solver in \textit{ortools}).

When implemented as a min cost flow problem, Chierichietti's et al. methodology has two free parameters, $t'$ and $\tau$. First, for a partition $\mathcal{C}$ of data $X$, we have that $1/t'\leq \mathrm{balance}(\mathcal{C})\leq \mathrm{balance}(X)$ and hence $t'$ controls the lower bound of the diversity of the partition. Recall the definition of balance in (\ref{balance}). Second, $\tau$ is a free parameter related to the distance between points, and has a defined lower limit given by the maximum distance taken from the set formed by the distance between each point in one class and its respective closest point in the other class. For more information we refer to \cite{fair_k_means}.

We start with the data used for the example studied in Figure \ref{fig_transf}. We will address $k$-median clustering, for which \cite{fair_k_means} has a fair implementation. Since the data has two groups of the same size, we can solve an assignment problem or use a min cost flow problem. For the min cost flow problem, we have the set of parameters $\{(t',\tau)\}$ with $t' = 2,3,4$ and $\tau = \infty, 2.42$. Values for $\tau$ represent no locality and maximum locality. As comparison we will use $k$-median clustering after perturbing the data with $\delta_1$ with parameters $(U = 0, V = 4.4)$ and doing a MDS embedding. Results are shown in Table \ref{Table_Fig_1_comparison}. Since the data is in one-to-one correspondence between the two classes, both the alignment solution and the different min cost flow solutions give balance $=1$, hence the diversity preserving condition (\ref{fair_constraints}) is fulfilled. Our method gives balance close to 1 but does not fulfil (\ref{fair_constraints}). However, the average silhouette index of our method is higher, which means that clusters are more identifiable and compact, and the $k$-median objective function is also lower, and hence better. A plot of some of the different clusterings can be seen in Figure \ref{fig_comparison_1}.

\begin{table}[]
\caption{Rows 1-7 are implementations of fair clustering with fairlets while the last row is $k$-median clustering after using $\delta_1$, with $(U = 0, V = 4.4)$, and MDS embedding.}\label{Table_Fig_1_comparison}
\begin{center}
\begin{tabular}{c|c|c|c|}
\cline{2-4}
\multicolumn{1}{l|}{}                      & Balance & Average Silhouette & $k$-median Objective \\ \hline
\multicolumn{1}{|c|}{Assignment Problem}  & 1       & -0.004           & 463           \\ \hline
\multicolumn{1}{|c|}{$(2, \infty)$}        & 1       & -0.006            & 469           \\ \hline
\multicolumn{1}{|c|}{$(2, 2.42)$}      & 1       & -0.001            & 471           \\ \hline
\multicolumn{1}{|c|}{$(3, \infty)$}        & 1       & -0.002            & 466           \\ \hline
\multicolumn{1}{|c|}{$(3, 2.42)$}      & 1       & -0.004            & 468           \\ \hline
\multicolumn{1}{|c|}{$(4, \infty)$}        & 1       & -0.005            & 466           \\ \hline
\multicolumn{1}{|c|}{$(4, 2.42)$}      & 1       & -0.001            & 464          \\ \hline
\multicolumn{1}{|c|}{Attraction-repulsion} & 0.89 & 0.2             & 440           \\ \hline
\end{tabular}
\end{center}
\end{table}

\begin{figure}
\caption{Left: clusters obtained by fair $k$-median as an assignment problem(\cite{fair_k_means}). Middle: clusters obtained by fair $k$-median as min cost flow problem with $t' = 4,\tau = 2.42$(\cite{fair_k_means}). Right: clusters obtained by attraction-repulsion clustering with dissimilarity (\ref{k_add}) and MDS, using $k$-median.}
\label{fig_comparison_1}
\begin{center}
\includegraphics[scale=0.26]{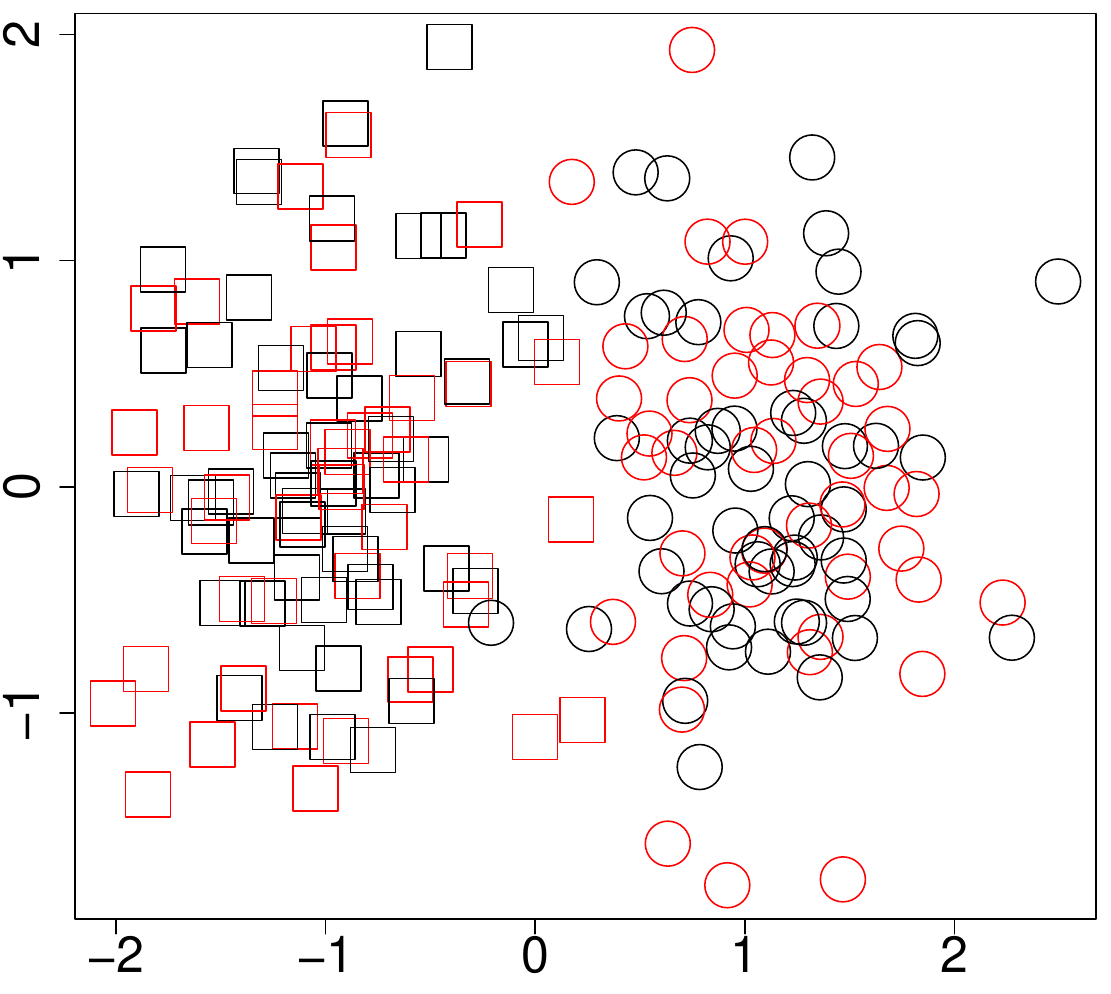}
\includegraphics[scale=0.26]{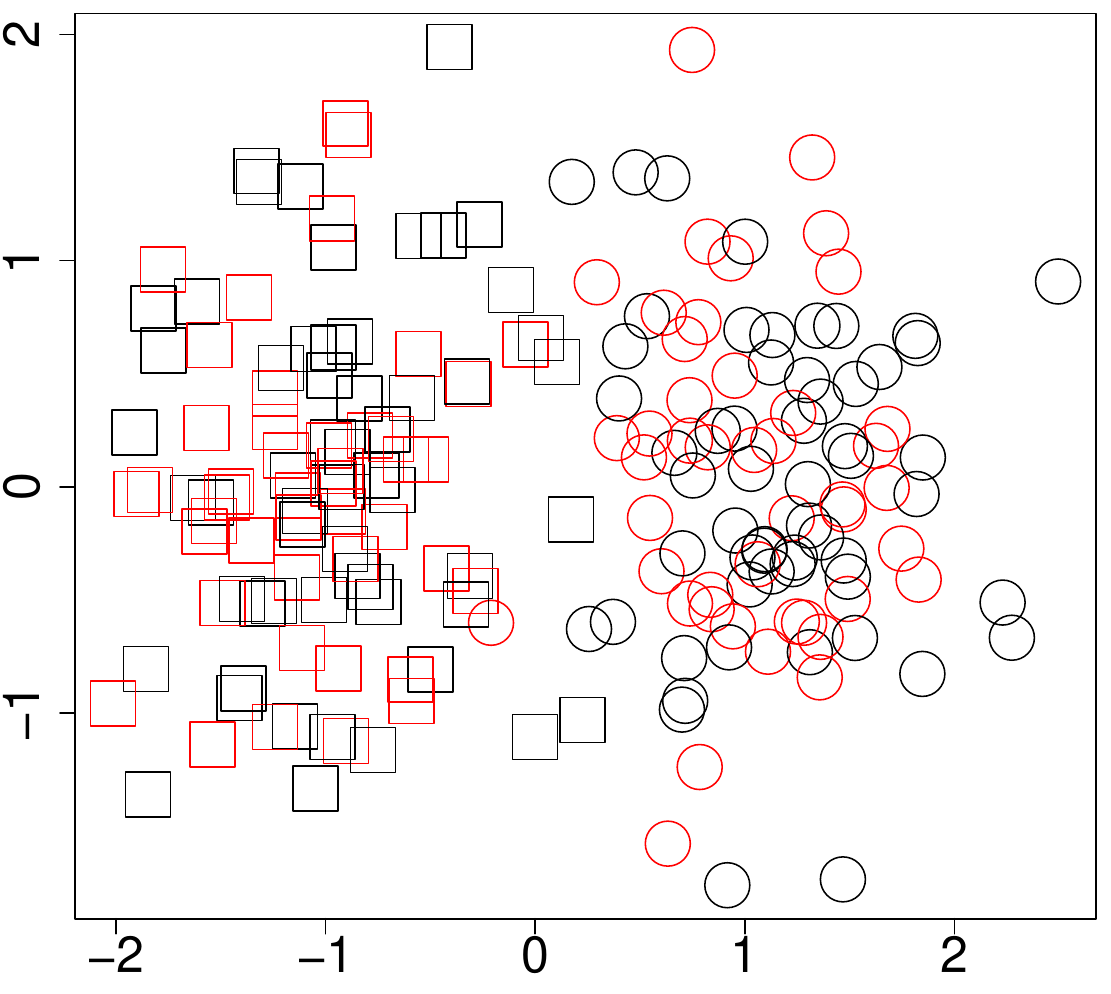}\includegraphics[scale=0.26]{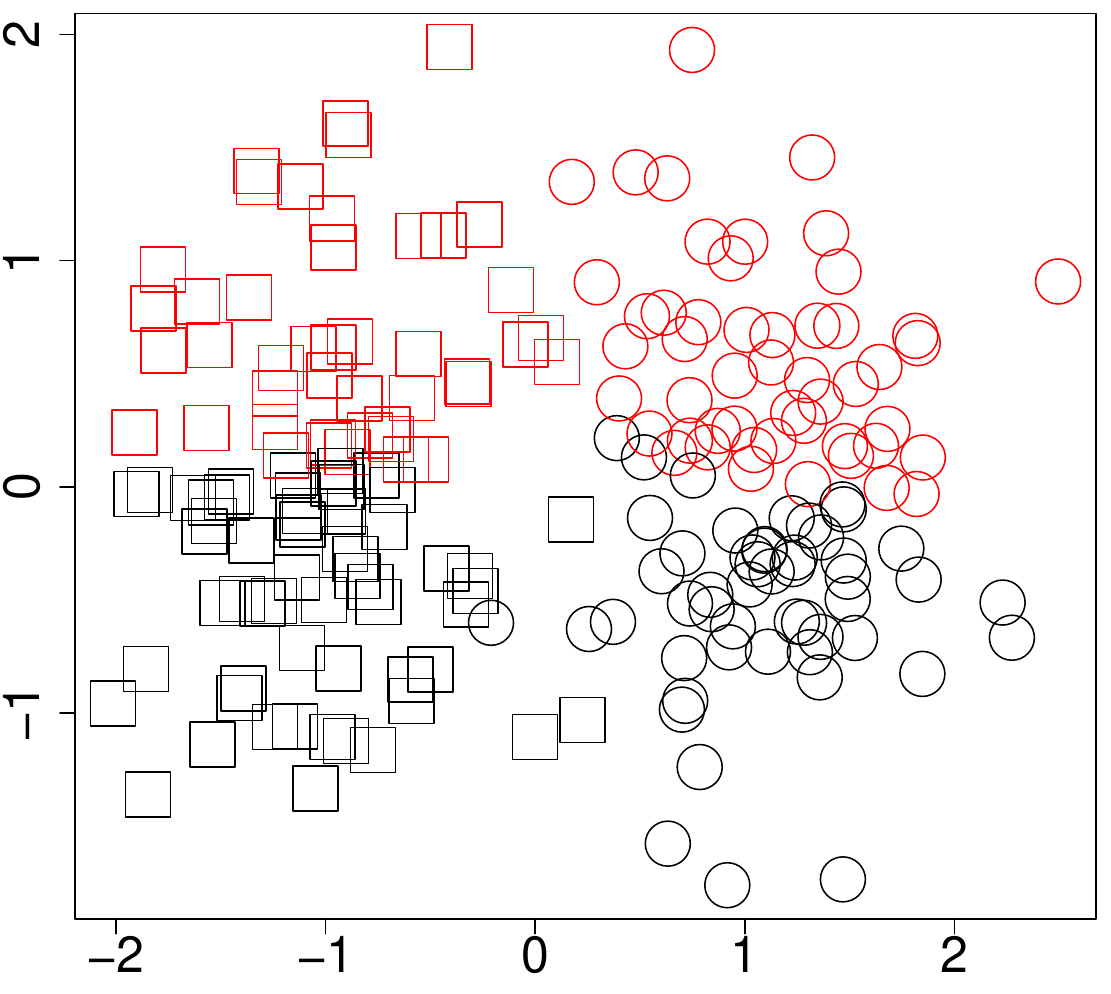}
\end{center}
\end{figure}

Our next comparison is for the data used in Figure \ref{fig_k_trick}. We stress that for fairlets we are working with the data after the implicit kernel transformation, i.e., the data shown top-right in Figure \ref{fig_k_trick}, and therefore results are comparable with attraction-repulsion clustering in the kernel trick setting. Results are shown in Table \ref{Table_Fig_4_comparison}. We see that consistently the fair $k$-median implementation gives balance values very close to $241/740 \approx 0.3257$, showing a high degree of diversity in the clusters. Our method gives a lower balance value; hence groups are less diverse, but as we see from the silhouette and $k$-median objective function values, the groups are more identifiable and more compact. Even more, comparing the middle-right of Figure \ref{fig_k_trick} and left of Figure \ref{fig_comparison_2}, we see that our procedure is even able to capture the underlying geometry of the data.
\begin{figure}
\caption{Left: clustering obtained using fair $k$-median as min cost flow problem with $t' = 4,\tau = 19.62$ (\cite{fair_k_means}).}
\label{fig_comparison_2}
\begin{center}
\includegraphics[scale=0.26]{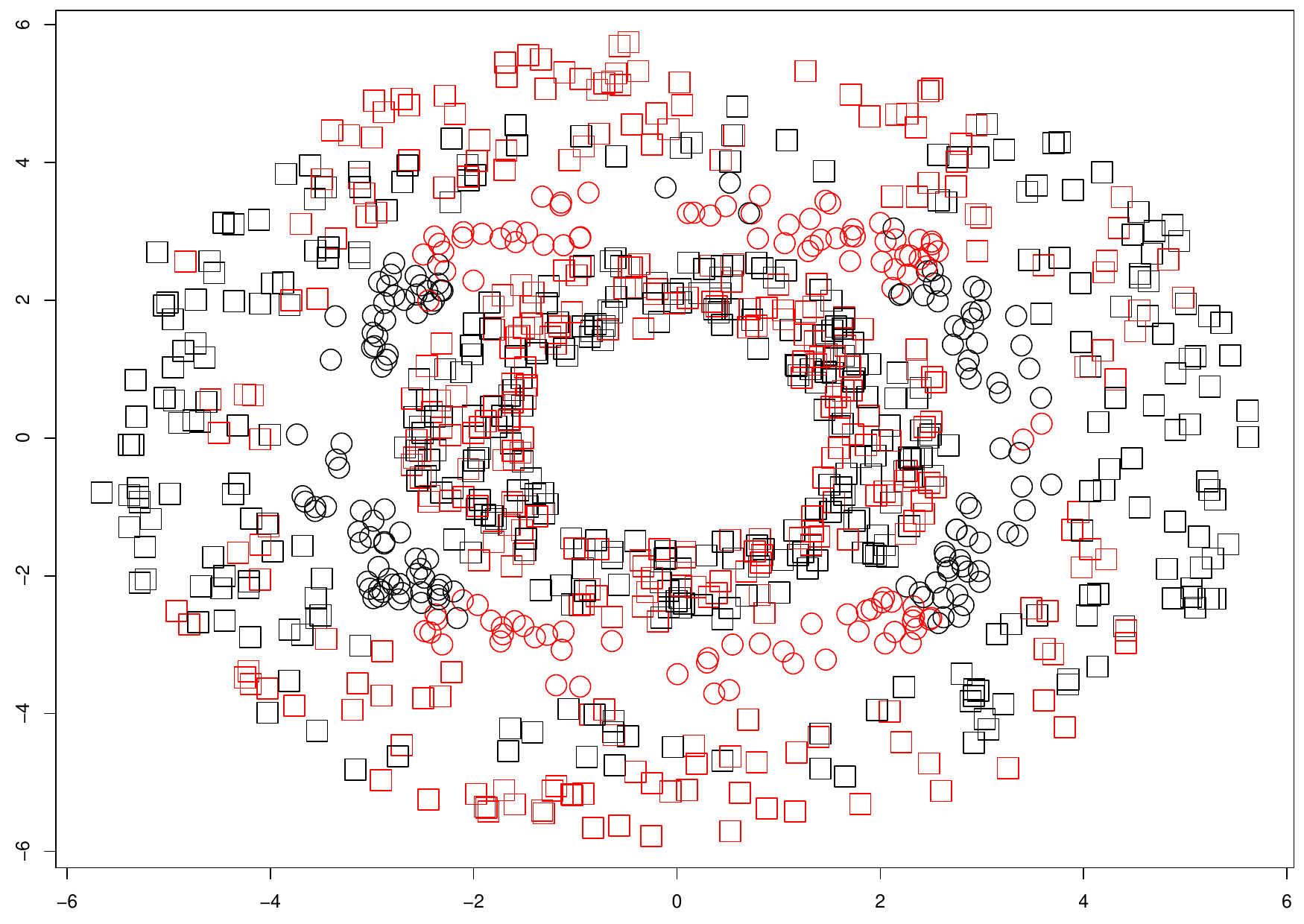}
\end{center}
\end{figure}

\begin{table}[]
\caption{Rows 1-6 are implementations of fair (diversity preserving) clustering with fairlets as a min cost flow problem with different parameters. Last row is a tclust clustering after using $\delta_4$, with $(u = 0.98, v = 20, w = 0.05, V = ((1,-1)'|(-1,0)')$, and MDS embedding.}\label{Table_Fig_4_comparison}
\begin{center}
\begin{tabular}{c|c|c|c|}
\cline{2-4}
\multicolumn{1}{l|}{}                      & Balance & Average Silhouette & $k$-median Objective \\ \hline
\multicolumn{1}{|c|}{$(4, \infty)$}        & 0.316  & 0.027             & 14613          \\ \hline
\multicolumn{1}{|c|}{$(4, 19.62)$}      & 0.325  & 0.069             & 14975           \\ \hline
\multicolumn{1}{|c|}{$(5, \infty)$}        & 0.322  & 0.036             & 14819           \\ \hline
\multicolumn{1}{|c|}{$(5, 19.62)$}      & 0.302  & 0.064             & 14854           \\ \hline
\multicolumn{1}{|c|}{$(6, \infty)$}        & 0.305  & 0.026             & 14683           \\ \hline
\multicolumn{1}{|c|}{$(6, 19.62)$}      & 0.299  & 0.051             & 14792           \\ \hline
\multicolumn{1}{|c|}{Attraction-repulsion} & 0.189  & 0.4             & 6751            \\ \hline
\end{tabular}
\end{center}
\end{table}

\begin{table}[]
\caption{Rows 1-8 are implementations of fair clustering with fairlets as a min cost flow problem with different parameters. Rows 9-10 are the best results we have obtained with attraction-repulsion clustering. The last row represents the values for a $k$-means clustering in the original data.}\label{Tab_Ricci_comparison}
\begin{footnotesize}
\begin{center}
\begin{tabular}{c|c|c|c|c|c|c|}
\cline{2-7}
\multicolumn{1}{l|}{}                             & \multicolumn{3}{c|}{k = 2}                        & \multicolumn{3}{c|}{k = 4}                        \\ \cline{2-7} 
\multicolumn{1}{l|}{}                             & Balance & Aver. Silhouette & $k$-median Objec. & Balance & Aver. Silhouette & $k$-median Objec. \\ \hline
\multicolumn{1}{|c|}{$(2, \infty)$}               & 0.735  & 0.01             & 4060            & 0.588  & -0.04            & 4038            \\ \hline
\multicolumn{1}{|c|}{$(2, 18.61958)$}             & 0.707  & 0.21             & 3683            & 0.667  & 0.06           & 3658            \\ \hline
\multicolumn{1}{|c|}{$(3, \infty)$}               & 0.682  & 0.07             & 3881            & 0.655  & -0.04            & 4170            \\ \hline
\multicolumn{1}{|c|}{$(3, 18.61958)$}             & 0.682  & 0.17             & 3546            & 0.625  & 0.06             & 3437            \\ \hline
\multicolumn{1}{|c|}{$(4, \infty)$}               & 0.650  & 0.07             & 3949            & 0.652  & -0.06           & 4160            \\ \hline
\multicolumn{1}{|c|}{$(4, 18.61958)$}             & 0.700  & 0.16             & 3590            & 0.643  & 0.07             & 3547            \\ \hline
\multicolumn{1}{|c|}{$(5, \infty)$}               & 0.674  & 0.06             & 3922            & 0.684  & -0.05            & 3836            \\ \hline
\multicolumn{1}{|c|}{$(5, 18.61958)$}             & 0.615  & 0.23             & 3475            & 0.563  & 0.06             & 3470           \\ \hline
\multicolumn{1}{|c|}{A-R Ward}   & 0.613  & 0.30             & 3866            & 0.357  & 0.28             & 3356            \\ \hline
\multicolumn{1}{|c|}{A-R Kmeans} & 0.524  & 0.40             & 3420            & 0.478  & 0.29             & 3123            \\ \hline
\multicolumn{1}{|c|}{Kmeans} & 0.341  & 0.42             & 3265            & 0.111  & 0.39             & 2988            \\ \hline
\end{tabular}
\end{center}
\end{footnotesize}
\end{table}

\begin{figure}
\caption{Diversity enhancing, diversity preserving and standard clustering for the Ricci dataset, where  circles represent not white individuals and squares represent white individuals. First row: $k$-means for 2 and 4 clusters in the  unperturbed (original) data. Second row: $k$-means for 2 and 4 clusters in the MDS setting with $\delta_1$. Third row: Ward's method for 2 and 4 clusters with $\delta_2$. Fourth row: fair $k$-median as in \cite{fair_k_means} with $t'= 5,\tau = 18.62$ for 2 and 4 clusters.}
\label{fig_plot_soldiers}
\begin{center}
\includegraphics[scale=0.32]{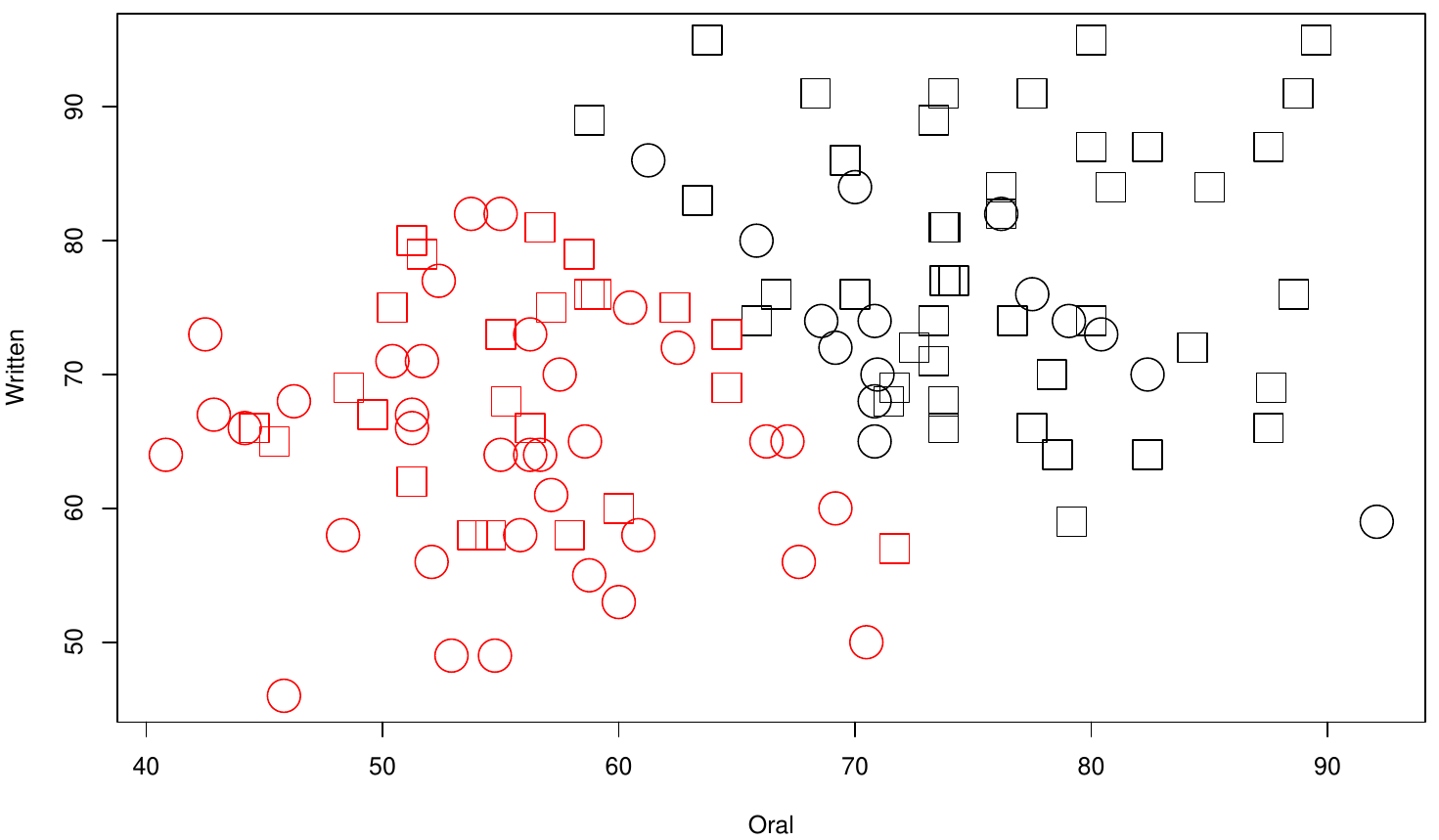}\includegraphics[scale=0.32]{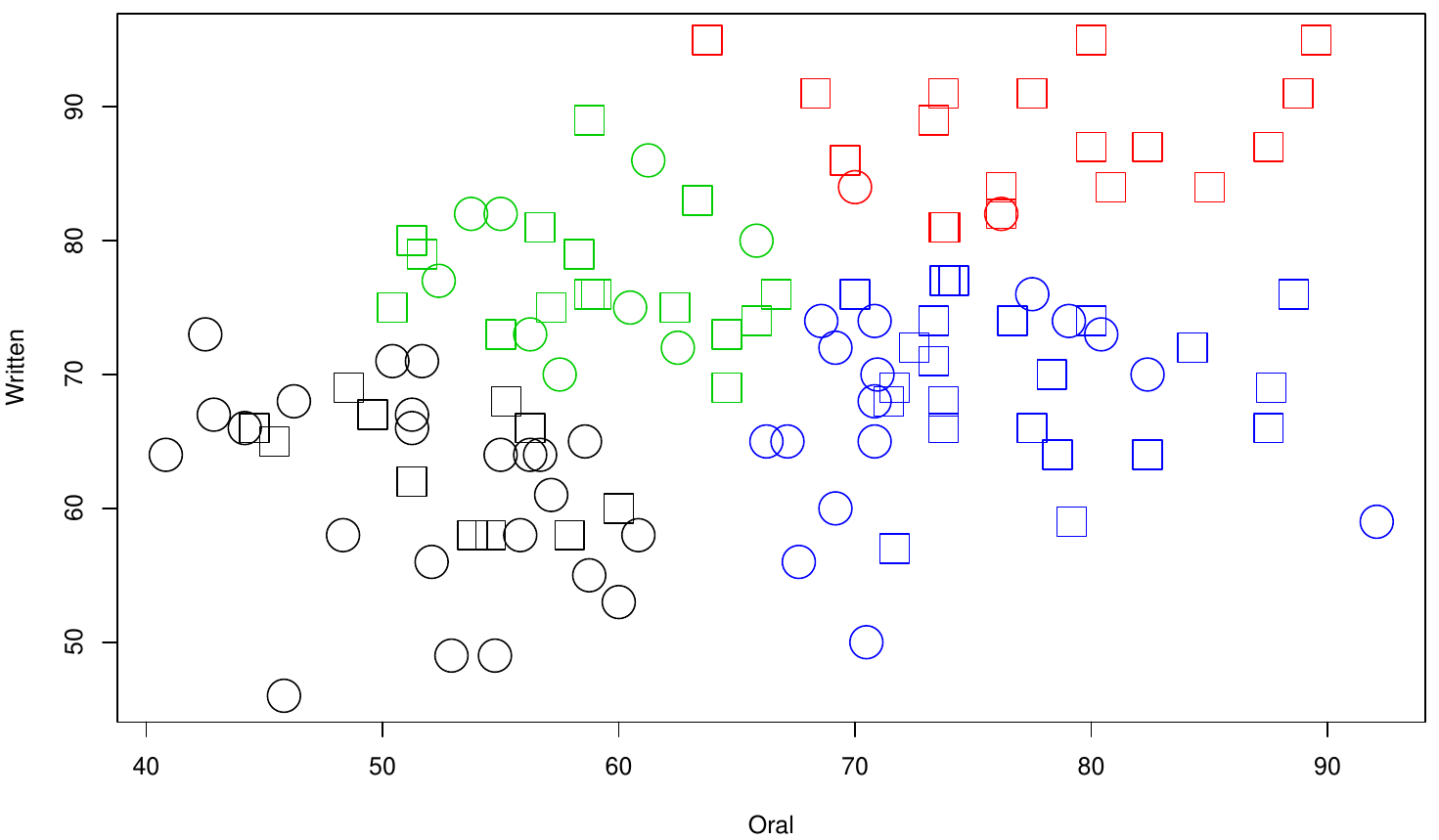}
\includegraphics[scale=0.32]{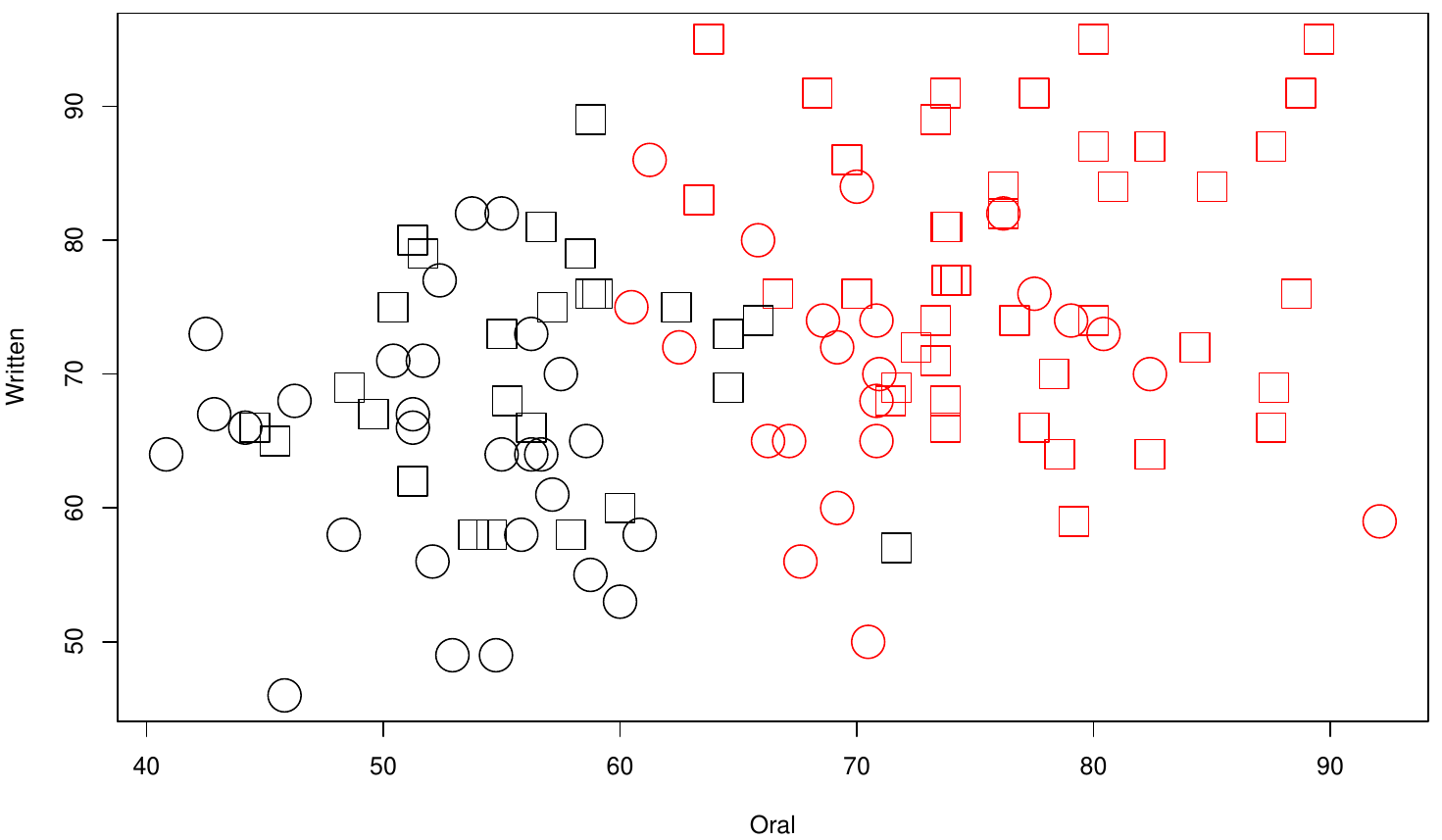}\includegraphics[scale=0.32]{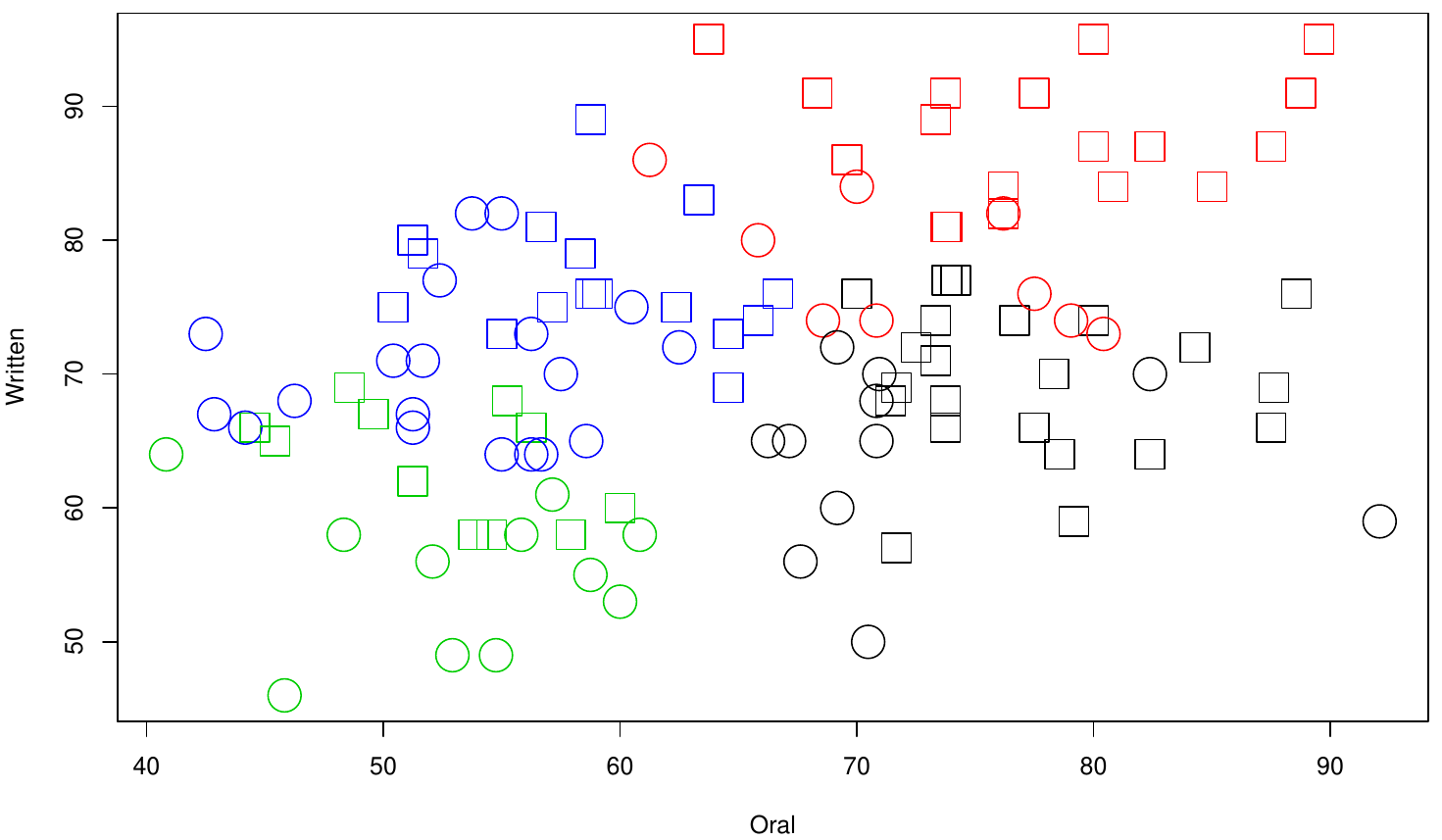}
\includegraphics[scale=0.32]{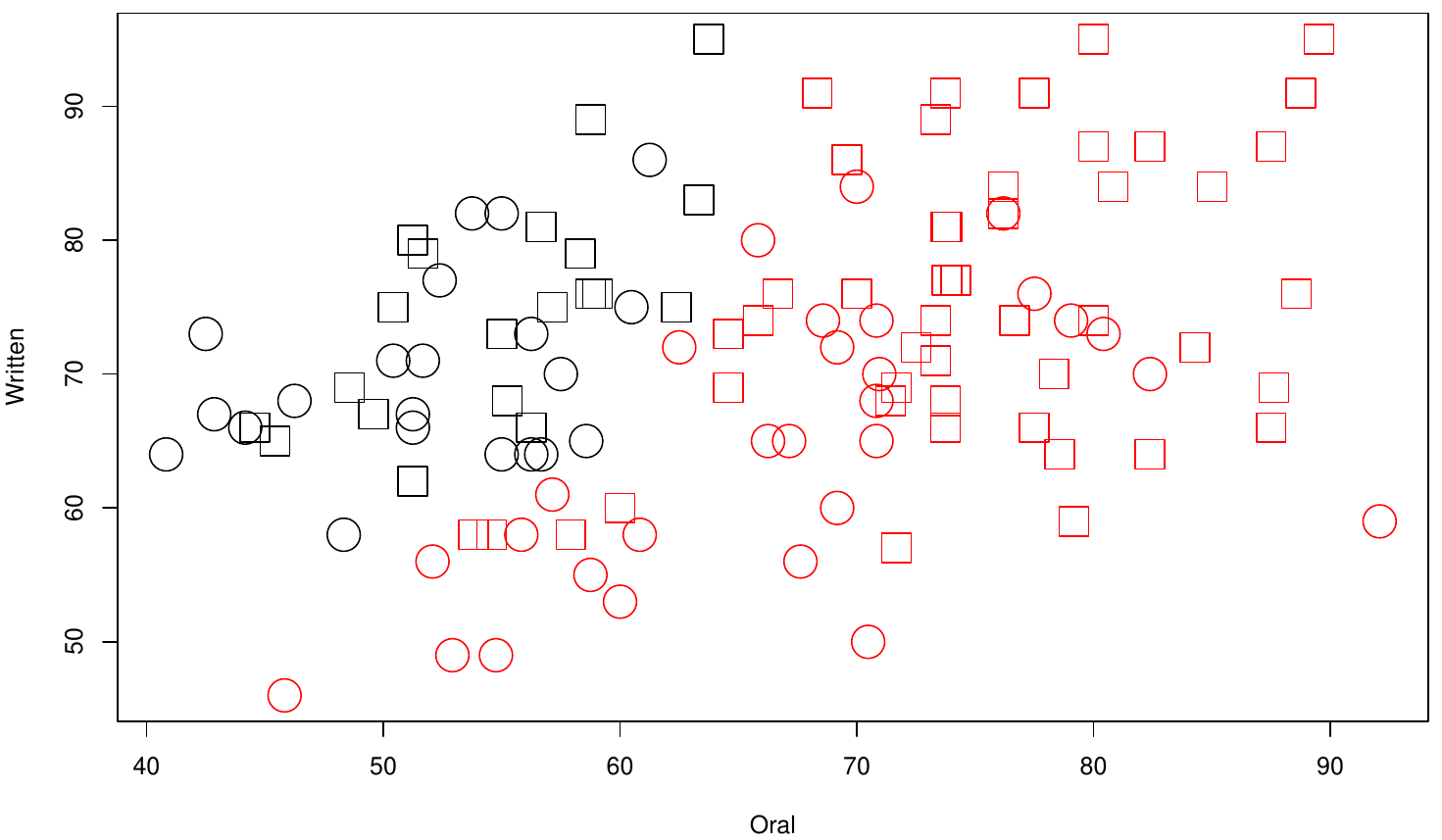}\includegraphics[scale=0.32]{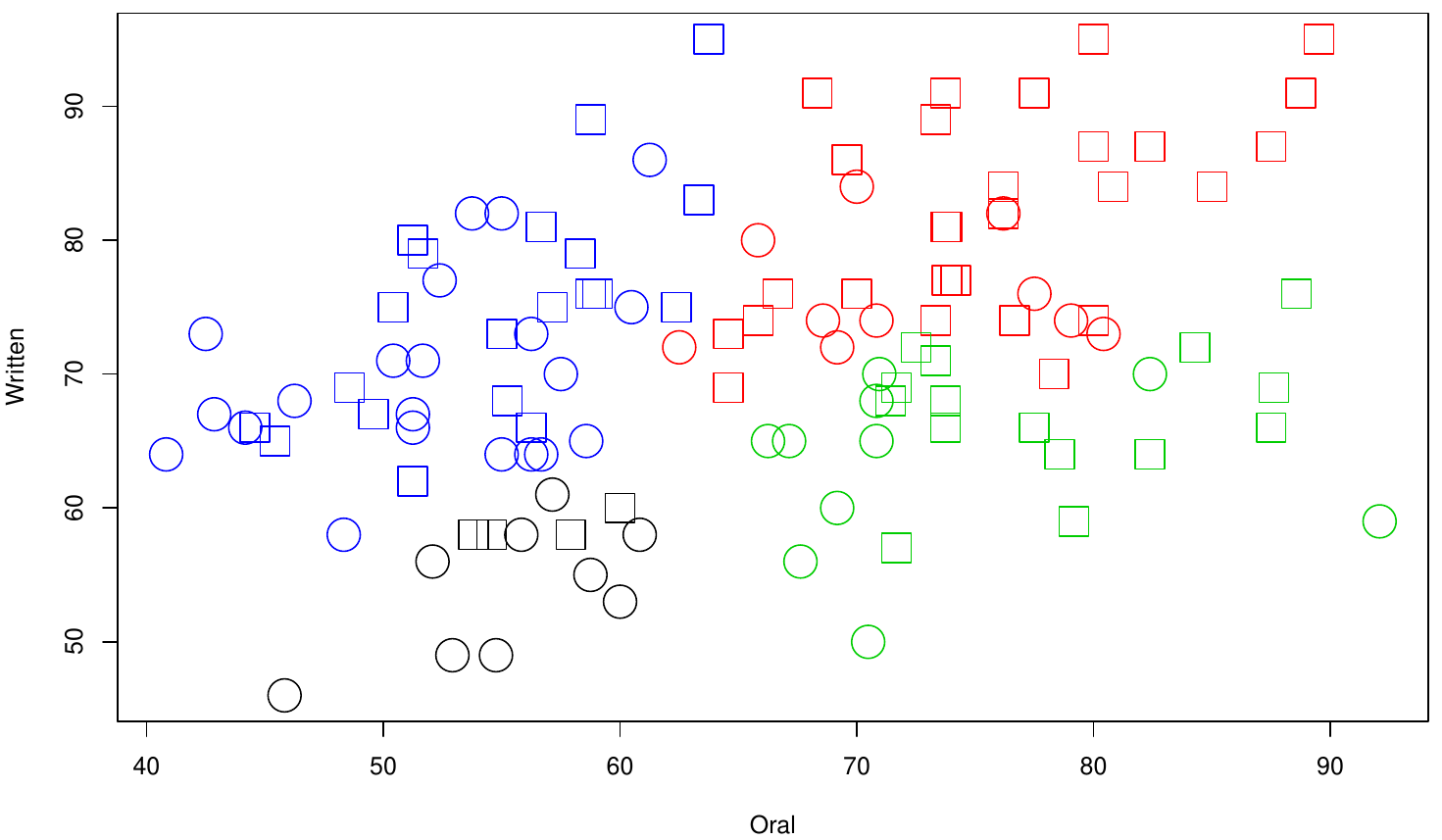}
\includegraphics[scale=0.31]{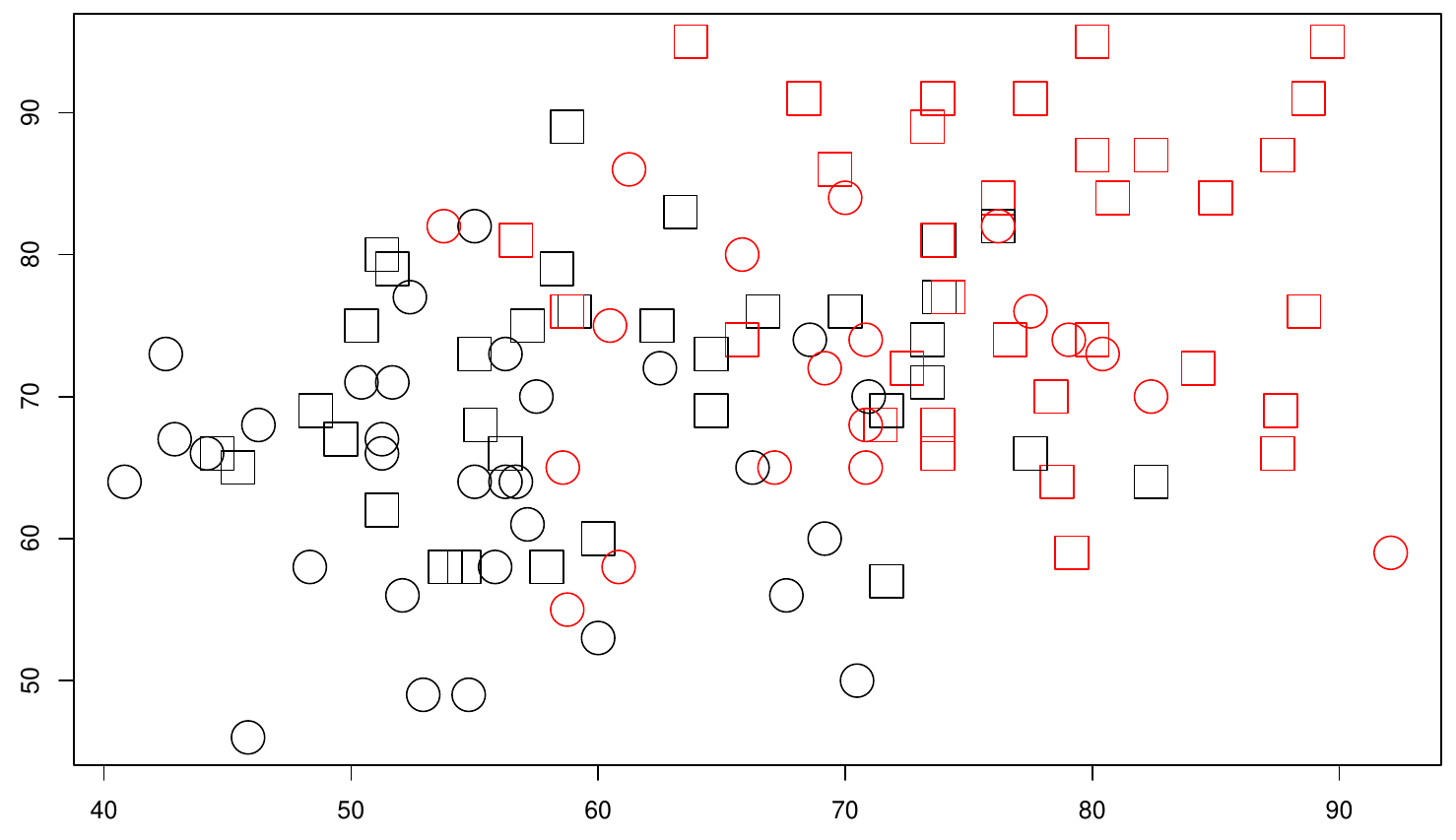}\includegraphics[scale=0.31]{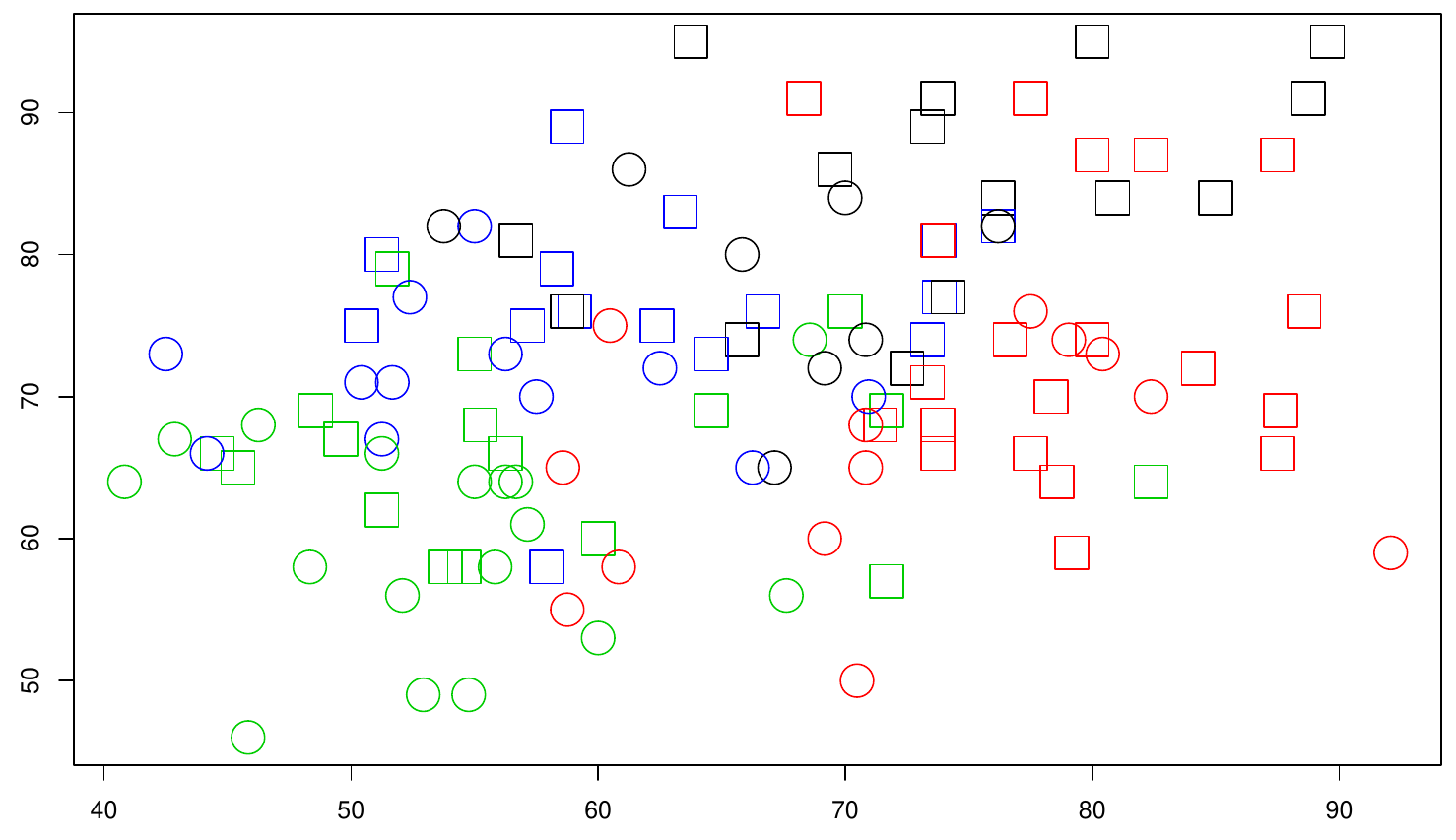}
\end{center}
\end{figure}

Our last comparison is on a real data set known as the Ricci dataset, consisting of scores in an oral and a written exam of 118 firefighters, where the sensitive attribute is the race of the individual. This dataset was a part of the case Ricci v. DeStefano presented before the Supreme Court of the United States.

For applying our attraction-repulsion clustering we codify white individuals as $S = 1$ and black and hispanic individuals as $S = -1$. The appropriate parameters for the dissimilarities are chosen to give a good performance (after a grid search as suggested in Section \ref{section_parameters}). The best results are obtained with our adaptation of Ward's method with $\delta_2$ and parameters $(u = 3.125, v = 10)$ and $k$-means after applying $\delta_1$, with parameters $(U = 0, V= 500)$, and a MDS embedding. Results are given in Table \ref{Tab_Ricci_comparison}. We see the balance given by using fairlets is higher than the one obtained with our procedures. However, we see that our procedures produce more identifiable and compact clusters. As a remark, we see that both procedures achieve a nice improvement in diversity compared to the $k$-means solution in its standard version. 

A plot of some of the clusterings can be seen in Figure \ref{fig_plot_soldiers}. Visually it is quite clear why the average silhouette index is higher in the attraction-repulsion clustering than in the fair $k$-median. It also clarifies what we mean by more identifiable and compact clusters.

\subsection{Civil Rights Data Collection}\label{CRDC}
In this section we are going to apply our procedure to the Schools Civil Rights Data Collection (CRDC) for the year 2015-2016 available for download in the link https://ocrdata.ed.gov/DownloadDataFile. We are going to work with the data for entry, middle level, and high schools in the state of New Jersey.

From the CRDC data we can collect the number of students in each school that belong to six distinct categories: \textit{Hispanic, American Indian/Alaska Native, Asian, Native Hawaiian/Pacific Islander, Black, White.} An entry of our dataset looks like this.
\begin{scriptsize}
\begin{verbatim}
      LEA_STATE_NAME   LEAID                  SCH_NAME hispanic native_american asian pacific_islander black white total
53559     NEW JERSEY 3400004       Chatham High School       46               0   106                2    19  1021  1204
53560     NEW JERSEY 3400004     Chatham Middle School       52               0    91                0     7   877  1052
53561     NEW JERSEY 3400004   Lafayette Avenue School       28               0    70                0     4   526   647
53562     NEW JERSEY 3400004      Milton Avenue School       19               0    37                0     0   274   355
53563     NEW JERSEY 3400004  Washington Avenue School       31               0    34                0     0   337   427
53564     NEW JERSEY 3400004 Southern Boulevard School       25               0    58                0     4   349   461
      hispanic_frac native_american_frac asian_frac black_frack pacific_islander_frac white_frac
53559    0.03820598                    0 0.08803987 0.015780731            0.00166113  0.8480066
53560    0.04942966                    0 0.08650190 0.006653992            0.00000000  0.8336502
53561    0.04327666                    0 0.10819165 0.006182380            0.00000000  0.8129830
53562    0.05352113                    0 0.10422535 0.000000000            0.00000000  0.7718310
53563    0.07259953                    0 0.07962529 0.000000000            0.00000000  0.7892272
53564    0.05422993                    0 0.12581345 0.008676790            0.00000000  0.7570499
\end{verbatim}
\end{scriptsize}
Additionally, using \textit{geolocate} of the package \emph{ggmap} in R we can extract latitude and longitude coordinates of the schools in New Jersey. Hence we have a precise location for the schools and can calculate distances in a straight line between them, i.e., geodesic distance on Earth between schools. A plot of the locations can be seen in Figure \ref{nj_schools_map}.

\begin{figure}
\caption{Left: location of schools in the state of New Jersey. Right: mds-embedding of the straight-line distances between the schools.}\label{nj_schools_map}
\begin{center}
\includegraphics[scale=0.7]{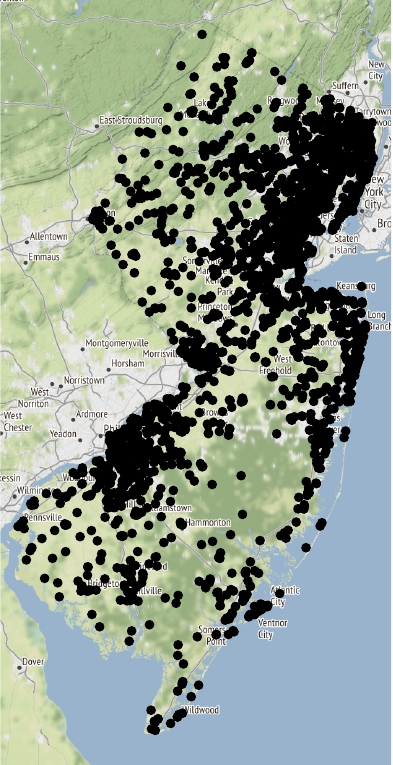}\hspace*{10pt}\includegraphics[scale=1.0845]{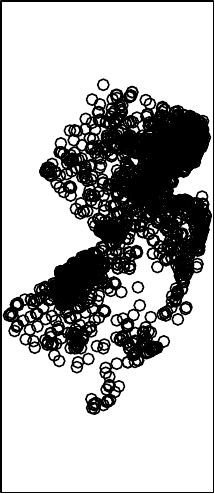}
\end{center}
\end{figure}

A relevant problem in public management is how to group some smaller entities (such as individual residences, schools, buildings, etc...) into some bigger entities (such as districts) to improve decision making, efficiency or some other relevant aspect. This is a typical problem in governance, with a famous example being electoral districts (constituencies) in the USA.

If we were faced with the problem of making diverse school districts, how should we proceed? Usually, districts incorporate meaningful spatial information and the diversity part in this case is related to the overall ethnic composition of the district. A possible approach to the problem is to perform diversity enhancing clustering, hence, identifying clusters of schools with districts. Next, we must consider what is the purpose of a district, since this is highly relevant for what are the ethical ramifications of diversity and for what are considered good clusters. For example, if all districts will be affected by some centralized decision, enforcing demographic parity through (\ref{fair_constraints}) may be a good proxy for fairness. Hence, diversity and fairness may be understood as interchangeable. However, if we were interested in a similar problem, which is making fair voting districts in a winner-takes-all system, promoting diversity may be highly inadequate (see for example \cite{Abbasi2021}). If we decide that we want to use diversity enhancing clustering, we should decide what types of clusters are adequate for our purposes? A simple approach is to look for spatially compact clusters because vicinity is usually a reasonable proxy for other shared features. At this point, we should be aware that diversity and vicinity requirements may be conflicting, and therefore some kind of trade-off between them is advisable. This setting is the one that we adopt and present attraction-repulsion clustering results below.


We will measure the degree of diversity of the partition by comparing the ethnic composition of each of its clusters to the ethnic composition of the overall dataset which is $p_t = (0.2423,\allowbreak 0.0018,\allowbreak 0.0988,0.0028, 0.1604,0.4741).$ This is done by using the \emph{unbalance} index defined in (\ref{unfiarness}).

Recall that the average silhouette index is a measure of the compactness of the clusters in  a partition. Hence our methodology is to use attraction-repulsion clustering where tuning is done over a mesh of distinct parameters, and the best parameters are those that give the lowest $\mathrm{unbalance}$ while keeping the average silhouette index over some threshold $\tau$. In this way we impose the maximum improvement in diversity while keeping part of the geographical information codified by the distance. This procedure can be seen in Algorithm \ref{tuningAlgo}.
\begin{algorithm}
\caption{Tuning}
\label{tuningAlgo}
\textbf{Input:} \emph{data, cluster.methods, $\delta_i$, parameters, $\tau$} 
\begin{algorithmic}[1]
\For {\emph{cluster.method} in \emph{cluster.methods}}
	\State $D \gets$ distance matrix computed using unprotected attributes in \emph{data}.
	\For {\emph{parameter} in \emph{parameters}}
		\State $\Delta \gets$ dissimilarity matrix computed using $\delta_i$, \emph{parameter} and entries of protected and unprotected attributes in \emph{data}
		\If {\emph{cluster.method} = $k$-means}
			\State $X \gets$ MDS embedding using $\Delta$
		\EndIf
		\State $\mathcal{C} \gets$ clustering using \emph{cluster.method}
		\State $U \gets \mathrm{unbalance}(\mathcal{C})$
		\State $aS \gets$ average silhouette index for $\mathcal{C}$ using $D$.
	\EndFor
	\State \emph{param.values} $\gets$ all respective touples $(U,aS)$ for the different \emph{parameter} values
	\State \emph{best.parameter} $\gets$ \emph{parameter} corresponding to the entry in \emph{param.values} such that $aS\geq \tau$ and with lowest $U$.
\EndFor
\end{algorithmic}
\textbf{Output:} best parameter for dissimilarity $\delta_i$ for each clustering method.
\end{algorithm}

Codification of the protected attributes in this case is straightforward, it is just the number of students at the school in each category. Hence, for instance Catham High School has $S_{C.H.S.} = (46,0,106,2,19,1021)$. We want to stress that another easy possibility is to use proportions with respect to the total number of students, i.e., $(0.0382, 0, 0.0880,0.0158,0.0017,0.8480)$. However, we choose to use the number of students since we think it should be considered, and that information is lost when using proportions. 

Selecting the grid or mesh of parameters for the different dissimilarities is an important task. For some of them, $\delta_2$ and $\delta_3$, it is mainly an analytical task of selecting the best values. However, in proposing candidates for $V$ in $\delta_1$ and $\delta_4$ we can use available information as we will see below. In this example we will concentrate on $\delta_1,\delta_2$ and $\delta_4$. To explore the effects of attraction-repulsion dissimilarities on different clustering methods we select a tiny sample. As hierarchical clustering methods, since we only have distances between schools, we will use complete, average, and single linkage. Through a MDS embedding we will use $k$-means. Hence, cluster.methods = (complete, average, single, $k$-means) in Algorithm \ref{tuningAlgo}.

The grid we use for $\delta_2$ in Algorithm \ref{tuningAlgo} is formed by $parameters=\{(u_i,v_j)\}_{i,j=1}^{i=19,j=10}$ with $\{u_i\}_{i=1}^{19} = \{0.1,  0.2,  0.3,  0.4,  0.5,\allowbreak  0.6,  0.7,  0.8,  0.9,  1,  2,  3,  4,  5,  6,\allowbreak  7,  8,  9, 10\}$ and $\{v_j\}_{j=1}^{10} = \{0.1,  0.3,  0.5,\allowbreak  0.7,  0.9,  1,  3.25,\allowbreak  5.5, 7.75, 10\}$.

As we stated previously, for proposing values for $V = v_0\tilde{V}$ in $\delta_1$ and $\delta_4$ we are going to use some a priori information that can be corroborated by the data. The most numerous minorities are the Hispanic, Asian, and Black communities.  Even more, it is well known that poor neighbourhoods have higher concentrations of minorities, and therefore schools in those areas should be representative of that. Hence, this is a major source for a lack of diversity in a mainly geographical clustering of the schools. Since white students are the majority, values of the proportion for the previously mentioned minorities and white students will affect our unbalance index the most. Hence we should attempt to achieve mixing in precisely these groups. Our first three proposals are variations of schemes that should improve mixing in the above-mentioned communities.
$$\tilde{V}=\Bigg\{\left( {\begin{array}{cccccc}
   1 & -1 & -1 & -1 & -1 & -1 \\
   0 & 0 & 0 & 0 & 0 & 0 \\
   -1 & -1 & 1 & -1 & -1 & -1\\
   0 & 0 & 0 & 0 & 0 & 0 \\
   -1 & -1 & -1 & -1 & 1 & -1\\
   0 & 0 & 0 & 0 & 0 &1 \\
  \end{array} } \right),
  \left( {\begin{array}{cccccc}
   1 & -1 & -1 & -1 & -1 & -1 \\
   0 & 0 & 0 & 0 & 0 & 0 \\
   0 & 0 & 1 & 0 & 0 & 0 \\
   0 & 0 & 0 & 0 & 0 & 0 \\
   -1 & -1 & -1 & -1 & 1 & -1\\
   0 & 0 & 0 & 0 & 0 &1 \\
  \end{array} } \right),$$
$$  \left( {\begin{array}{cccccc}
   1 & -1 & -1 & -1 & -1 & -1 \\
   0 & 0 & 0 & 0 & 0 & 0 \\
   0 & 0 & 0 & 0 & 0 & 0 \\
   0 & 0 & 0 & 0 & 0 & 0 \\
   -1 & -1 & -1 & -1 & 1 & -1\\
   0 & 0 & 0 & 0 & 0 &1 \\
  \end{array} } \right),
  \left( {\begin{array}{cccccc}
   1 & -1 & -1 & -1 & -1 & -1 \\
   -1 & 1 & -1 & -1 & -1 & -1 \\
   -1 & -1 & 1 & -1 & -1 & -1 \\
   -1 & -1 & -1 & 1 & -1 & -1 \\
   -1 & -1 & -1 & -1 & 1 & -1\\
   -1 & -1 & -1 & -1 & -1 & 1 \\
  \end{array} } \right),$$
$$  \left( {\begin{array}{cccccc}
   1 & 0 & 0 & 0 & 0 & 0 \\
   -1 & 1 & -1 & -1 & -1 & -1 \\
   0 & 0 & 1 & 0 & 0 & 0 \\
   -1 & -1 & -1 & 1 & -1 & -1 \\
   0 & 0 & 0 & 0 & 1 & 0\\
   0 & 0 & 0 & 0 & 0 & 1 \\
  \end{array} } \right),
  \left( {\begin{array}{cccccc}
   0 & 0 & 0 & 0 & 0 & 0 \\
   -1 & 1 & -1 & -1 & -1 & -1 \\
   0 & 0 & 0 & 0 & 0 & 0 \\
   -1 & -1 & -1 & 1 & -1 & -1 \\
   0 & 0 & 0 & 0 & 0 & 0\\
   0 & 0 & 0 & 0 & 0 & 0 \\
  \end{array} } \right)\Bigg\}.$$
The fourth proposal is the obvious one, which tries to produce mixing in all communities. The fifth and sixth proposals try to impose mixing mainly for the smallest minorities American Indians/Alaska Natives and Native Hawaiians/Pacific Islanders.

\begin{figure}
\caption{Continuous lines represent unbalance while dashed lines represent average silhouette index. The clustering methods are complete linkage in black, $k$-means in red, single linkage in green and average linkage in blue. Top row is a comparison between unperturbed clustering methods, bottom row is a comparison of clustering methods after using $\delta_1$. The x label indicates the number of clusters which goes from 2 to 15. We stress that single linkage with $\delta_1$, green line in bottom row, does not achieve $\tau>0$ for $k\geq 9$.  }\label{unperturbedClusterings}
\begin{center}
\includegraphics[scale=0.8]{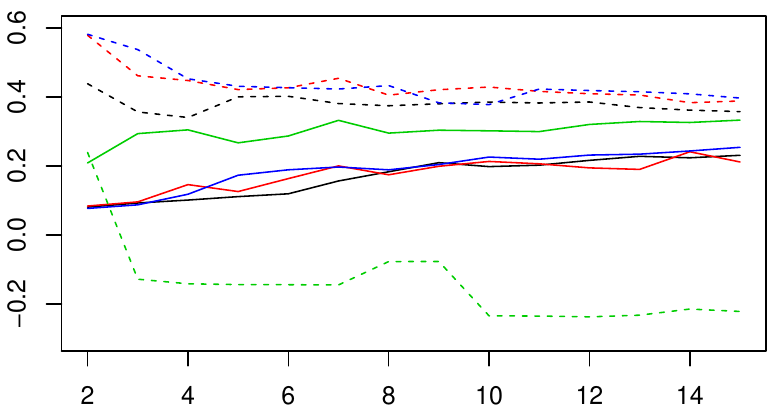}
\includegraphics[scale=0.8]{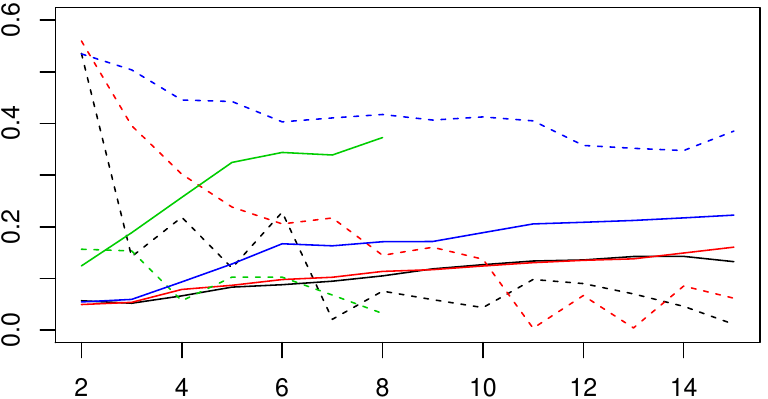}
\end{center}
\end{figure}

\begin{figure}
\caption{Continuous lines represent unbalance while dashed lines represent average silhouette index. Top: a comparison between effects of unperturbed and perturbed situations  for $k$-means. Bottom: same comparison as in Top but for complete linkage clustering. In black we have the unperturbed situation, in red we have the best $\delta_1$ perturbation, in green the best $\delta_2$ and in blue we have $\delta_4$. The x label indicates the number of clusters which goes from 2 to 15.}\label{values_comparison}
\begin{center}
\includegraphics[scale=0.8]{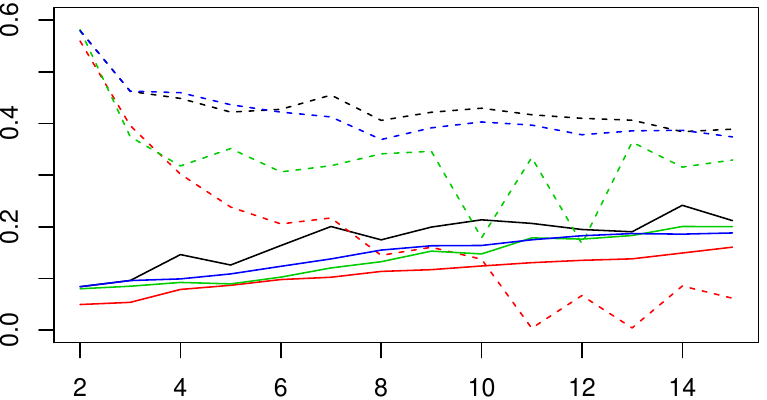}
\includegraphics[scale=0.8]{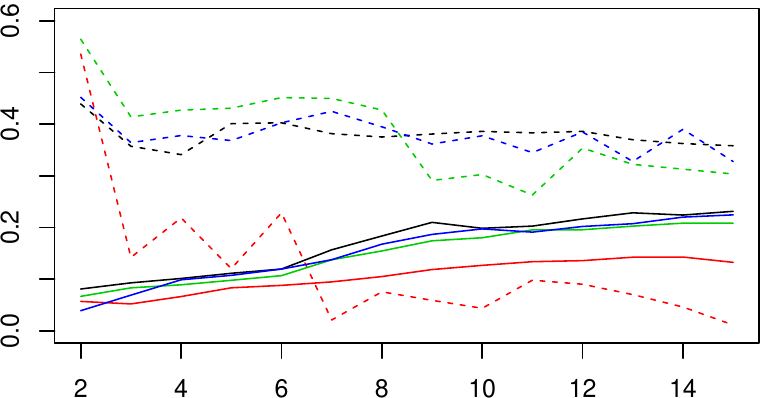}
\end{center}
\end{figure}

Now we can define the values we will use for the input parameters in Algorithm \ref{tuningAlgo}. For $\delta_1$ we have $parameters = \{(U,v_0,V')\}$ where $U = 0_{6\times 6}$, $v_0 \in \{0.001, 0.002, 0.003, 0.004, 0.005, 0.006, 0.007, 0.008,\allowbreak 0.009,\allowbreak 0.01, 0.02, 0.03, 0.04, 0.05, 0.06, 0.07, 0.08,\allowbreak 0.09,\allowbreak 0.1,\allowbreak 0.2, 0.3, 0.4, 0.5, 0.6, 0.7, 0.8, 0.9, 1\}$ and $V'\in\tilde{V}$. For $\delta_4$ another parameter that can incorporate a priori information is  $w$, which in this case tells us how strong the influence between schools  that are further apart should be. For the local dissimilarity $\delta_4$ we propose to use $parameters = \{(u,v,w,V)\}$ such that $u \in \{0.5,2,8\}$, $v \in \{0.1,1,10\}$, $w \in \{0.1,0.5,0.9,1,5.5,10\}$ and $V\in\tilde{V}$.

Before showing results for the attraction-repulsion procedures, we will cluster the data without considering race information, hence only geographical distance is of importance. From the top row of Figure \ref{unperturbedClusterings}, we can say that $k$-means (red) is giving a good performance since it has low unbalance index and a reasonably high average silhouette index. We stress that the $k$-means procedure is done in the MDS embedding shown on the right in Figure \ref{nj_schools_map}. In the left column of Figure \ref{map_comparison} we can see $k$-means clustering for 5 and 11 clusters. We clearly see that spatial proximity is the driving force of the clustering. The values of unbalance and average silhouette index  are respectively $(0.13, 0.42)$ and $(0.21, 0.42)$.

To apply our attraction-repulsion clustering as shown in Algorithm \ref{tuningAlgo} we need to fix the silhouette bound. To do this we take $\tau = 0$, and hence we are not imposing very strong compactness criteria, recall that the silhouette index varies between -1 and 1, but we still want clusters to be relatively compact. In this way we are making a trade-off between reduction in unbalance and spatial coherence. In Figure \ref{values_comparison} we see the effects of the different best parameters for dissimilarities $\delta_1,\delta_2$ and $\delta_4$ for $k$-means and complete linkage clustering. Generally, we see that $\delta_1$ (in red) is the dissimilarity that produces the strongest reduction in unbalance (solid line) and of course this is on behalf of a reduction in average silhouette index (dashed line). Also as expected a local dissimilarity as $\delta_4$ (blue) brings only a modest reduction of unbalance but maintains a high spatial coherence.

Hence, if we want to slightly gerrymander districts to improve diversity while maintaining a very high geographical coherence we should use $\delta_4$. On the other hand, if we want the maximum unbalance reduction achievable with our procedure we should use $\delta_1$.

To decide which clustering method is the best, from our small selection of methods, we use the bottom row of Figure \ref{unperturbedClusterings}. There we see that for attraction-repulsion clustering using $\delta_1$, $k$-means (red) and complete linkage (black) produce similar reduction in unbalance while $k$-means keeps a relatively higher average silhouette index. Therefore, in this case $k$-means with perturbation $\delta_1$ and a MDS embedding seems to be the best procedure.

In Figure \ref{map_comparison} we can see a visual comparison between $k$-means clustering in the different situations. The best parameters for the dissimilarities in the cases we have shown are the following

\begin{footnotesize}
$$params(\mathcal{C}^{\delta_4}_5) = \left(2,1,10,\left( {\begin{array}{cccccc}
   1 & -1 & -1 & -1 & -1 & -1 \\
   0 & 0 & 0 & 0 & 0 & 0 \\
   -1 & -1 & 1 & -1 & -1 & -1\\
   0 & 0 & 0 & 0 & 0 & 0 \\
   -1 & -1 & -1 & -1 & 1 & -1\\
   0 & 0 & 0 & 0 & 0 &1 \\
  \end{array} } \right)\right),
$$$$  params(\mathcal{C}^{\delta_4}_{11}) = \left(2,0.1,1,\left( {\begin{array}{cccccc}
   1 & -1 & -1 & -1 & -1 & -1 \\
   0 & 0 & 0 & 0 & 0 & 0 \\
   0 & 0 & 1 & 0 & 0 & 0 \\
   0 & 0 & 0 & 0 & 0 & 0 \\
   -1 & -1 & -1 & -1 & 1 & -1\\
   0 & 0 & 0 & 0 & 0 &1 \\
  \end{array} } \right)\right),$$
$$params(\mathcal{C}^{\delta_1}_5) = \left(0_{6\times 6},0.03,\left( {\begin{array}{cccccc}
   1 & -1 & -1 & -1 & -1 & -1 \\
   -1 & 1 & -1 & -1 & -1 & -1 \\
   -1 & -1 & 1 & -1 & -1 & -1 \\
   -1 & -1 & -1 & 1 & -1 & -1 \\
   -1 & -1 & -1 & -1 & 1 & -1\\
   -1 & -1 & -1 & -1 & -1 & 1 \\
  \end{array} } \right)\right),$$
$$  params(\mathcal{C}^{\delta_1}_{11}) = \left(0_{6\times 6},0.07,\left( {\begin{array}{cccccc}
   1 & -1 & -1 & -1 & -1 & -1 \\
   -1 & 1 & -1 & -1 & -1 & -1 \\
   -1 & -1 & 1 & -1 & -1 & -1 \\
   -1 & -1 & -1 & 1 & -1 & -1 \\
   -1 & -1 & -1 & -1 & 1 & -1\\
   -1 & -1 & -1 & -1 & -1 & 1 \\
  \end{array} } \right)\right).$$
\end{footnotesize}
What we see is that the matrices $V$ that we introduced trying to codify available a priori information work very well. Hence real-world intuitions are compatible with our model. In the case of the parameters for $\delta_1$ again what we expected from intuition is seen, i.e., to reduce unbalance for a bigger number of clusters it is necessary to use a stronger perturbation (a higher value for $v_0$). In the plots we clearly see the behaviour we have previously mentioned. The local dissimilarity $\delta_4$ does a sort of positive gerrymandering while $\delta_1$ imposes a stronger reduction in unbalance and hence significantly alters the clustering results.
\begin{figure}
\caption{Results of $k$-means clustering, top row looking for 5 clusters and bottom row looking for 11. Left: unperturbed situation. Middle: $\delta_4$ perturbation and MDS embedding. Right: $\delta_1$ perturbation and MDS embedding.}\label{map_comparison}
\begin{center}
\includegraphics[scale=0.75]{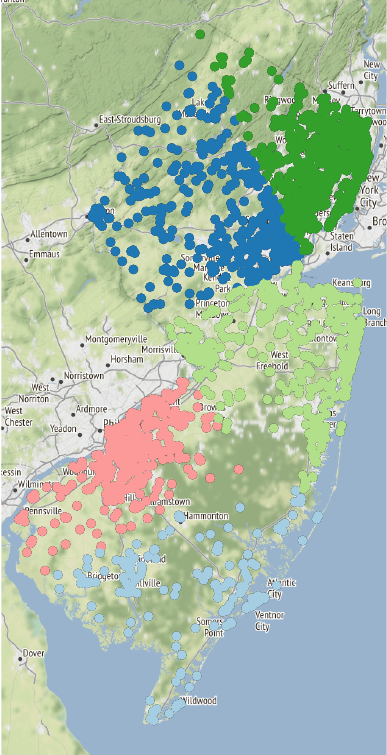}\hspace*{5pt}\includegraphics[scale=0.75]{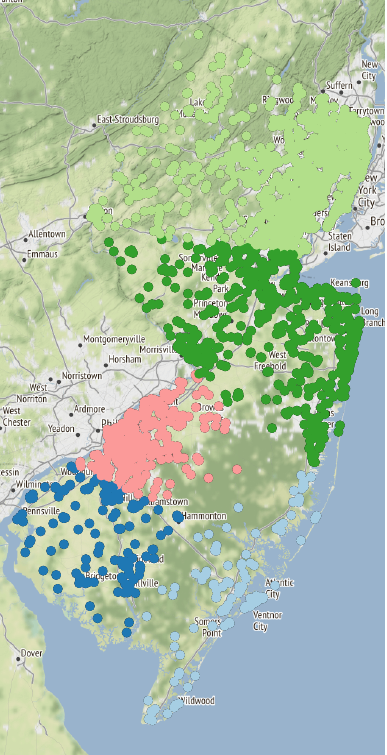}\hspace*{5pt}\includegraphics[scale=0.75]{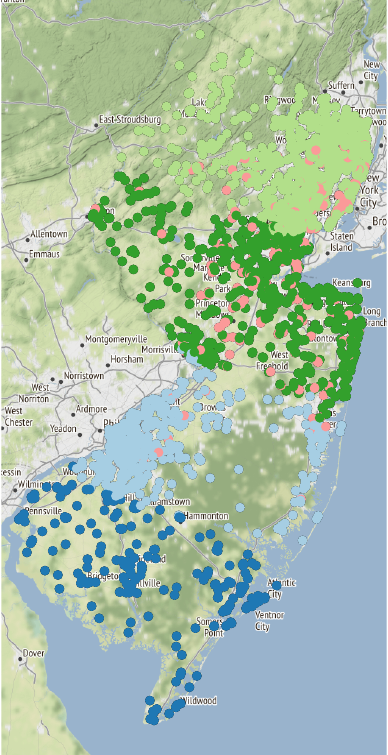}
\vspace*{5pt}

\includegraphics[scale=0.75]{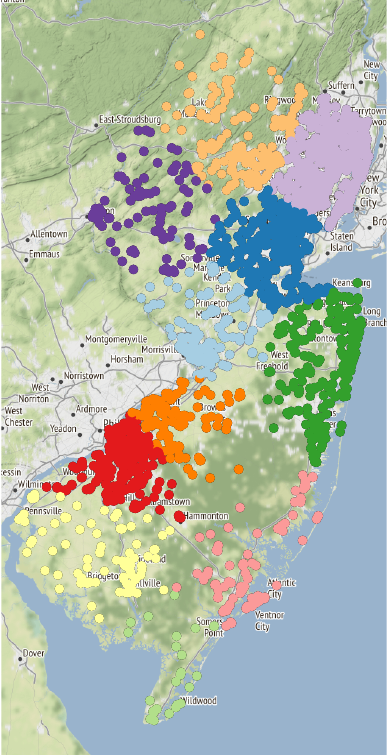}\hspace*{5pt}\includegraphics[scale=0.75]{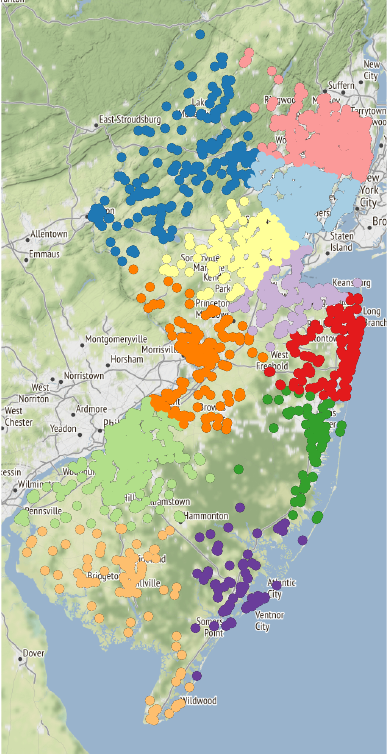}\hspace*{5pt}\includegraphics[scale=0.75]{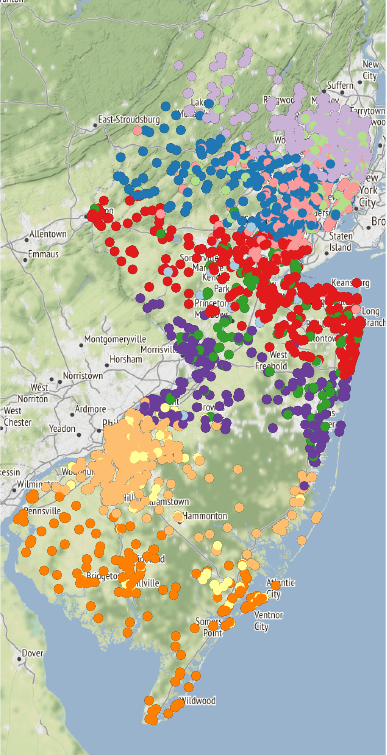}
\end{center}
\end{figure}

\section{Conclusions}\label{section_conclusions}

We have introduced a pre-processing heuristic based on novel dissimilarities that is applicable virtually to any clustering method and can produce gains in diversity. Hence, when diversity is a good proxy for fairness, for example when demographic parity is required, our methods may fall into a fair clustering setting. Our heuristic approach allows for tractable and simple computations at the cost of tuning requirements and no theoretical guarantees for enforcing the diversity preserving condition (\ref{fair_constraints}).
	
The experiments described in Section \ref{section_applications}, while simple, allow us to reach some meaningful conclusions. To begin with, pre-processing transformations looking to impose diversity through independence of the marginal distributions with respect to the protected attributes, but ignoring the geometry of the data, can eliminate any meaningful cluster structure as seen in Section \ref{general_example}. This can also happen when using ad hoc fair (diversity preserving) clustering objective functions as shown in Section \ref{section_comparison}. Hence, it is advisable to reach a trade-off between the desire to find some cluster structure present in the geometry of the data and the desire to impose diversity preserving conditions. Furthermore, the justification of diversity, and more broadly, fairness concerns in clustering rely heavily on the purpose of the practitioner and the data itself, something that is implicit in Section \ref{CRDC}. The previous considerations make us urge practitioners to be clearly aware of:
\begin{itemize}
\item What is the purpose of the partitions that they are looking for
\item The definitions of diversity, with its relations to fairness, and of cluster structure that are suitable to their purpose
\item The conflict that may exist between diversity and the groupings present in the data
\item The level of trade-off between diversity and sensible cluster structure that is reasonable for the users purpose 
\end{itemize}

As for future directions of research, there are two clear paths. One is to try to use a transformation of the data that can give some theoretical guarantees for the demographic parity condition (\ref{fair_constraints}), most likely restricting clustering procedures to a certain type. Another is to consider the trade-off between diversity, cluster structure and maybe some other relevant criteria for the practitioner through a multi-objective optimization formulation. An additional alternative is to try to explore model-based clustering in the setting of attraction-repulsion dissimilarities.
  
\section*{Acknowledgements}
Research partially supported by FEDER, Spanish Ministerio de Economía y Competitividad, grant MTM2017-86061-C2-1-P and Junta de Castilla y León, grants VA005P17 and VA002G18.  Research partially supported by ANITI program and DEEL IRT. Research partially funded by the Basque Government through the BERC 2018-2021 program and by the Spanish Ministry of Science, Innovation, and Universities (BCAM Severo Ochoa accreditation SEV-2017-0718).


\begin{thebibliography}{39}
\providecommand{\natexlab}[1]{#1}
\providecommand{\url}[1]{{#1}}
\providecommand{\urlprefix}{URL }
\expandafter\ifx\csname urlstyle\endcsname\relax
  \providecommand{\doi}[1]{DOI~\discretionary{}{}{}#1}\else
  \providecommand{\doi}{DOI~\discretionary{}{}{}\begingroup
  \urlstyle{rm}\Url}\fi
\providecommand{\eprint}[2][]{\url{#2}}

\bibitem{Abbasi2021}
Abbasi M, Bhaskara A, Venkatasubramanian S (2021) Fair clustering via equitable
  group representations. FAccT '21: Proceedings of the 2021 ACM Conference on
  Fairness, Accountability, and Transparency pp 504--514

\bibitem{Ahmadian2019}
Ahmadian S, Epasto A, Kumar R, Mahdian M (2019) Clustering without
  over-representation. arXiv:190{51}2753v1

\bibitem{backurs2019}
Backurs A, Indyk P, Onak K, Schieber B, Vakilian A, Wagner T (2019) Scalable
  fair clustering. arXiv:190203519

\bibitem{fair_paula}
del Barrio E, Gamboa F, Gordaliza P, Loubes J (2018) Obtaining fairness using
  optimal transport theory. arXiv:180{60}3195

\bibitem{bera2019}
Bera S, Chakrabarty D, Negahbani M, College D (2019) Fair algorithms for
  clustering. arXiv:190102393

\bibitem{bercea2018}
Bercea I, Gross M, Khuller S, Kumar A, Rösner C, Schmidt D, Schmidt M (2018)
  On the cost of essentially fair clusterings. arXiv:181110319

\bibitem{2018arXiv181001729B}
Besse P, Castets-Renard C, Garivier A, Loubes J (2018) Can everyday {AI} be
  ethical. {F}airness of {M}achine {L}earning {A}lgorithms. arXiv:181{00}1729

\bibitem{Chen2019}
Chen X, Fain B, Lyu L, Munagala K (2019) Proportionally fair clustering.
  arXiv:190{50}3674

\bibitem{fair_k_means}
Chierichetti F, Kumar R, Lattanzi S, Vassilvitskii S (2017) Fair clustering
  through fairlets. Advances in Neural Information Processing Systems
  30:5029--5037

\bibitem{chouldechova2017fair}
Chouldechova A (2017) Fair prediction with disparate impact: {A} study of bias
  in recidivism prediction instruments. Big data 5:153--163

\bibitem{multid_scaling}
Cox T, Cox M (2000) Multidimensional Scaling. Chapman and Hall/CRC

\bibitem{Cristianini2004}
Cristianini N (2004) Kernel Methods for Pattern Analysis. Cambridge University
  Press

\bibitem{Cristianini2019}
Cristianini N (2019) Shortcuts to artificial intelligence. In: Pelillo M,
  Scantamburlo T (eds) Machines We Trust: Perspectives on Dependable AI, MIT
  Press

\bibitem{dbscan}
Ester M, Kriegel H, Sander J, Xu X (1996) A {D}ensity-{B}ased {A}lgorithm for
  {D}iscovering {C}lusters in {L}arge {S}patial {D}atabases with {N}oise.
  Proceedings of the {S}econd {I}nternational {C}onference on {K}nowledge
  {D}iscovery and {D}ata {M}ining pp 226--231

\bibitem{Everittetal}
Everitt B, Landau S, Leese M, Stahl D (2011) Cluster Analysis. Wiley

\bibitem{cosa_3}
Feldman M, Friedler S, Moeller J, Scheidegge C, Venkatasubramanian S (2015)
  Certifying and removing disparate impact. Proceedings of the 21th {ACM}
  {SIGKDD} {I}nternational {C}onference on {K}nowledge {D}iscovery and {D}ata
  {M}ining pp 259--268

\bibitem{repuls_clust}
Ferraro M, Giordani P (2013) On possibilistic clustering with repulsion
  constraints for imprecise data. Information Sciences 245:63--75

\bibitem{2018arXiv180204422F}
Friedler S, Scheidegger C, Venkatasubramanian S, Choudhary S, Hamilton E, Roth
  D (2018) A comparative study of fairness-enhancing interventions in machine
  learning. arXiv:180{20}4422F

\bibitem{tclust}
Garc\'{i}a-Escudero L, Gordaliza A, Matr\'{a}n C, Mayo-\'{I}scar A (2008) A
  general trimming approach to robust cluster analysis. Ann Statist
  36:1324--1345

\bibitem{Ghadiri2021}
Ghadiri M, Samadi S, Vampala S (2021) Socially fair k-means clustering. FAccT
  '21: Proceedings of the 2021 ACM Conference on Fairness, Accountability, and
  Transparency pp 438--448

\bibitem{Hausdorf2003}
Hausdorf B, Hennig C (2003) Biotic element analysis in biogeography. Systematic
  Biology 52(5):717--723

\bibitem{HennigMeilaMurtaghRocci}
Hennig C, Meila M, Murtagh F, Rocci R (2015) Handbook of Cluster Analysis. CRC
  Press

\bibitem{Huang2019}
Huang L, Jiang S, Vishnoi N (2019) Coresets for clustering with fairness
  constraints. Advances in Neural Information Processing Systems pp
  7587–--7598

\bibitem{Hubert1985}
Hubert L, Arabie P (1985) Comparing partitions. Journal of Classification
  2:193--218

\bibitem{Kaufman1987}
Kaufman L, Rousseeuw PJ (1987) Clustering by means of medoids. In: Dodge Y (ed)
  Statistical Data Analysis Based on the L1 Norm and Related Methods,
  Birkh{\"a}user, pp 405--416

\bibitem{fair_gpc}
Kehrenberg T, Chen Z, Quadrianto N (2018) Interpretable fairness via target
  labels in gaussian process models. arXiv:181{00}5598v2

\bibitem{lance-wil}
Lance G, Williams W (1967) A {G}eneral {T}heory of {C}lassificatory {S}orting
  {S}trategies: 1. {H}ierarchical {S}ystems. The Computer Journal 9(4):373--380

\bibitem{2016arXiv161008077L}
Lum K, Johndrow J (2016) A statistical framework for fair predictive
  algorithms. arXiv:161{00}8077L

\bibitem{Mazud2019}
Mazu{d Z}iko I, Yuan J, Granger E, Be{n A}yed I (2019) Variational fair
  clustering. arXiv:190{60}8207v5

\bibitem{mehrabi2019}
Mehrabi N, Morstarre F, Saxena N, Lerman K, Galstyan A (2019) A {S}urvey on
  {B}ias and {F}airness in {M}achine {L}earning. arXiv:190{80}9635v2

\bibitem{hier_clust}
Murtagh F, Contreras P (2011) Algorithms for hierarchical clustering: an
  overview. WIREs Data Mining and Knowledge Discovery

\bibitem{Oh2007}
Oh MS, Raftery AE (2007) Model-based clusteringwith dissimilarities: A bayesian
  approach. Journal of Computational and Graphical Statistics 16(3):559--585

\bibitem{silhouette}
Rousseeuw P (1987) Silhouettes: {A} graphical aid to the interpretation and
  validation of cluster analysis. Journal of Computational and Applied
  Mathematics 20:53--65

\bibitem{rosner18}
Rösner C, Schmidt M (2018) Privacy preserving clustering with constraints.
  45th {I}nternational {C}olloquium on {A}utomata, {L}anguages and
  {P}rogramming

\bibitem{kkmeans}
Schölkopf B, Smola A, Müller K (1998) Nonlinear component analysis as a
  kernel eigenvalue problem. Neural Computation 10(5):1299–1319

\bibitem{schmidt2018}
Schmidt M, Schwiegelshohn C, Sohler C (2018) Fair core-sets and streaming
  algorithms for fair k-means clustering. arXiv:181210854

\bibitem{Shalev2014}
Shalev-Shwartz S, Ben-David S (2014) Understanding Machine Learning: From
  Theory to Algorithms. Cambridge University Press

\bibitem{Wang2019}
Wang B, Davidson I (2019) Towards fair deep clusteringwith multi-state
  protected variables. arXiv:190{11}0053v1

\bibitem{cosa_1}
Zafar M, Valera I, Rodriguez M, Gummadi K (2017) Fairness constraints:
  Mechanisms for fair classification. Proceedings of the 20th International
  Conference on Artificial Intelligence and Statistics 54:962--970

\end{thebibliography}

\end{document}